\documentclass[8pt]{article}
\pdfoutput=1
\usepackage{arxiv}
\setcitestyle{authoryear,open={(},close={)}}
\renewcommand{\cite}{\citep}

\usepackage{tabulary}
\usepackage{tabularx}
\usepackage{array}
\usepackage{algorithm, algorithmicx}
\usepackage{algpseudocode}
\usepackage{setspace}
\usepackage{bbm}
\usepackage{pifont}
\usepackage{graphicx}
\usepackage[font=small]{caption}
\usepackage{subcaption}
\usepackage{xcolor}
\usepackage{wrapfig}
\usepackage{selectp}

\usepackage[most]{tcolorbox}
\usepackage{fancyvrb}
\usepackage{fvextra}

\usepackage{caption}
\DeclareCaptionLabelFormat{andtable}{#1~#2  \&  \tablename~\thetable}

\usepackage{mathtools}
\usepackage{setspace}

\usepackage[capitalize,noabbrev]{cleveref}

\newlength\savewidth

\newtheorem{theorem}{Theorem}[section]

\newtheorem{lemma}[theorem]{Lemma}

\theoremstyle{definition}

\theoremstyle{remark}

\title{
Outliers with Opposing Signals Have an \\Outsized Effect on Neural Network Optimization
}

\author{%
	\large{Elan Rosenfeld} \\
	\normalsize{Carnegie Mellon University}\\
	{\texttt{elan@cmu.edu}}
	\And
	\large{Andrej Risteski} \\
	\normalsize{Carnegie Mellon University}\\
	{\texttt{aristesk@andrew.cmu.edu}}
}
\ifdefined\usebigfont

\usepackage[fontsize=13pt]{scrextend}
\AtBeginDocument{
	\newgeometry{left=1.56in,right=1.56in,top=1.71in,bottom=1.77in}
}
\else
\fi

\usepackage{amsmath,amsfonts,bm}

\def\eqref#1{equation~\ref{#1}}

\def\1{\bm{1}}

\DeclareMathAlphabet{\mathsfit}{\encodingdefault}{\sfdefault}{m}{sl}
\SetMathAlphabet{\mathsfit}{bold}{\encodingdefault}{\sfdefault}{bx}{n}

\newcommand{\E}{\mathbb{E}}

\newcommand{\R}{\mathbb{R}}

\DeclareMathOperator{\sign}{sign}

\begin{document}

\maketitle

\newcommand{\groupspace}{\mathcal{G}}
\newcommand{\groupset}{\mathbb{G}}
\newcommand{\calN}{\mathcal{N}}
\newcommand{\calX}{\mathcal{X}}
\newcommand{\calY}{\mathcal{Y}}
\newcommand{\noise}{\epsilon}
\newcommand{\dt}[1]{\frac{d#1}{dt}}
\newcommand{\dx}[1]{\frac{d#1}{dx}}
\newcommand{\psdnorm}[2]{\|#1\|_{#2}}
\newcommand{\colvec}[2]{\begin{bmatrix}#1 \\ #2\end{bmatrix}}
\newcommand{\indep}{\perp\!\!\!\perp}
\newcommand{\notindep}{\not\!\indep}
\newcommand{\sgn}{\textrm{sign}}
\newcommand{\betastar}{\beta}
\newcommand{\outlierbeta}{o}
\newcommand{\maj}{\textrm{maj}}
\newcommand{\full}{\textrm{full}}
\newcommand{\twonormsq}[1]{\|#1\|_2^2}
\renewcommand{\P}{\mathbb{P}}

\begin{abstract}
We identify a new phenomenon in neural network optimization which arises from the interaction of depth and a particular heavy-tailed structure in natural data. Our result offers intuitive explanations for several previously reported observations about network training dynamics. In particular, it implies a conceptually new cause for progressive sharpening and the edge of stability; we also highlight connections to other concepts in optimization and generalization including grokking, simplicity bias, and Sharpness-Aware Minimization.

Experimentally, we demonstrate the significant influence of paired groups of outliers in the training data with strong \emph{opposing signals}: consistent, large magnitude features which dominate the network output throughout training and provide gradients which point in opposite directions. Due to these outliers, early optimization enters a narrow valley which carefully balances the opposing groups; subsequent sharpening causes their loss to rise rapidly, oscillating between high on one group and then the other, until the overall loss spikes. We describe how to identify these groups, explore what sets them apart, and carefully study their effect on the network's optimization and behavior.
We complement these experiments with a mechanistic explanation on a toy example of opposing signals and a theoretical analysis of a two-layer linear network on a simple model.
Our finding enables new qualitative predictions of training behavior which we confirm experimentally. It also provides a new lens through which to study and improve modern training practices for stochastic optimization, which we highlight via a case study of Adam versus SGD.
\end{abstract}

\section{Introduction}

There is a steadily growing list of intriguing properties of neural network (NN) optimization which are not readily explained by classical tools from optimization. Likewise, we have varying degrees of understanding of the mechanistic causes for each. Extensive efforts have led to possible explanations for the effectiveness of Adam \citep{kingma2014adam}, Batch Normalization \citep{ioffe2015batch} and other tools for successful training---but the evidence is not always entirely convincing, and there is certainly little theoretical understanding. 
Other findings, such as grokking \citep{power2022grokking} or the edge of stability \citep{cohen2021gradient}, do not have immediate practical implications but provide new ways to study what sets NN optimization apart.
These phenomena are typically considered in isolation---though they are not completely disparate, it is unknown what specific underlying causes they may share. Clearly, a better understanding of NN training dynamics in a specific context can lead to algorithmic improvements \citep{chen2021empirical}; this suggests that any commonality will be a valuable tool for further investigation.

In this work, we identify a phenomenon in NN optimization which offers a new perspective on many of these prior observations and which we hope will contribute to a deeper understanding of how they may be connected. While we do not (and do not claim to) give a complete explanation, we present strong qualitative and quantitative evidence for a single high-level idea---one which naturally fits into several existing narratives and suggests a more coherent picture of their origin. Specifically, we demonstrate the prevalence of paired groups of outliers in natural data which have a significant influence on a network's optimization dynamics. These groups are characterized by the inclusion of one or more (relatively) large magnitude features that dominate the network's output at initialization and throughout most of training. In addition to their magnitude, the other distinctive property of these features is that they provide large, consistent, and \emph{opposing} gradients, in that following one group's gradient to decrease its loss will increase the other's by a similar amount. Because of this structure, we refer to them as \emph{Opposing Signals}. These features share a non-trivial correlation with the target task, but they are often not the ``correct'' (e.g., human-aligned) signal. In fact, in many cases these features perfectly encapsulate the classic statistical conundrum of ``correlation vs. causation''---for example, a bright blue sky background does not determine the label of a CIFAR image, but it does most often occur in images of planes. Other features \emph{are} very relevant, such as the presence of wheels and headlights in images of trucks and cars, or the fact that a colon often precedes either ``the'' or a newline token in written text.

\begin{figure}[t]
\centering
    \includegraphics[width=\linewidth]{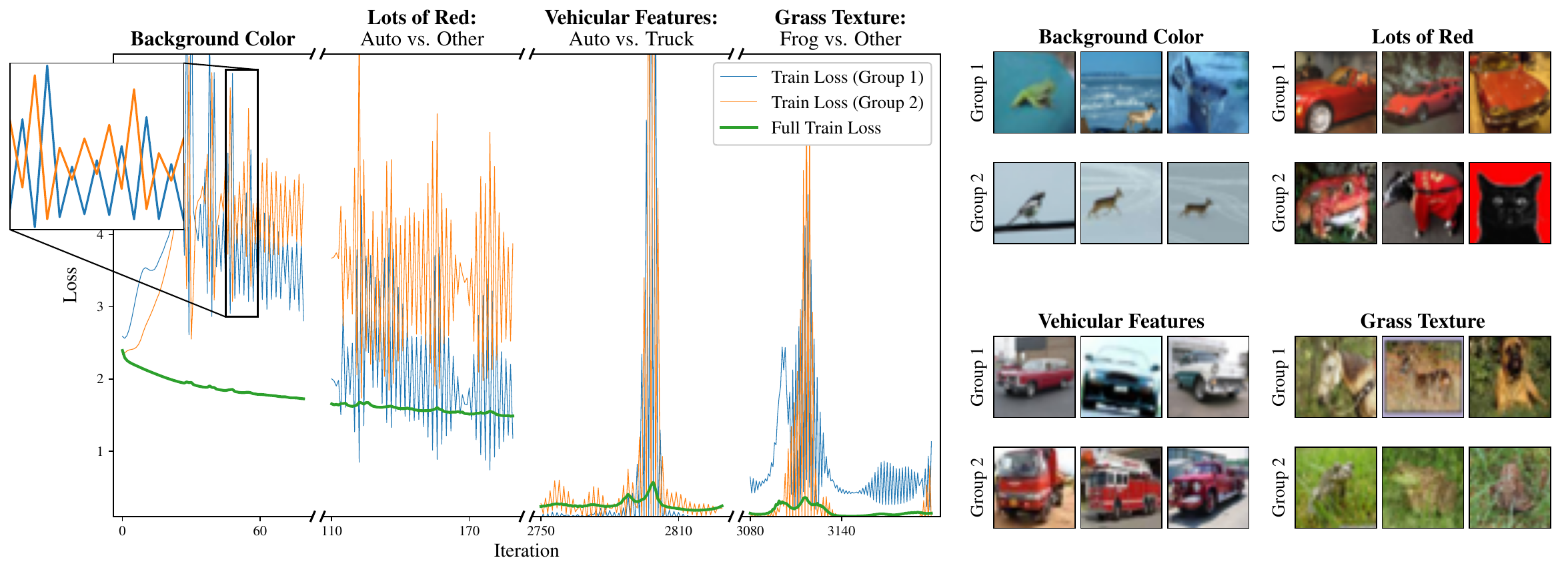}
    \caption{\textbf{Training dynamics of neural networks are heavily influenced by outliers with opposing signals.} We plot the overall loss of a ResNet-18 trained with GD on CIFAR-10, plus the losses of a small but representative set of outlier groups. These groups have consistent \emph{opposing signals} (e.g., wheels and headlights can mean either \texttt{car} or \texttt{truck}). Throughout training, losses on these groups oscillate with growing and shrinking amplitude---this has an obvious correspondence to the intermittent spikes in overall loss and appears to be a direct cause of the edge of stability phenomenon.
    }
    \label{fig:main-teaser}
\end{figure}

Opposing signals are most easily understood with an example, which we will give along with a brief outline of their effect on training dynamics; a more detailed description is presented in \cref{sec:main-explanation}.
\cref{fig:main-teaser} depicts the training loss of a ResNet-18 \citep{he2016deep} trained with full-batch gradient descent (GD) on CIFAR-10 \citep{krizhevsky2009learning}, along with a few dominant outlier groups and their respective losses.
In the early stages of training, the network enters a narrow valley in weight space which carefully balances the pairs' opposing gradients; subsequent sharpening of the loss landscape \citep{Jastrzebski2020The, cohen2021gradient} causes the network to oscillate with growing magnitude along particular axes, upsetting this balance. Returning to our example of a sky background, one step results in the class \texttt{plane} being assigned greater probability for all images with sky, and the next will reverse that effect. In essence, the ``sky $=$ \texttt{plane}'' subnetwork grows and shrinks.\footnote{It would be more precise to say ``strengthening connections between regions of the network's output and neurons which have large activations for sky-colored inputs''.
Though we prefer to avoid informal terminology,
this example makes clear that the more relaxed phrasing is usually much cleaner. We therefore employ it when the intended meaning is clear.}
The direct result of this oscillation is that the network's loss on images of planes with a sky background will alternate between sharply increasing and decreasing with growing amplitude, with the exact opposite occurring for images of \emph{non}-planes with sky.
Consequently, the gradients of these groups will alternate directions while growing in magnitude as well. As these pairs represent a small fraction of the data, this behavior is not immediately apparent from the overall training loss---but eventually, it progresses far enough that the overall loss spikes. As there is an obvious direct correspondence between these two events throughout, we conjecture that opposing signals are a direct cause of the \emph{edge of stability} phenomenon \citep{cohen2021gradient}. We also note that the most influential signals appear to increase in complexity over time \citep{nakkiran2019sgd}.

We repeat this experiment across a range of vision architectures and training hyperparameters: though the precise groups and their order of appearance change, the pattern occurs consistently. We also verify this behavior for transformers on next-token prediction of natural text and small ReLU MLPs on simple 1D functions; we give some examples of opposing signals in text in \cref{app:text-signals}. However, we rely on images for exposition because it offers the clearest intuition. To isolate this effect, most of our experiments use GD, but we observe similar patterns during SGD which we present in \cref{sec:sgd}.

\paragraph{Summary of contributions.} The primary contribution of this paper is demonstrating the existence, pervasiveness, and large influence of opposing signals during NN optimization.
We further present our current best understanding, with supporting experiments, of how these signals \emph{cause} the observed training dynamics---in particular, we provide evidence that it is a consequence of depth and steepest descent methods. We complement this discussion with a toy example and an analysis of a two-layer linear net on a simple model. Notably, though rudimentary, our explanation enables concrete qualitative predictions of NN behavior during training, which we confirm experimentally. It also provides a new lens through which to study modern stochastic optimization methods, which we highlight via a case study of SGD vs. Adam.
We see possible connections between opposing signals and a wide variety of phenomena in NN optimization and generalization, including \emph{grokking} \citep{power2022grokking}, \emph{catapulting/slingshotting} \citep{lewkowycz2020large, thilak2022slingshot}, \emph{simplicity bias} \citep{valle-perez2018deep}, \emph{double descent} \citep{belkin2019reconciling, Nakkiran2020Deep}, and Sharpness-Aware Minimization \citep{foret2021sharpnessaware}. We discuss these and other connections in \cref{sec:discussion}.

\section{Characterizing and Identifying Opposing Signals}
\label{sec:setup}

Though their influence on aggregate metrics is non-obvious, identifying outliers with opposing signals is straightforward. Our methodology is as follows: when training a network with GD, we track its loss on each individual training point. For a given iteration, we select the training points whose loss exhibited the most positive and most negative change in the preceding step (there is large overlap between these sets in successive steps). This set will sometimes contain multiple opposing signals, which we distinguish via visual inspection. This last detail means that the images we depict are not random, but we emphasize that it would not be correct to describe this process as cherry-picking: though precise quantification is difficult, these signals consistently obey the maxim ``I know it when I see it''. This is particularly true for images, such as the groups in \cref{fig:main-teaser} which have easily recognizable patterns. To demonstrate this fact more generally, \cref{app:addl-images} contains the pre-inspection samples for a ResNet-18, VGG-11 \citep{simonyan2014very}, and a small Vision Transformer \citep{dosovitskiy2020image} at several training steps and for multiple seeds; we believe the implied groupings are immediate, even if not totally objective. We see algorithmic approaches to automatically clustering these samples as a direction for future study---for example, one could select samples by correlation in their loss time-series, or by gradient alignment.

\paragraph{Measuring alternative metrics.} Given how these samples are selected, several other characterizations seem appropriate. For instance, one-step loss change is often a reasonable proxy for gradient norm; we could also consider the largest eigenvalue of the loss of the \emph{individual point}, or how much curvature it has in the direction of the overall loss's top eigenvector. For large networks these options are far more compute-intensive than our chosen method, but we can evaluate them on specific groups. In \cref{fig:interest-stats} we track these metrics for several opposing group pairs and we find that they are consistently much larger than that of random samples from the training set.

\subsection{On the Possibility of a Formal Definition} Though the features and their exemplar samples are immediately recognizable, \textbf{we do not attempt to \emph{exactly} define a ``feature'', nor an ``outlier'' with respect to that feature.} The presence of a particular feature is often ambiguous, and it is difficult to define a clear threshold for what makes a given point an outlier.\footnote{In the case of language---where tokenization is discrete and more interpretable---a precise definition is sometimes possible. For example, one opposing pair in \cref{app:text-signals} consists of sequences whose penultimate token is a colon and whose last token is either ``the'' or a newline.} Thus, instead of trying to exactly partition the data, we simply note that these heavy tails \emph{exist} and we use the most obvious outliers as representatives for visualization. In \cref{fig:main-teaser,fig:interest-stats} we choose an arbitrary cutoff of twenty samples per group.

\begin{figure}[t]
    \centering
    \includegraphics[width=\linewidth]{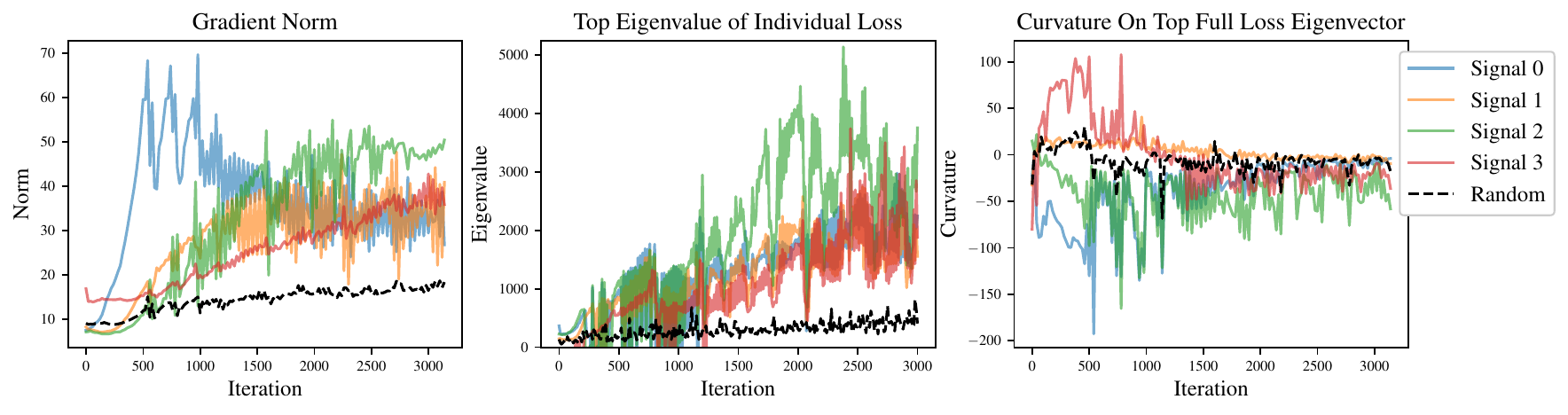}
    \caption{\textbf{Tracking other metrics which characterize outliers with opposing signals.} Maximal per-step change in loss relates to other useful metrics, such as per-sample gradient norm and curvature. We combine each pair of groups in \cref{fig:main-teaser} to create training subsets which each exemplify one ``signal'': we see that these samples are also significant outliers according to the other metrics. (For a point $x$, ``Curvature on Top Full Loss Eigenvector'' is defined as $v^\top H(x)v$, where $v$ is the top eigenvector of the full loss Hessian and $H(x)$ is the Hessian of the loss on $x$ alone.)}
    \label{fig:interest-stats}
\end{figure}

We also note that what qualifies as an opposing signal or outlier may vary over time. For visual clarity, \cref{fig:main-teaser} depicts the loss on only the most dominant group pair in its respective training phase, but this pattern occurs simultaneously for many different signals and at multiple scales throughout training. Further, the opposing signals are with respect to the model's internal representations (and the label), not the input space itself; this means that the definition is also a property of the architecture. For example, following \citet{cohen2021gradient} we train a small MLP to fit a Chebyshev polynomial on evenly spaced points in the interval $[-1, 1]$ (\cref{fig:chebyshev}). This data has no ``outliers'' in the traditional sense, and it is not immediately clear what opposing signals are present. Nevertheless, we observe the same alternating behavior: we find a pair where one group is a small interval of $x$-values and the opposing group contains its neighbors, all in the range $[-1, -0.5]$. This suggests that the network has internal activations which are heavily influential only for more negative $x$-values. In this context, these two groups are the outliers.

\begin{figure}[t!]
    \centering
    \includegraphics[width=\linewidth]{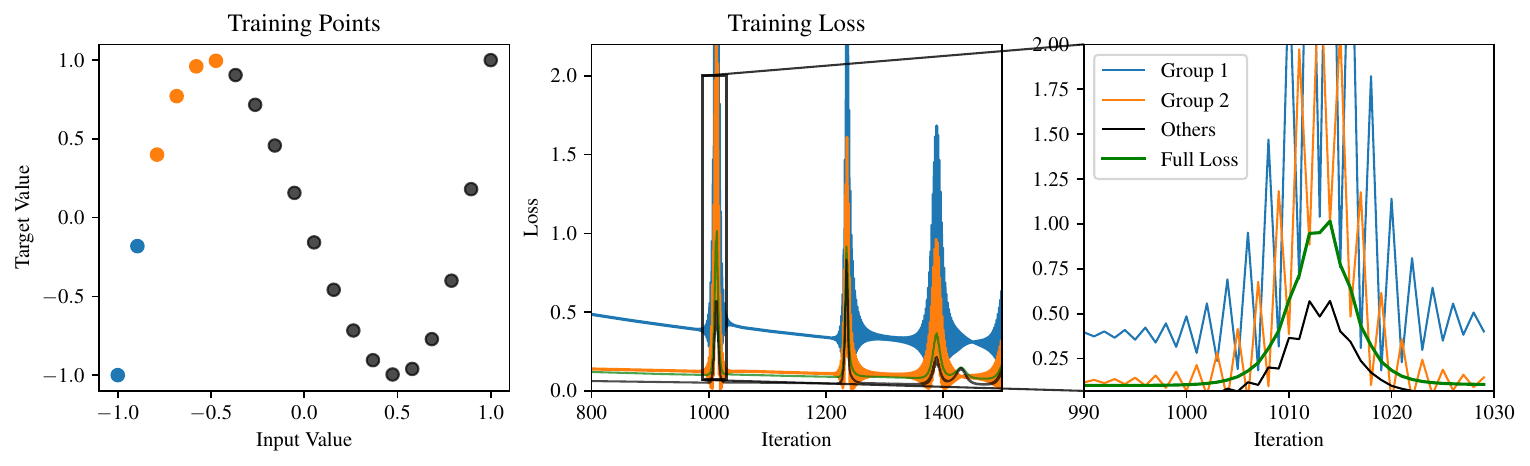}
    \caption{\textbf{Opposing signals when fitting a Chebyshev polynomial with a small MLP.} Though the data lacks traditional ``outliers'', it is apparent that the network has some features which are most influential only on the more negative inputs (or whose effect is otherwise cancelled out by other features). Since the correct use of this feature is opposite for these two groups, they provide opposing signals.}
    \label{fig:chebyshev}
\end{figure}

\section{Understanding the Effect of Opposing Signals}
\label{sec:main-explanation}

Beyond noting their existence, our eventual goal will be to derive actionable insights from this finding. To do this, it is necessary to gain a better understanding of \emph{how} these outliers cause the observed behavior. In this section we give a simplified ``mental picture'' which serves as our current understanding this process. We begin with an informal discussion of the outsized influence of opposing signals and how they lead to progressive sharpening; this subsection collates and expands on prior work to give important context for how these signals differ from typically imagined ``noise''. Next, we give a mechanistic description of the specific effect of opposing signals with a toy example. This explanation is intentionally high-level, but we will eventually see how it gives concrete predictions of specific behaviors, which we then verify on real networks. Finally, we prove that this behavior occurs on a two-layer linear network under a simple model.

\subsection{Progressive Sharpening, and Intuition for Why these Features are so Influential}
\label{subsec:prog-sharp}
At a high level, most variation in the input is unneeded when training a network to minimize predictive error---particularly with depth and high dimension, only a small fraction of information will be propagated to the last linear layer \citep{huh2021low}.
Starting from random initialization, training a network aligns adjacent layers' singular values \citep{saxe2013exact, mulayoff2020unique} to amplify meaningful signal while downweighting noise,\footnote{In this discussion we use the term ``noise'' informally. We refer not necessarily to pure randomness, but more generally to input variation which is least useful in predicting the target.} growing \emph{sensitivity} to the important signal. This sensitivity can be measured, for example, by the spectral norm of the input-output Jacobian, which grows during training \citep{ma2021linear}; it has also been connected to growth in the norm of the output layer \citep{wang2022analyzing}. 

Observe that with this growth, small changes to \emph{how the network processes inputs} become more influential.
Hypothetically, a small weight perturbation could massively increase loss by redirecting unhelpful noise to the subspace to which the network is most sensitive, or by changing how the last layer uses it. 
The increase of this sensitivity thus represents precisely the growth of loss Hessian spectrum, with the strength of this effect increasing with depth \citep{wang2016analysis, du2018algorithmic, mulayoff2020unique}.\footnote{The coincident growth of these two measures was previously noted by \citet{ma2021linear, gamba2023lipschitz, macdonald2023progressive}, though they did not make explicit this connection to how the network processes different types of input variance.}

Crucially, this sharpening also depends on the structure of the input. If the noise is independent of the target, it will be downweighted throughout training. In contrast, \emph{genuine signals which oppose each other} will be retained and perhaps even further amplified by gradient descent; this is because the ``correct'' feature may be much smaller in magnitude (or not yet learned), so using the large, ``incorrect'' feature is often the most immediate way of minimizing loss. 
As a concrete example, observe that a randomly initialized network will lack the features required for the subtle task of distinguishing birds from planes. But it \emph{will} capture the presence of sky, which is very useful for reducing loss on such images by predicting the conditional $p(\text{class } | \text{ sky})$ (this is akin to the ``linear/shallow-first'' behavior described by \citet{nakkiran2019sgd, mangalam2019do}). Thus, any method attempting to minimize loss as fast as possible (e.g., steepest descent) may actually upweight these features. 
Furthermore, amplified opposing signals will cause greater sharpening than random noise, because using a signal to the benefit of one group is maximally harmful for the other---e.g., confidently predicting \texttt{plane} whenever there is sky will cause enormous loss on images of other classes with sky. Since random noise is more diffuse, this effect is less pronounced.

This description is somewhat abstract. To gain a more precise understanding, we illustrate the dynamics more explicitly on a toy example.

\subsection{Illustrating with a Hypothetical Example of Gradient Descent}
\label{sec:toy-example}
    
Consider the global loss landscape of a neural network: this is the function which describes how the loss changes as we move through parameter space. Suppose we identify a direction in this space which corresponds to the network's use of the ``sky'' feature to predict \texttt{plane} versus some other class. That is, we will imagine that whenever the input image includes a bright blue background, moving the parameters in one direction increases the logit of the \texttt{plane} class and decreases the others, and vice-versa. We will also decompose this loss---\textbf{among images with a sky background, we consider \emph{separately} the loss on those labeled \texttt{plane} versus those with any other label.} Because the sky feature has large magnitude, a small change in weight space will induce a large change in the network outputs---
i.e., a small movement in the direction ``sky = \texttt{plane}'' will greatly increase loss on these non-\texttt{plane} images.

\begin{figure}[t]
    \centering
    \includegraphics[width=\linewidth]{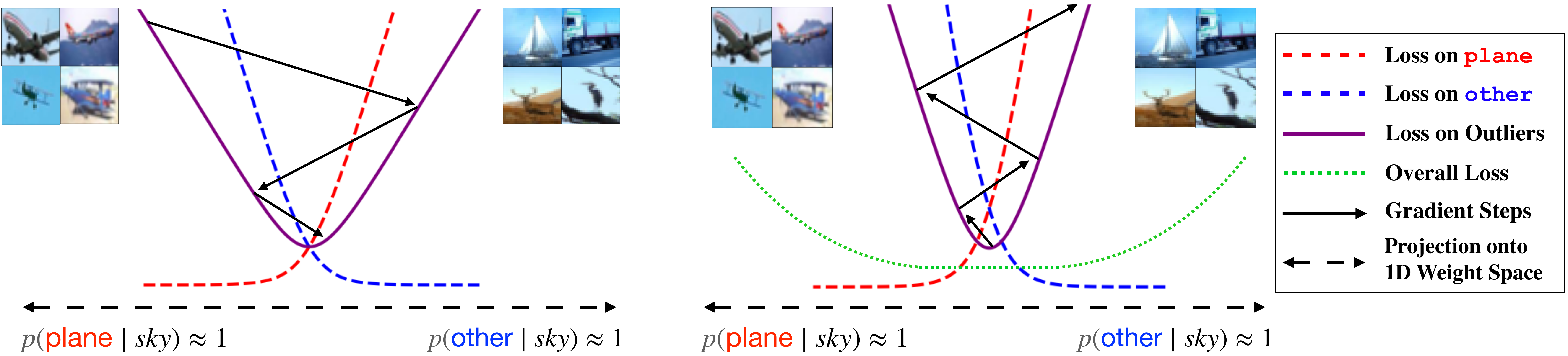}
    \caption{\textbf{A toy illustration of the effect of opposing signals.} Images with many blue pixels cause large activations, with high loss sensitivity. We project the loss to the hypothetical weight-space dimension ``sky = \texttt{plane}''. \textbf{Left:} Early optimization approaches the minimum, balancing the opposing gradients for \textcolor{red}{\texttt{plane}} and \textcolor{blue}{\texttt{other}} (these are losses for \emph{separate} training subsets: those labeled \texttt{plane} vs. those with any other label---the purple curve is their average). Progress continues through this valley, further growing the feature magnitude. \textbf{Right}: The valley sharpens and the iterates diverge, alternating between high and low loss for each group. Because most training points are insensitive to this axis, the overall loss may not be noticeably affected at first. Eventually either (a) the loss growth forces the network to downweight ``sky'', flattening the valley; or (b) the weights ``catapult'' to a different basin.}
    \label{fig:toy-example}
\end{figure}
\cref{fig:toy-example} depicts this heavily simplified scenario. Early in training, optimizing this network with GD will rapidly move towards the minimum along this direction. In particular, until better features are learned, the direction of steepest descent will lead to a network which upweights the sky feature and predicts $p(\text{class } | \text{ sky})$ whenever it occurs. Once sufficiently close to the minimum,
the gradient will point ``through the valley'' towards amplifying the more relevant signal \citep{xing2018walk}. However, this will also cause the sky feature to grow in magnitude---as well as its \emph{potential} influence were the weights to be selectively perturbed, as described above. Both these factors contribute to progressive sharpening.

Here we emphasize the distinction between the loss on the \emph{outliers} and the full train loss. 
As images without sky are not nearly as sensitive to movement along this axis, their gradient and curvature is much smaller---and since they comprise the majority of the dataset, the global loss landscape may not at first be significantly affected.
Continued optimization will oscillate across the minimum with growing magnitude, but this growth may not be immediately apparent.
Furthermore, \emph{progress orthogonal to these oscillations need not be affected}---we find some evidence that these two processes occur somewhat independently, which we present in \cref{sec:sgd}.
Returning to the loss decomposition, we see that these oscillations will cause the losses to grow and \emph{alternate}, with one group having high loss and then the other. Eventually the outliers' loss increases sufficiently and the overall loss spikes, either flattening the valley and returning to the first phase, or ``catapulting'' to a different basin \citep{wu2018how, lewkowycz2020large, thilak2022slingshot}. This phenomenon is depicted in \cref{fig:main-teaser}. Finally, we note that if one visualizes the dynamics in \cref{fig:toy-example} from above---so the left/right direction on the page becomes up/down---it gives exactly the pattern of a network's weights projected onto the top eigenvector of the Hessian (e.g., \cref{fig:adam-sgd-proj} later in this work).

\paragraph{Verifying our toy examples's predictions.} Though this explanation lacks precise details, it does enable concrete predictions of network behavior during training. \cref{fig:sky-logits} tracks the predictions of a ResNet-18 on an image of sky---to eliminate possible confounders, we create a synthetic image as a single color block. Though the ``\texttt{plane} vs. \texttt{other}'' example seems almost \emph{too} simple, we see exactly the described behavior---initial convergence to the minimum along with rapid growth in feature norm, followed by oscillation in class probabilities. Over time, the network learns to use other signal and downweights the sky feature, as evidenced by the slow decay in feature norm. We reproduce this figure for many other inputs and for a VGG-11-BN in \cref{app:logit-track}, with similar findings.

Our example also suggests that \textbf{oscillation serves as a valuable regularizer that reduces reliance on easily learned opposing signals which may not generalize.} When a signal is used to the benefit of one group and the detriment of another, the advantaged group's loss goes down while the other's goes up, meaning the latter's gradient grows in magnitude while the former's shrinks. As the now gradient-dominating group is also the one disadvantaged by the use of this signal, the network will be encouraged to downweight this feature. In \cref{app:logit-track-small-lr} we reproduce \cref{fig:sky-logits} with a VGG-11-BN trained with a very small learning rate to closely approximate gradient flow. We see that gradient flow and GD are very similar until reaching the edge of stability. After this point, the feature norm under GD begins to slowly decay while oscillating; in contrast, in the absence of oscillation, the feature norms of opposing signals under gradient flow grow continuously. If it is the case that opposing signals represent ``simple'' features which generalize worse, this could help to explain the poor generalization of gradient flow. A similar effect was observed by \citet{Jastrzebski2020The}, who noted that large initial learning rate leads to a better-conditioned loss landscape later.

\begin{figure}[t]
    \centering
    \includegraphics[width=\linewidth]{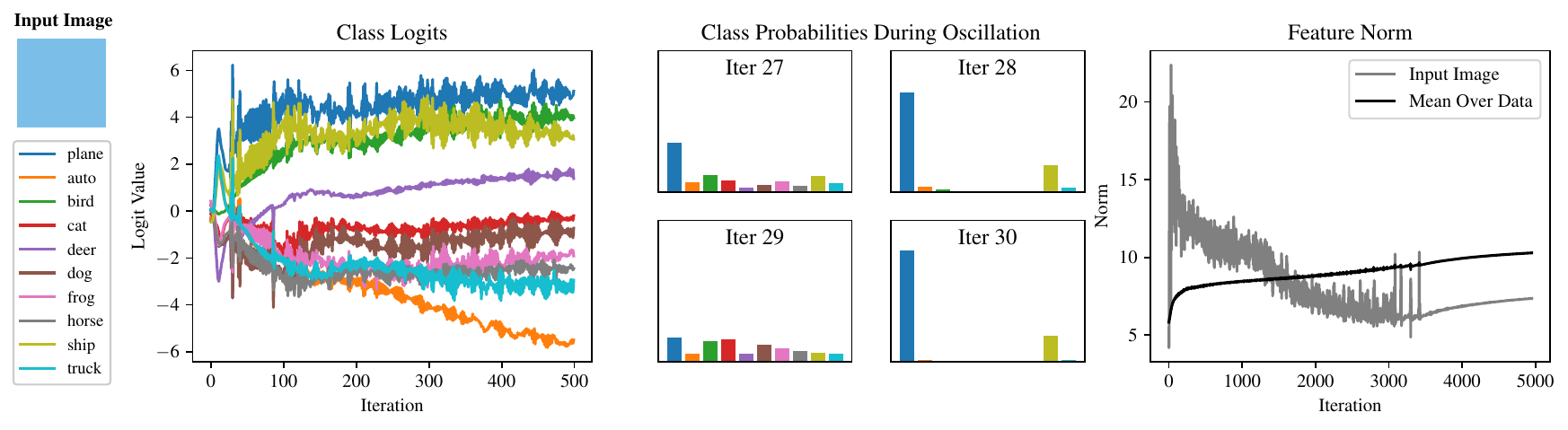}
    \caption{\textbf{Passing a sky-colored block through a ResNet during GD precisely tracks the predictions of our toy example.} \textbf{Left:} In the first phase, the network rapidly learns to use the sky feature to minimize loss. As signal is amplified, so too is the sky-colored input, and oscillation begins as depicted in \cref{fig:toy-example}. \textbf{Middle:} During oscillation, gradient steps alternate along the axis ``sky $=$ \texttt{plane}'' (and a bit \texttt{ship}). \textbf{Right:} The initial phase amplifies the sky input, causing rapid growth in feature norm. The network then oscillates, slowly learning to downweight this feature and rely on other signal (average feature norm provided for comparison).}
    \label{fig:sky-logits}
\end{figure}

\subsection{Theoretical Analysis of Opposing Signals in a Simple Model}
\label{sec:theory}

To demonstrate this effect formally, we study misspecified linear regression on inputs $x\in \R^d$ with a two-layer linear network. Though this model is simplified, it enables preliminary insight into the factors we think are most important for these dynamics to occur. Since we are analyzing the dynamics from initialization until the stability threshold, it will be sufficient to study the trajectory of \emph{gradient flow}---for reasonable step sizes $\eta$, a similar result then follows for gradient descent. Our analysis reproduces the initial phase of quickly reducing loss on the outliers, followed by the subsequent growth in sensitivity to the \emph{way} the opposing signal is used---i.e., progressive sharpening. We also verify this pattern (and the subsequent oscillation, which we do not formally prove) in experiments on real and synthetic data in \cref{sec:addl-figs}. 

We remark that one relevant factor which our model lacks is the concept of a ``partially useful'' signal as described at the end of \cref{subsec:prog-sharp}. This seems to require a somewhat more complex model to properly capture (e.g., multinomial logistic regression) so we view this analysis as an early investigation, capturing only part of relevant aspects of the phenomena we observe.

\paragraph{Model.} We model the observed features as a distribution over $x \in \R^{d_1}$, assuming only that its covariance $\Sigma$ exists---for clarity we treat $\Sigma = I$ in the main text. We further model an additional vector $x_o \in \R^{d_2}$ representing the opposing signal, with $d_2 \geq d_1$. We will suppose that on some small fraction of outliers $p \ll 1$, $x_o \sim \text{Unif}\left( \left\{\pm \sqrt{\frac{\alpha}{p d_2}} \mathbf{1} \right\} \right)$ ($\mathbf{1}$ is the all-ones vector) for some $\alpha$ which governs the feature magnitude, and we let it be $\mathbf{0}$ on the remainder of the dataset. We model the target as the linear function $y = \beta^\top x + \frac{1}{\sqrt{d_2}} \mathbf{1}^\top |x_o|$; this captures the idea that the signal $x_o$ correlates strongly with the target, but in opposing directions of equal strength. Finally, we parameterize the network with vectors $b\in\R^{d_1}, b_o \in \R^{d_2}$ and scalar $c$ in one single vector $\theta$, as $f_\theta(x) = c \cdot (b^\top x + b_o^\top x_o)$. Note the specific distribution of $x_o$ is unimportant---furthermore, in our simulations we observed the exact same pattern with cross-entropy loss. From our experiments and this analysis, it seems that depth and a small signal-to-noise ratio are the only elements needed for this behavior to arise.

\paragraph{Setup.} A standard initialization would be to sample $[b, b_o]^\top \sim \calN(0, \frac{1}{d_1+d_2} I)$, which would then imply highly concentrated distributions for the quantities of interest. As tracking the precise concentration terms would not meaningfully contribute to the analysis, we simplify by directly assuming that at initialization these quantities are equal to their expected order of magnitude: $\twonormsq{b} = \mathbf{1}^\top b = \frac{d_1}{d_1+d_2}$, $\twonormsq{b_o} = \mathbf{1}^\top b_o = \frac{d_2}{d_1+d_2}$, and $b^\top \beta = \frac{\|\beta\|}{\sqrt{d_1+d_2}}$. Likewise, we let $c=1$, ensuring that both layers have the same norm. We perform standard linear regression by minimizing the population loss $L(\theta) := \frac{1}{2} \E[(f_\theta(x) - y)^2]$. We see that the minimizer of this objective has $b_o = \mathbf{0}$ and $cb = \beta$. However, an analysis of gradient flow will elucidate how depth and strong opposing signals lead to sharpening as this minimum is approached.

\paragraph{Results.} In exploring progressive sharpening, \citet{cohen2021gradient} found that sometimes the model would have a brief \emph{decrease} in sharpness, particularly for square loss. In fact, this is consistent with our above explanation: for larger $\alpha$ and a sharper loss (e.g. the square loss), the network will initially prioritize minimizing loss on the outliers, thus heavily reducing sharpness. Our first result proves that this occurs in the presence of large magnitude opposing signals:

\begin{theorem}[Initial \emph{decrease} in sharpness]
    \label{thm:sharpness-decrease}
    Let $k := \frac{d_2}{d_1} \geq 1$, and assume $\|\beta\| > \max\left(\frac{d_1}{\sqrt{d_1 + d_2}}, \frac{24}{5} \right)$. At initialization, the sharpness $\|\nabla^2_\theta L(\theta)\|_2$ lies in $(\alpha, 3\alpha)$. Further, if $\sqrt{\alpha} = \Omega\left(\|\beta\| k\ln k \right)$, then both $b_o$ and the overall sharpness will \emph{decrease} as $\tilde{O}(e^{-\alpha t})$ from $t = 0$ until some time $t_1 \leq \frac{\ln \nicefrac{\|\beta\|}{2}}{2\|\beta\|}$.
\end{theorem}

Proofs can be found in \cref{app:proofs}. After this decrease, signal amplification can proceed---but this also means that the sharpness with respect to \emph{how the network uses the feature $x_o$} will grow, so a small perturbation to the parameters $b_o$ will induce a large increase in loss.

\begin{theorem}[Progressive sharpening]
    \label{thm:sharpness-increase}
    If $\sqrt{\alpha} = \Omega\left(1 + \|\beta\|^2 k\ln k \right)$, then at starting at time $t_1$ the sharpness will increase linearly in $\|\beta\|$ until some time $t_2 \geq \frac{1}{2\twonormsq{\beta}}$, reaching at least $\frac{5}{8} \|\beta\| \alpha$. This lower bound on sharpness applies to each dimension of $b_o$.
\end{theorem}

Oscillation will not occur during gradient flow---but for gradient descent with step size $\eta > \frac{16}{5\|\beta\|\alpha}$, $b_o$ will start to increase in magnitude while oscillating across the origin. If this growth continues, it will rapidly \emph{reintroduce} the feature, causing the loss on the outliers to grow and alternate. Such reintroduction (an example of which occurs around iteration 3000 in \cref{fig:sky-logits}) seems potentially helpful for exploration. In \cref{fig:verify-model} in the Appendix we simulate our model and verify exactly this sequence of events. We also show that an MLP trained on CIFAR-10 displays the same characteristic behavior.\looseness=-1

\subsection{Additional Findings}

\paragraph{Sharpness often occurs overwhelmingly in the first few layers.}
\cref{thm:sharpness-increase} shows that progressive sharpening occurs specifically in $b_o$. Generally, our model suggests that sharpness will begin in the last layer but that that early in training it will shift to the earlier layers since they have more capacity to redirect the signal. In \cref{app:sharpness-location} we track what fraction of curvature\footnote{The ``fraction of curvature'' is with respect to the top eigenvector of the loss Hessian. We partition this vector by network layer, so each sub-vector's squared norm represents that layer's contribution to the overall curvature.} of the top eigenvector lies in each layer of various networks during training with GD. In a ResNet-18, sharpness occurs almost exclusively in the first convolutional layer after the first few training steps; the same pattern appears more slowly while training a VGG-11. In a Vision Transformer curvature occurs overwhelmingly in the embedding layer and very slightly in the earlier MLP projection heads. The text transformer (NanoGPT) follows the same pattern, though with less extreme concentration in the embedding. Thus it does seem to be the case that earlier layers have the most significant sharpness---especially if they perform dimensionality reduction or have particular influence over how the signal is propagated to later layers. This seems the likely cause of large gradients in the early layers of vision models \citep{chen2021empirical, kumar2022how}, suggesting that this effect is equally influential during finetuning and pretraining and that further study can improve optimization.

\paragraph{Batchnorm may smooth training, even if not the loss itself.} 
\citet{cohen2021gradient} noted that batchnorm (BN) \citep{ioffe2015batch} does not prevent networks from reaching the edge of stability and concluded, contrary to \citet{santurkar2018does}, that BN does not smooth the loss landscape. We conjecture that the benefit of BN may be in downweighting the influence of opposing signals and mitigating this oscillation. In other words, BN may smooth the \emph{optimization trajectory} of neural networks, rather than the loss itself (this is consistent with the distinction made by \citet{cohen2021gradient} between regularity and smoothness).
In \cref{sec:sgd} we demonstrate that Adam \emph{also} smooths the optimization trajectory and that minor changes to emulate this effect can aid stochastic optimization.
We imagine that the effect of BN could also depend on the use of GD vs. SGD. Specifically, our findings hint at a possible benefit of BN which applies only to SGD: reducing the variance of imbalanced opposing signals across random minibatches.

\paragraph{For both GD and SGD, approximately half of training points go up in loss on each step.} Though only the outliers are wildly oscillating, many more images contain some small component of the features they exemplify. \cref{fig:frac-loss-increase} in the Appendix shows that the fraction of points which increase in loss hovers around 50\% for every step---to some extent, a small degree of oscillation appears to be happening to the entire dataset.

\paragraph{Different losses have different effects on sharpening.} Our model would predict that adding label smoothing to the cross-entropy loss should reduce sharpening, because smoothing reduces loss curvature under extreme overconfidence. Indeed, \citet{macdonald2023progressive} show this to be the case. This also hints at why logistic loss may be more suitable for NN optimization, because it only has substantial curvature around $x=0$ ($x$ being the logit, i.e. when prediction entropy is high), so unlike square or exponential loss, large magnitude features will not massively increase sharpness.
We expect a similar property could contribute to the relative behavior of different activations (e.g. ReLU or tanh).\looseness=-1

\section{The Interplay of Opposing Signals and Stochasticity}
\label{sec:sgd}

Full-batch GD is not used in practice when training NNs. It is therefore pertinent to ask what these findings imply about stochastic optimization. We begin by verifying that this pattern persists during SGD. \cref{fig:vgg-sgd} displays the losses for four opposing group pairs of a VGG-11-BN trained on CIFAR-10 with SGD batch size 128. We observe that the paired groups do exhibit clear opposing oscillatory patterns, but they do not alternate with every step, nor do they always move in opposite directions. This should not be surprising: we expect that not every batch will have a given signal in one direction or the other. For comparison, we include the \emph{full} train loss in each figure---that is, including the points not in the training batch. We see that the loss on the outliers has substantially larger variance; to confirm that this is not just because the groups have many fewer samples, we also plot the loss on a random subset of training points of the same size. We reproduce this plot with a VGG-11 without BN in \cref{fig:vggbn-sgd-groups} in the Appendix. 
\begin{figure}[t]
    \centering
    \begin{subfigure}{.45\linewidth}
        \includegraphics[width = \linewidth]{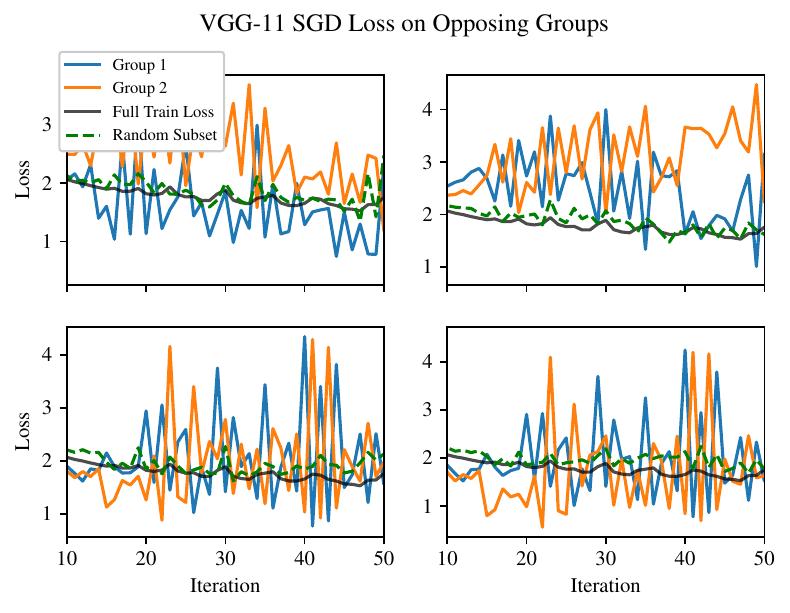}
        \caption{}
        \label{fig:vgg-sgd}
    \end{subfigure}
    \begin{subfigure}{.47\linewidth}
        \includegraphics[width = \linewidth]{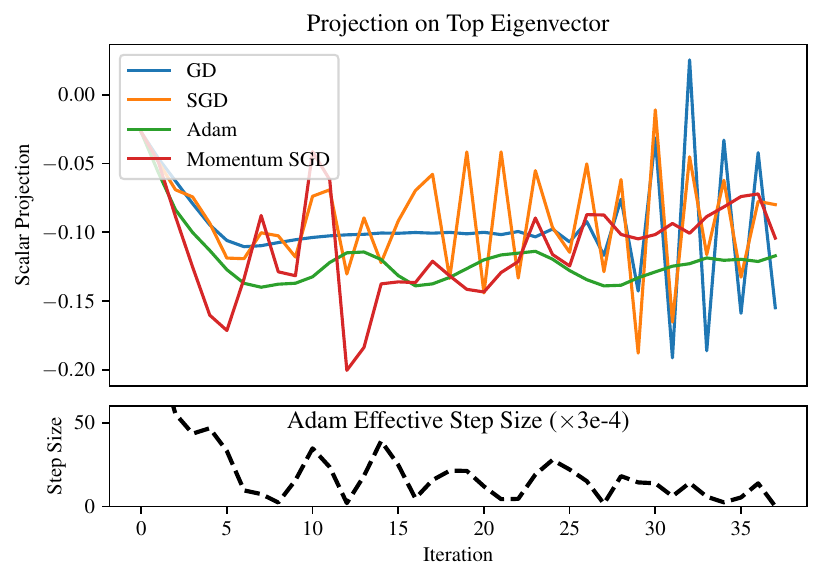}
        \caption{}
        \label{fig:adam-sgd-proj}
    \end{subfigure}
    \caption{\textbf{Outliers with opposing signals have a significant influence even during SGD. Left:} We plot the losses of paired outlier groups on a VGG-11-BN trained on CIFAR-10, along with the full train loss for comparison. Modulo batch randomness, the outliers' loss follow the same oscillatory pattern with large magnitude. See appendix for the same without batchnorm. \textbf{Right (top):} We train a small MLP on a 5k subset of CIFAR-10 with various optimizers and project the iterates onto the top Hessian eigenvector. SGD closely tracks GD, bouncing across the valley; momentum somewhat mitigates the sharp jumps. Adam smoothly oscillates along one side. \textbf{Right (bottom):} Adam's effective step size drops sharply when moving too close or far from the valley floor. }
\end{figure}

Having verified that this behavior occurs in the stochastic setting, we conjecture that current best practices for neural network optimization owe much of their success to how they handle opposing signals. As a proof of concept, we will make this more precise with a preliminary investigation of the Adam optimizer \citep{kingma2014adam}.

\subsection{How Adam Handles Gradients with Opposing Signals}

To better understand their differences, \cref{fig:adam-sgd-proj} visualizes the parameter iterates of Adam and SGD with momentum on a ReLU MLP trained on a 5k subset of CIFAR-10, alongside those of GD and SGD (all methods use the same initialization and sequence of training batches).
The top figure is the projection of these parameters onto the top eigenvector of the loss Hessian of the network trained with GD, evaluated at the first step where the sharpness crosses $\nicefrac{2}{\eta}$. We observe that SGD tracks a similar path to GD, though adding momentum  mitigates the oscillation somewhat. In contrast, the network optimized with Adam markedly departs from this pattern, smoothly oscillating along one side. We identify three components of Adam which potentially contribute to this effect:

\paragraph{Advantage 1: Smaller steps along high curvature directions.} Adam's normalization causes smaller steps along the top eigenvector, especially near the minimum. The lower plot in \cref{fig:adam-sgd-proj} shows that the effective step size in this direction---i.e., the absolute inner product of the parameter-wise step sizes and the top eigenvector---rapidly drops to zero as the iterates approach the valley floor (in the opposite direction, the gradient negates the momentum for the same effect). We conjecture that general normalization may not be essential to Adam's performance; 
we even expect it could be somewhat harmful by limiting exploration. On the other hand, normalizing steps by curvature \emph{parameter-wise} does seem important; \citet{pan2023toward} argue the same and show that parameter-wise gradient clipping improves SGD substantially. We highlight why this may be useful in the next point.\looseness=-1

\paragraph{Advantage 2: Managing heavy-tailed gradients and avoiding steepest descent.} \citet{zhang2020adaptive} identified the ``trust region'' as an important contributor to Adam's success in attention models, pointing to heavy-tailed noise in the stochastic gradients. More recently, \citet{kunstner2023noise} argued that Adam's superiority does not come from better handling noise, which they supported by experimenting with large batch sizes. Our result reconciles these contradictory claims by showing that \textbf{the difficulty is not heavy-tailed \emph{noise}, but strong, directed (and perhaps imbalanced) opposing signals.} Unlike traditional ``gradient noise'', larger batch sizes may not reduce the effect of these signals---that is, the gradient is heavy-tailed (across parameters) even without being stochastic. Furthermore, the largest steps emulate Sign SGD, which is notably \emph{not} a descent method. \cref{fig:adam-sgd-proj} shows that Adam's steps are more parallel to the valley floor than those of steepest descent. \textbf{Thus it seems advantageous to \emph{intentionally} avoid steps along the gradient which point towards the local minimum,} which might lead to over-reliance on these features. Indeed, \citet{benzing2022gradient} observe that true second order methods perform worse than SGD on NNs, and \citet{kunstner2023noise} show that Adam shares some behavior with Sign SGD with momentum. 
This point is also consistent with the observed generalization benefits of a large learning rate for SGD on NNs \citep{Jastrzebski2020The}; in fact, opposing signals naturally fit the concept of ``easy-to-generalize'' features as modeled by \citet{li2019towards}.\looseness=-1

\paragraph{Advantage 3: Dampening.} Lastly, Adam's third important factor: downweighting the most recent gradient. Traditional SGD with momentum $\beta < 1$ takes a step which weights the current gradient by $\frac{1}{1+\beta} > \frac{1}{2}$.
Though this makes intuitive sense, our results imply that heavily weighting the most recent gradient can be problematic. Instead, we expect an important addition is \emph{dampening}, which multiplies the stochastic gradient at each step by some $(1-\tau) < 1$. We observe that Adam's (unnormalized) gradient is equivalent to SGD with momentum and dampening both equal to $\beta_1$, plus a debiasing step. Recently proposed alternatives also include dampening in their momentum update but do not explicitly identify the distinction \citep{zhang2020adaptive, pan2023toward, chen2023symbolic}.\looseness=-1

\subsection{Proof of Concept: Using These Insights to Aid Stochastic Optimization}

To test whether our findings translate to practical gains, we design a variant of SGD which incorporates these insights. First, we use dampening $\tau=0.9$ in addition to momentum. Second, we choose a global threshold: if the gradient magnitude for a parameter is above this threshold, we take a fixed step size; otherwise, we take a gradient step as normal. The exact method appears in \cref{app:adam-vs-sgd}. \textbf{We emphasize that our goal here is not to propose a new optimization algorithm.} Instead, we are exploring the potential value gained from knowledge of the existence and influence of opposing signals.\looseness=-1

Results in \cref{app:adam-vs-sgd} show that this approach matches Adam when training ResNet-56/110 on CIFAR-10 with learning rates for the unthresholded parameters across several orders of magnitude ranging from $10^{-4}$ to $10^3$. Notably, the fraction of parameters above the threshold (whose step size is fixed) is only around 10-25\% per step. This implies that the trajectory and behavior of the network is dominated by this small fraction of parameters; the remainder can be optimized much more robustly, but their effect on the network's behavior is obscured. We therefore see the influence of opposing signals as a possible explanation for the ``hidden'' progress in grokking \citep{barak2022hidden, nanda2023progress}.
We also compare this method to Adam for the early phase of training GPT-2 \citep{radford2019language} on the OpenWebText dataset \citep{gokaslan2019openwebtext}---not only do they perform the same, their loss similarity suggests that their exact trajectory may be very similar (\cref{appsec:split-gpt2}). Here the fraction of parameters above the threshold hovers around 50\% initially and then gradually decays. The fact that many more parameters in the attention model are above the threshold suggests that the attention mechanism is more sensitive to opposing signals and that further investigation of how to mitigate this instability may be fruitful.

\section{Discussion and Future Work}
\label{sec:discussion}
Many of the observations we make in this paper are not new, having been described in various prior works. Rather, this work identifies a possible \emph{higher-order cause} which neatly ties these findings together. There are also many works which pursue a more theoretical understanding of each of these phenomena independently. Such analyses begin with a set of assumptions (on the data, in particular) and prove that the given behavior follows. In contrast, this work \emph{begins} by identifying a condition---the presence of opposing signals---which we argue is likely a major cause of these behaviors. These two are not at odds: we believe in many cases our result serves as direct evidence for the validity of these modeling assumptions and that it may enable even more fine-grained analyses. This work provides an initial investigation which we hope will inspire future efforts towards a more complete understanding.

We now highlight some connections to these earlier findings. More general related work can be found in \cref{app:related-work}.

\paragraph{Heavy-tailed loss spectrum.} Earlier studies of the loss landscape noted a small group of very large outlier Hessian eigenvalues or Jacobian singular values (e.g. \citealt{sagun2016eigenvalues, sagun2017empirical, papyan2018full}, see \cref{app:related-work} for more). Our method of identifying these paired groups, along with the metrics tracked in \cref{fig:interest-stats}, indicate that these outlier directions in the spectrum are precisely the directions with opposing signals in the gradient and that this pattern may be key to better understanding the generalization ability of NNs trained with SGD.

\paragraph{Progressive sharpening and the edge of stability.}
More recent focus has shifted to the top Hessian eigenvalue(s), where it was empirically observed that their magnitude (the loss ``sharpness'') grows during training \citep{jastrzebski2018on, Jastrzebski2020The, cohen2021gradient} (so-called \emph{progressive sharpening}), leading to rapid oscillation in weight space \citep{xing2018walk, jastrzebski2018on}. \citet{cohen2021gradient} also found that for GD this coincides with a consistent yet non-monotonic decrease in training loss over long timescales, which they named the \emph{edge of stability}.
We observe that prior analyses have proven the \emph{occurrence} of progressive sharpening and the edge of stability under various assumptions \citep{arora2022understanding, wang2022analyzing}, but the underlying \emph{cause} has not been made clear. Our discussion, experiments, and theoretical analysis in \cref{sec:main-explanation} provide strong evidence for a genuine cause which aligns with several of these existing modeling assumptions. Roughly, our results seem to imply that progressive sharpening occurs when the network learns to rely on (or \emph{not} rely on) opposing signals in a very specific way, while simultaneously amplifying overall sensitivity. This growth in sensitivity means a small parameter change modifying how opposing signals are used can massively increase loss. This leads to intermittent instability orthogonal to the ``valley floor'', accompanied by gradual training loss decay and occasional spikes as described by the toy example in \cref{fig:toy-example} and depicted on real data in \cref{fig:main-teaser}. Empirically, this oscillation seems somewhat independent of movement parallel to the floor (see \cref{app:adam-vs-sgd}), but further study of the precise dynamics is needed.\looseness=-1

\paragraph{Spurious correlations, grokking, and slingshotting.} In images, the features corresponding to opposing signals match the traditional picture of ``spurious correlations'' surprisingly closely---it could be that a network maintaining balance or diverging along a direction also determines whether it continues to use a ``spurious'' feature or is forced to find an alternative way to minimize loss. Indeed, the exact phenomenon of a network ``slingshotting'' to a new region with improved generalization has been directly observed \citep{wu2018how, lewkowycz2020large, jastrzebski2021catastrophic, thilak2022slingshot}. \emph{Grokking} \citep{power2022grokking}, whereby a network learns to generalize long after memorizing the training set, is closely related. Several works have shown that grokking is a ``hidden'' phenomenon, with gradual amplification of generalizing subnetworks \citep{barak2022hidden, nanda2023progress, merrill2023tale}; it has even been noted to co-occur with weight oscillation \citep{notsawo2023predicting}. Our experiments in \cref{sec:sgd,app:adam-vs-sgd} show that the influence of opposing signals obscures the behavior of the rest of the network, offering one possible explanation.

\paragraph{Simplicity bias and double descent.} \citet{nakkiran2019sgd} observed that NNs learn functions of increasing complexity throughout training. Our experiments---particularly the slow decay in the norm of the feature embedding of opposing signals---lead us to believe it would be more correct to say that they \emph{unlearn} simple functions, which enables more complex subnetworks with smaller magnitude and better performance to take over. At first this seems at odds with the notion of \emph{simplicity bias} \citep{valle-perez2018deep, shah2020pitfalls}, defined broadly as a tendency of networks to rely on simple functions of their inputs. However, it does seem to be the case that the network will use the simplest (e.g., largest norm) features that it can, so long as such features allow it to approach zero training loss; otherwise it may eventually diverge. This tendency also suggests a possible explanation for \emph{double descent} \citep{belkin2019reconciling, Nakkiran2020Deep}:
even after interpolation, the network pushes towards greater confidence and the weight layers continue to balance \citep{saxe2013exact, du2018algorithmic}, increasing sharpness. This could lead to oscillation, pushing the network to learn new features which generalize better \citep{wu2018how, li2019towards, rosenfeld2022domain, thilak2022slingshot}. This behavior would also be more pronounced for larger networks because they exhibit greater sharpening. Note that the true explanation is not quite so straightforward: generalization is sometimes improved via methods that \emph{reduce} oscillation (like loss smoothing), implying that this behavior is not always advantageous. A better understanding of these nuances is an important subject for future study.\looseness=-1

\paragraph{Sharpness-Aware Minimization} Another connection we think merits further inquiry is Sharpness-Aware Minimization (SAM) \citep{foret2021sharpnessaware}, which is known to improve generalization of neural networks for reasons still not fully understood \citep{wen2023sharpness}. In particular, the better-performing variant is 1-SAM, which takes positive gradient steps on each training point in the batch individually. It it evident that several of these updates will point along directions of steepest descent/ascent orthogonal to the valley floor (and, if not normalized, the updates may be \emph{very} large). Thus it may be that 1-SAM is in some sense ``simulating'' oscillation and divergence out of this valley in both directions, enabling exploration in a manner that would not normally be possible until the sharpness grows large enough---these intermediate steps would also encourage the network to downweight these features sooner and faster. In contrast, standard SAM would only take this step in one of the two directions, or perhaps not at all if the opposing signals are equally balanced. Furthermore, unlike 1-SAM the intermediate step would blend together all opposing signals in the minibatch.
These possibilities seem a promising direction for further exploration.

\section{Conclusion}
The existence of groups of training data with such a significant yet non-obvious influence on neural network training raises as many questions as it answers. This work presents an initial investigation into the effect of opposing signals on various aspects of optimization, but there is still much more to understand. Though it is clear they have a large influence on training, less obvious is whether reducing their influence is \emph{necessary} for improved optimization or simply coincides with it. At the same time, there is evidence that the behavior these signals induce may serve as an important method of exploration and/or regularization. If so, another key question is whether these two effects can be decoupled---or if the incredible generalization ability of neural networks is somehow inherently tied to their optimization instability.

\newpage
\section*{Acknowledgements}
We thank Saurabh Garg for detailed feedback on an earlier version of this work. Thanks also to Christina Baek and Bingbin Liu for helpful comments and Jeremy Cohen for pointers to related work. This research is supported in part by NSF awards IIS-2211907, CCF-2238523, an Amazon Research Award, and the CMU/PwC DT\&I Center.

\bibliography{arxiv}
\bibliographystyle{plainnat}

\newpage
\appendix
\clearpage
\section{Related Work}
\label{app:related-work}

\paragraph{Characterizing the NN loss landscape.} Earlier studies of the loss landscape commonly identified a heavy-tailedness with a small group of very large outlier Hessian eigenvalues or Jacobian singular values \citep{sagun2016eigenvalues, sagun2017empirical, papyan2018full, oymak2019generalization, papyan2019measurements, fort2019emergent, ghorbani2019investigation, li2020hessian, papyan2020traces, kopitkov2020neural}. Later efforts focused on concretely linking these observations to corresponding behavior, often with an emphasis on SGD's bias towards particular solutions \citep{wu2018how, jastrzkebski2017three, Jastrzebski2020The} and what this may imply about its resulting generalization \citep{jastrzebski2018on, zhu2019anisotropic, wu2022alignment}. Our method for identifying these paired groups, along with \cref{fig:interest-stats}, indicates that these outlier directions in the Hessian/Jacobian spectrum are precisely the directions with opposing signals in the gradient, and that this pattern may be key to better understanding the generalization ability of NNs trained with SGD.

\paragraph{Progressive sharpening and the edge of stability.}
Shifting away from the overall structure, more recent focus has been specifically on top eigenvalue(s), where it was empirically observed that their magnitude (the loss ``sharpness'') grows when training with SGD \citep{jastrzebski2018on, Jastrzebski2020The} and GD \citep{kopitkov2020neural, cohen2021gradient} (so-called ``progressive sharpening''). This leads to rapid oscillation in weight space \citep{xing2018walk, jastrzebski2018on, cohen2021gradient, cohen2022adaptive}. \citet{cohen2021gradient} also found that for GD this coincides with a consistent yet non-monotonic decrease in training loss over long timescales, which they named the ``edge of stability''; moreover, they noted that this behavior runs contrary to our traditional understanding of NN convergence.
Many works have since investigated the possible origins of this phenomenon \citep{zhu2023understanding, kreisler2023gradient}. Several of these are deeply related to our findings: \citet{ma2022beyond} connect this behavior to the existence of multiple ``scales'' of losses; the outliers we identify corroborate this point. \citet{damian2022selfstabilization} prove that GD implicitly regularizes the sharpness---we identify a conceptually distinct source of such regularization, as described in \cref{sec:main-explanation}. \citet{arora2022understanding} show under some conditions that the GD trajectory follows a minimum-loss manifold towards lower curvature regions. This is consistent with our findings, and we believe this manifold to be precisely the path which evenly balances the opposing gradients. \citet{wang2022analyzing} provide another thorough analysis of NN training dynamics at the edge of stability; their demonstrated phases closely align with our own. They further observe that this sharpening coincides with a growth in the norm of the last layer, which was also noted by \citet{macdonald2023progressive}. Our proposed explanation for the effect of opposing signals offers some insight into this relationship, but further investigation is needed.
\clearpage

\section{Examples of Opposing Signals in Text}
\label{app:text-signals}
\begin{center}  
\begin{tcolorbox}[width=\textwidth, nobeforeafter, title=Punctuation Ordering]
\small\texttt{the EU is “the best war-avoidance mechanism ever invented[”]. \\ 
the 2008 economic crash and in doing so “triggered a crisis of rejection[”]. \\ 
Don’t farm this thing out under the guise of a “contest[”]. \\
I really thought she was going to use another C-word besides “coward[”]. \\
because it was one of the few that still “dry-farmed[”]. \\
He describes the taste as “almost minty[”]. \\
I did receive several offers to “help out a bit[”]. \\
we can from our investments, regardless of the costs to the rest of society[”]. \\
Or “it won’t make a difference anyway[”]. \\
Nor is it OK to say “the real solution is in a technological breakthrough[”]. \\
and that’s what they mean by “when complete[”]. \\
and that the next big investment bubble to burst is the “carbon bubble[”]. \\
he had been “driven by ideological and political motives[”]. \\
Prime Minister Najib Razak’s personal bank account was a “genuine donation[”]. \\
exceptional intellect, unparalleled integrity, and record of independence[”]. \\
was the “most consequential decision I’ve ever been involved with[”]. \\
which some lawmakers have called the “filibuster of all filibusters[”]. \\
Democrats vowed to filibuster what some openly called a “stolen seat[”].
}
\tcbline
\small\texttt{His leather belt was usually the delivery method of choice.[”] \\
I used to catch me a few and make pets out of them.[”] \\
catch pneumonia because I got my underwear on, but Bob here is naked.'\ [”] \\
My husband-to-be built a gun cabinet. It was that type of community: normal.[”] \\
‘use both hands.’ That was Ricky to a tee. He’s a jokester.[”] \\
He added, with a half-smile, “I’m guessing he didn’t mean the drinking.[”] \\
on the county road. If you can’t get to it, that doesn’t make sense.’ [”] \\
“It’s laborious and boring. He loved excitement and attention.[”] \\
‘That little son-of-a-gun is playing favorites,’ and turned it against him.[”] \\
“For medicinal purposes, for medical purposes, absolutely, it’s fine.[”] \\
“That’s a choice that growers make. It’s on their side of the issue.[”] \\
to be a good steward of your land. You have to make big decisions in a hurry.[”] \\
“But of course modern farming looks for maximum yield no matter what you have to put in. And in the case of California, that input is water.[”] \\
And we have been drawing down on centuries of accumulation. Pretty soon those systems are not going to be able to provide for us.[”] \\
not a luxury crop like wine. I’m really excited ab
out the potential.[”] \\
“We knew and still believe that it was the right thing to do.[”] \\
spending lots of time in the wind tunnel, because it shows when we test them.[”] \\
Compliance was low on the list, but I think it’s a pretty comfortable bike.[”] \\
the opportunity to do that.’ I just needed to take it and run with it.[”] \\
}
\end{tcolorbox}
\noindent\begin{minipage}{\textwidth}
\captionof{figure}{\textbf{Examples of opposing signals in text.} Found by training GPT-2 on a subset of OpenWebText. Sequences are on separate lines, the token in brackets is the target and all prior tokens are (the end of the) context. As both standards are used, it is not always clear whether punctuation will come before or after the end of a quotation (we include the period after the quote for clarity---the model does not condition on it). Note that the double quotation is encoded as the \emph{pair} of tokens \texttt{[447, 251]}, and the loss oscillation is occurring for sequences that end with this pair, either before (top) or after (bottom) the occurrence of the period token (\texttt{13}).}
\end{minipage}

\begin{tcolorbox}[width=\textwidth, nobeforeafter, title=New Line or `the' After Colon]
\small\texttt{In order to prepare your data, there are three things to do:[\textbackslash n] \\
in the FP lib of your choice, namely Scalaz or Cats. It looks like this:[\textbackslash n] \\
Let the compiler guide you, it will only accept one implementation:[\textbackslash n] \\
Salcedo said of the work:[\textbackslash n] \\
Enter your email address:[\textbackslash n] \\
According to the CBO update:[\textbackslash n] \\
Here's how the Giants can still make the playoffs:[\textbackslash n] \\
described how he copes with his condition in an interview with The Telegraph:[\textbackslash n] \\
Here’s a list of 5 reasons as to why self diagnosis is valid:[\textbackslash n] \\
successive Lambda invocations. It looks more or less like this:[\textbackslash n] \\
data, there are three things to do:[\textbackslash n] \\
4.2 percent in early 2018.\textbackslash n\textbackslash nAccording to the CBO update:[\textbackslash n] \\
other than me being myself.”\textbackslash n\textbackslash nWATCH:[\textbackslash n] \\
is to make the entire construction plural.\textbackslash n\textbackslash nTwo recent examples:[\textbackslash n] \\
We offer the following talking points to anyone who is attending the meeting:[\textbackslash n] \\
is on the chopping block - and at the worst possible moment:[\textbackslash n]
}
\tcbline
\small\texttt{as will our MPs in Westminster. But to me it is obvious: [the] \\
The wheelset is the same as that on the model above: [the] \\
not get so engrained or in a rut with what I had been doing. Not to worry: [the] \\
polemics against religion return in various ways to one core issue: [the] \\
which undergirds all other acts of love, both divine and human: [the] \\
integrate fighters from the Kurds’ two main political parties: [the] \\
robs this incredible title of precisely what makes it so wonderful: [the] \\
you no doubt noticed something was missing: [the] \\
Neil Gorsuch's 'sexist' comments on maternity leave: [the]
}
\end{tcolorbox}
\end{center}

\noindent\begin{minipage}{\textwidth}
\captionof{figure}{\textbf{Examples of opposing signals in text.} Found by training GPT-2 on a subset of OpenWebText. Sequences are on separate lines, the token in brackets is the target and all prior tokens are (the end of the) context. Sometimes a colon occurs mid-sentence---and is often followed by ``the''---other times it announces the start of a new line. The model must \emph{unlearn} ``: $\mapsto$ \texttt{[\textbackslash n]}'' versus ``: $\mapsto$ \texttt{[the]}'' and instead use other contextual information.}
\end{minipage}

\begin{figure}[ht!]
    \centering
    \includegraphics[width=.8\linewidth]{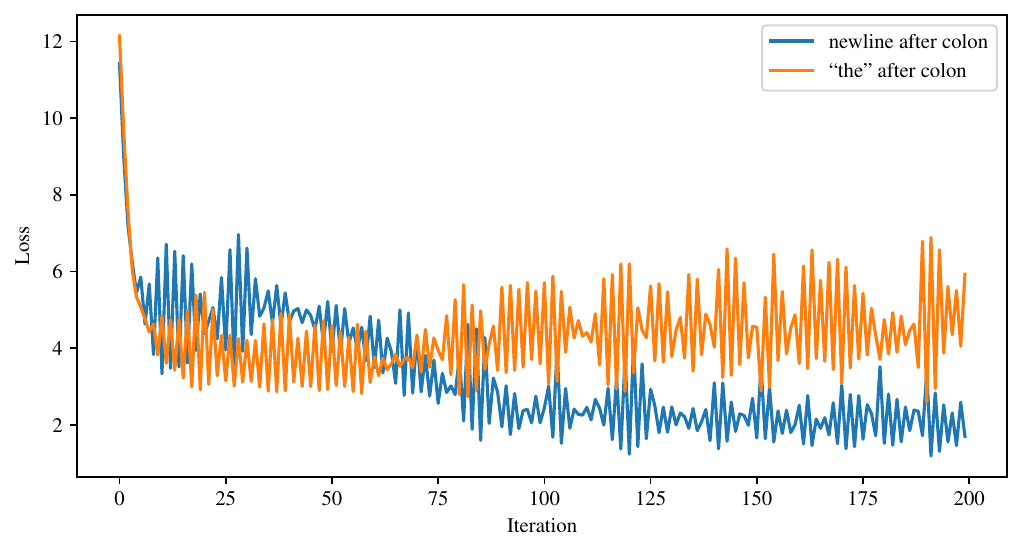}
    \caption{Loss of GPT-2 on the above opposing signals.}
\end{figure}

\clearpage
\section{Reproducing \cref{fig:sky-logits} in Other Settings}
\label{app:logit-track}

Though colors are straightforward, for some opposing signals such as grass texture it is not clear how to produce a synthetic image which properly captures what precisely the model is latching on to. Instead, we identify a real image which has as much grass and as little else as possible, with the understanding that the additional signal in the image could affect the results. We depict the grass image alongside the plots it produced.

\subsection{ResNet-18 Trained with GD on Other Inputs}

\begin{figure}[ht!]
    \centering
    \includegraphics[width=\linewidth]{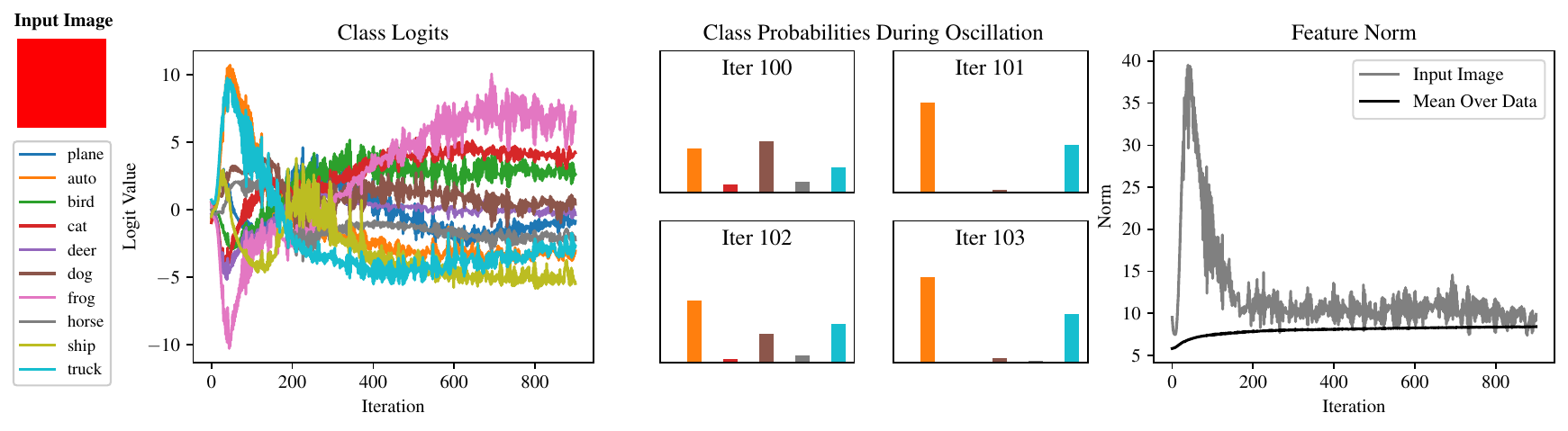}
    \caption{ResNet-18 on a red color block.}
\end{figure}

\begin{figure}[ht!]
    \centering
    \includegraphics[width=\linewidth]{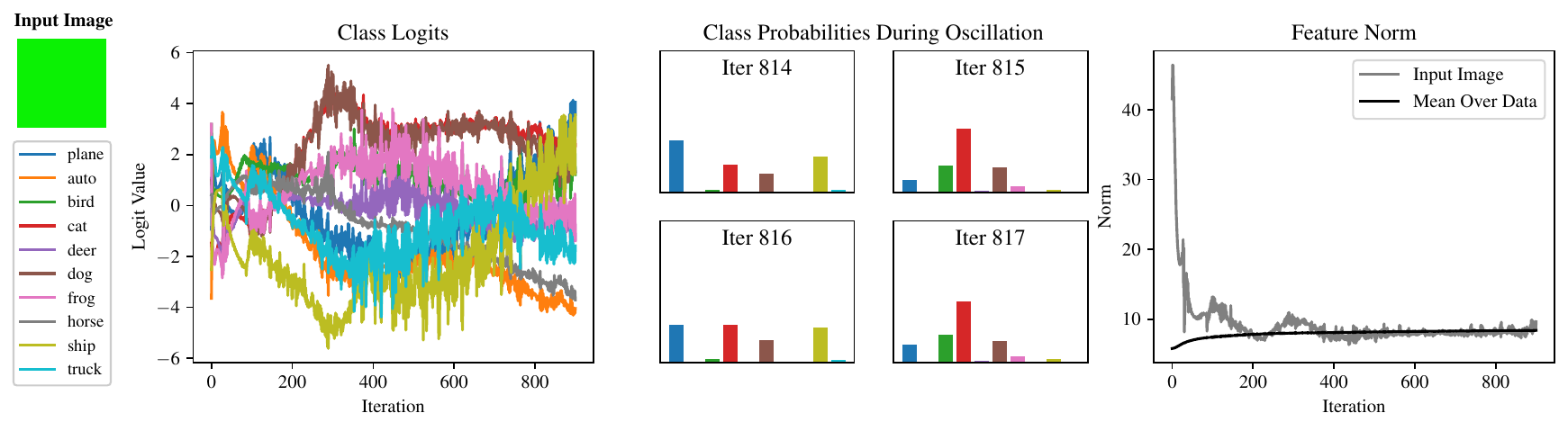}
    \caption{ResNet-18 on a green color block. As this color seems unnatural, we've included two examples of relevant images in the dataset.}
\end{figure}

\begin{figure}[ht!]
    \centering
    \begin{subfigure}{.1\linewidth}
        \includegraphics[width = \linewidth]{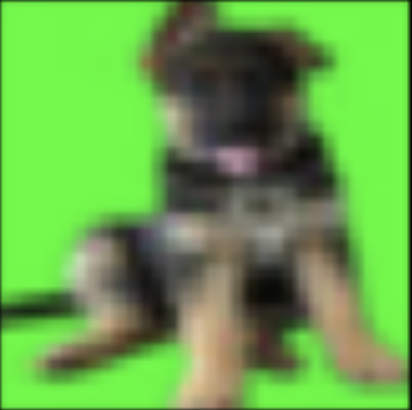}
    \end{subfigure}
    \begin{subfigure}{.1\linewidth}
        \includegraphics[width = \linewidth]{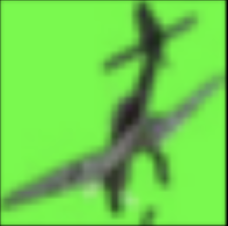}
    \end{subfigure}
    \caption{Examples of images with the above green color.}
\end{figure}

\begin{figure}[ht!]
    \centering
    \includegraphics[width=\linewidth]{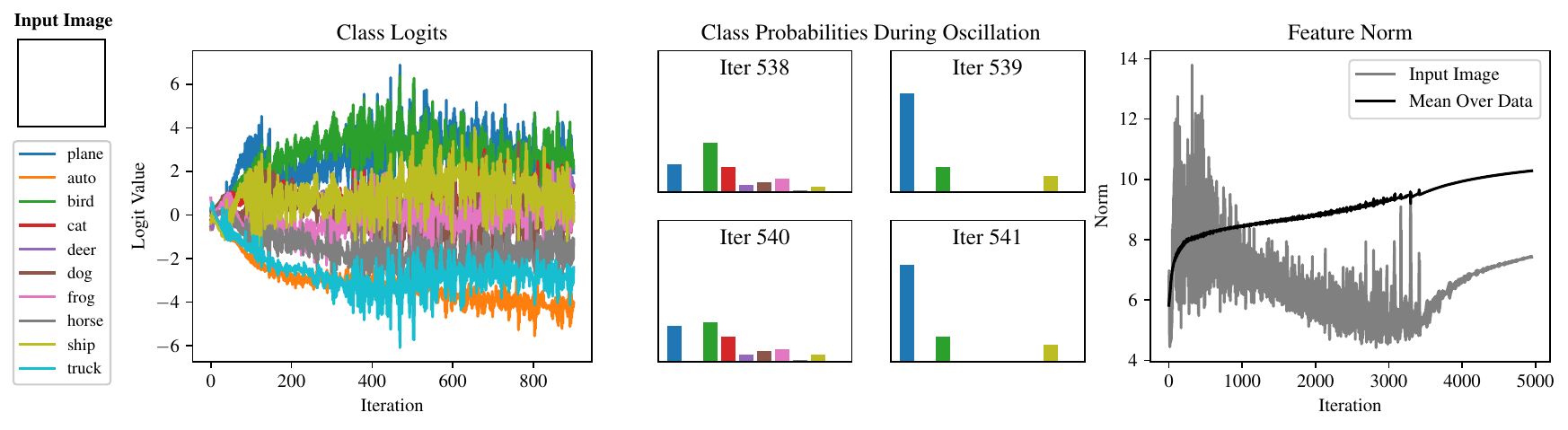}
    \caption{ResNet-18 on a white color block.}
\end{figure}

\begin{figure}[ht!]
    \centering
    \includegraphics[width=\linewidth]{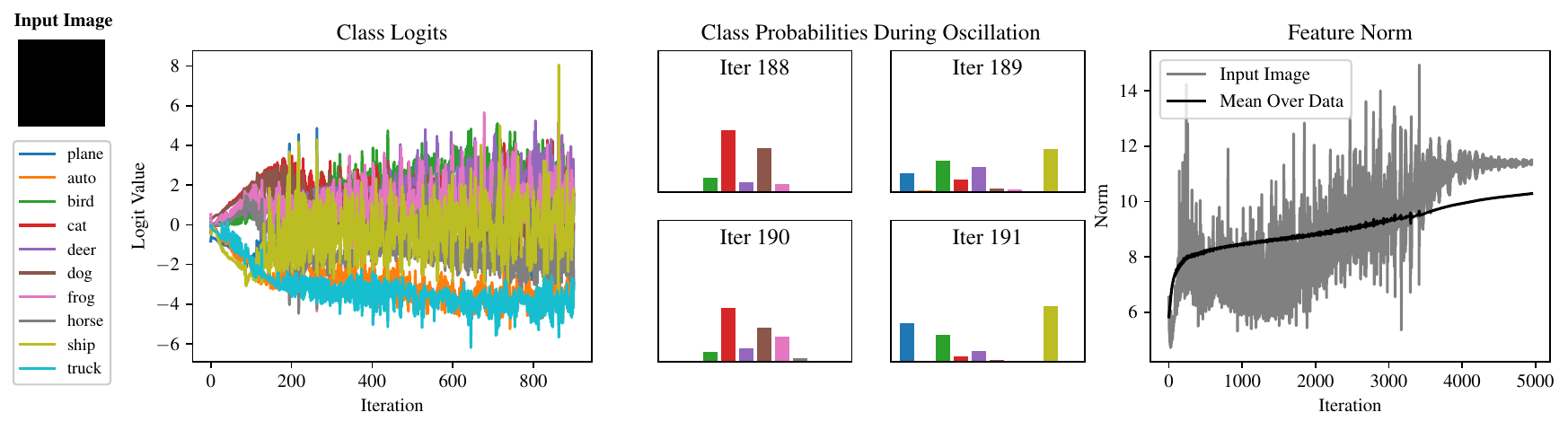}
    \caption{ResNet-18 on a black color block.}
\end{figure}

\begin{figure}[ht!]
    \centering
    \begin{subfigure}{.1\linewidth}
        \includegraphics[width = \linewidth]{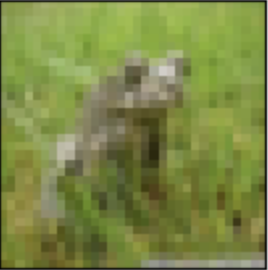}
    \end{subfigure}
    \begin{subfigure}{.89\linewidth}
        \includegraphics[width = \linewidth]{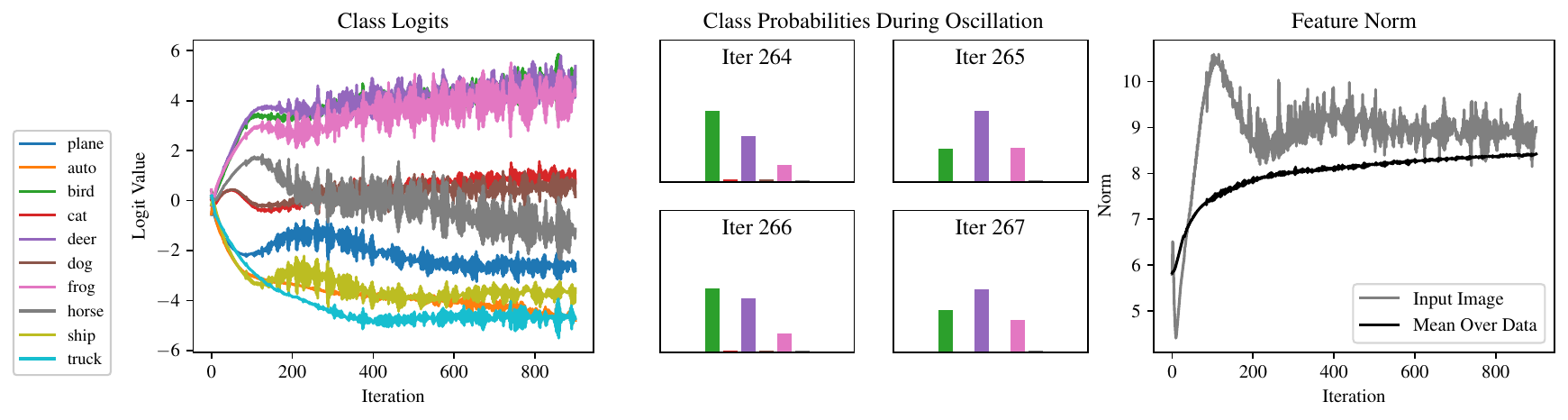}
    \end{subfigure}
    \caption{ResNet-18 on an image with mostly grass texture.}
\end{figure}

\clearpage
\subsection{VGG-11-BN Trained with GD}

For VGG-11, we found that the feature norm of the embedded images did not decay nearly as much over the course of training. We expect this has to do with the lack of a residual component. However, for the most part these features do still follow the pattern of  a rapid increase, followed by a marked decline.

\begin{figure}[h!]
    \centering
    \includegraphics[width=\linewidth]{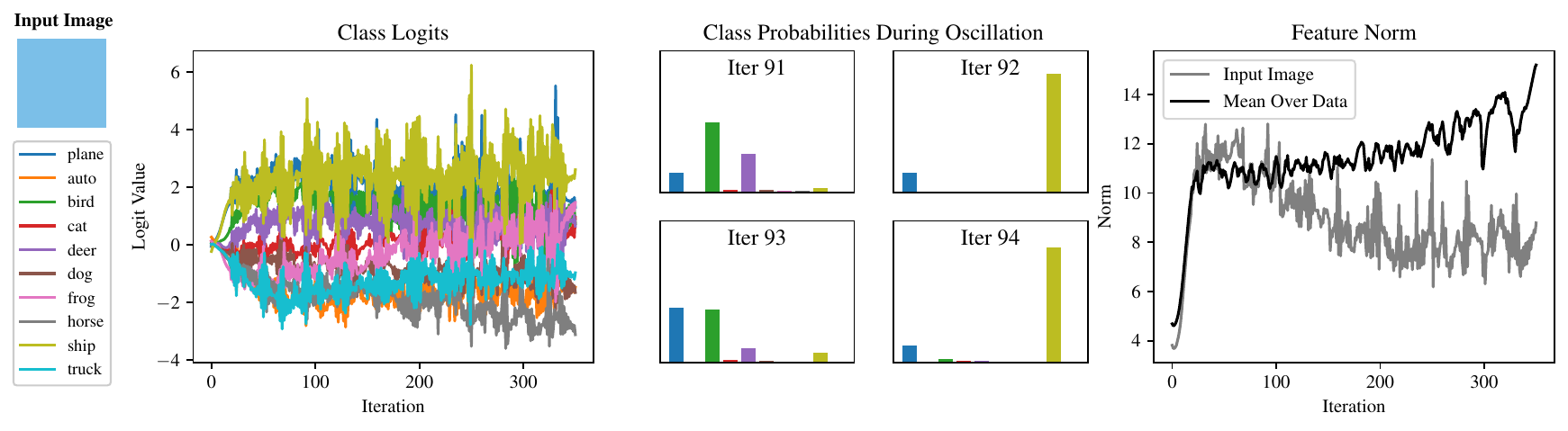}
    \caption{VGG-11-BN on a sky color block.}
\end{figure}

\begin{figure}[h!]
    \centering
    \includegraphics[width=\linewidth]{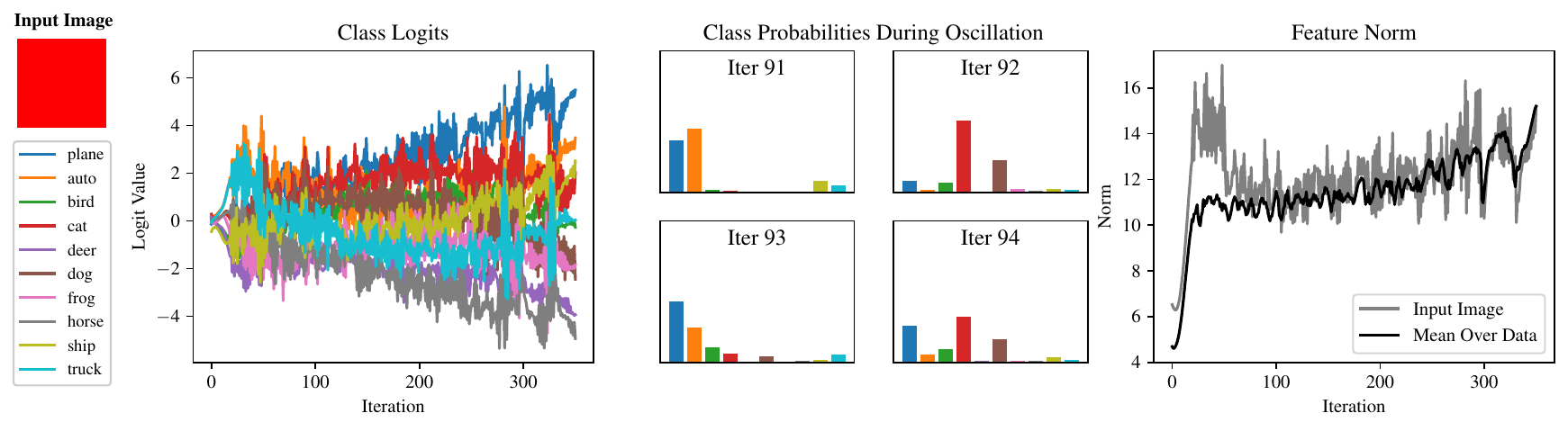}
    \caption{VGG-11-BN on a red color block.}
\end{figure}

\begin{figure}[b!]
    \centering
    \includegraphics[width=\linewidth]{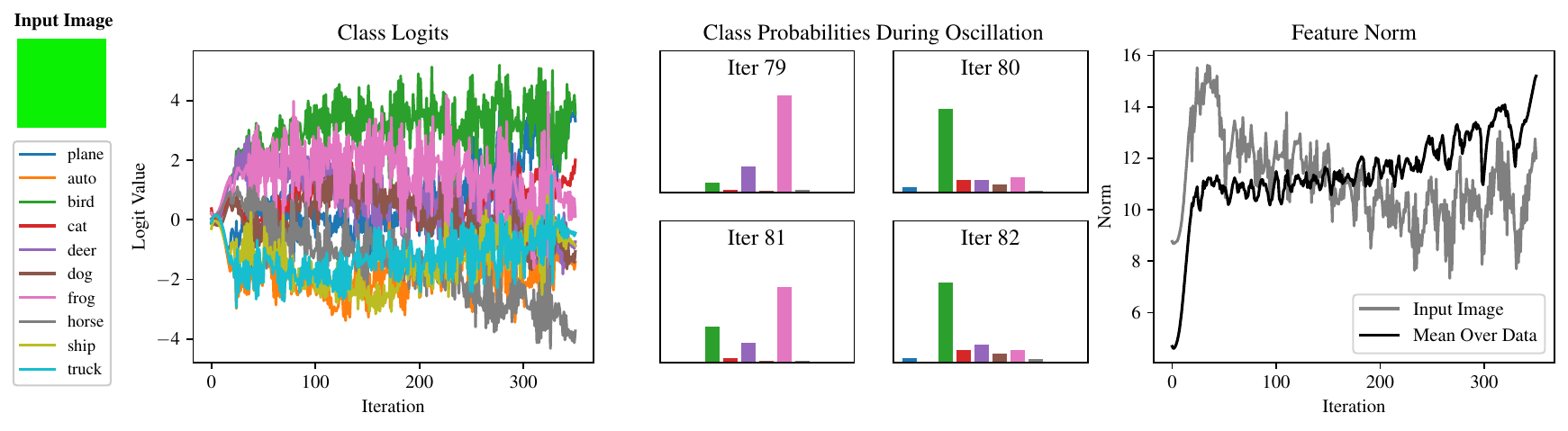}
    \caption{VGG-11-BN on a green color block. See above for two examples of relevant images in the dataset.}
\end{figure}

\begin{figure}[b!]
    \centering
    \begin{subfigure}{.1\linewidth}
        \includegraphics[width = \linewidth]{figures/png/Grass_Input.png}
    \end{subfigure}
    \begin{subfigure}{.89\linewidth}
        \includegraphics[width = \linewidth]{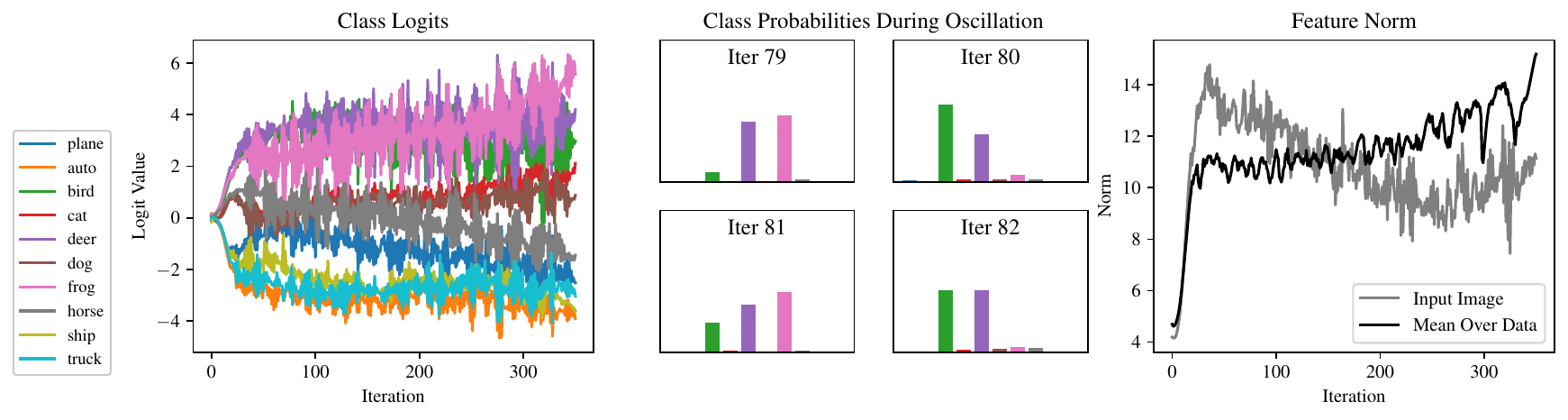}
    \end{subfigure}
    \caption{VGG-11-BN on an image with mostly grass texture.}
\end{figure}

\begin{figure}[h!]
    \centering
    \includegraphics[width=\linewidth]{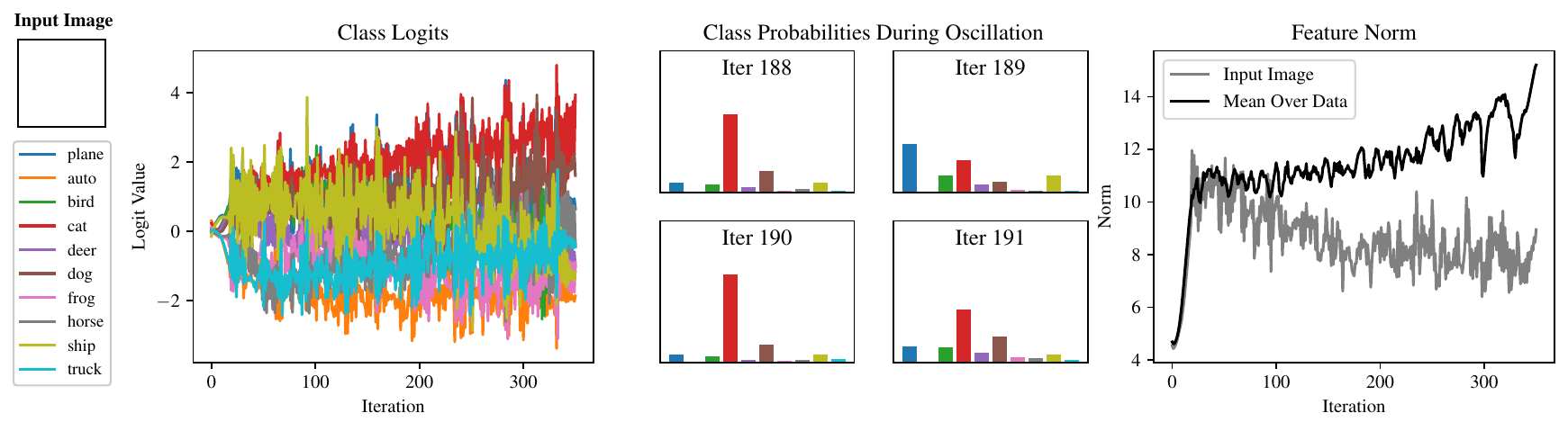}
    \caption{VGG-11-BN on a white color block.}
\end{figure}

\begin{figure}[h!]
    \centering
    \includegraphics[width=\linewidth]{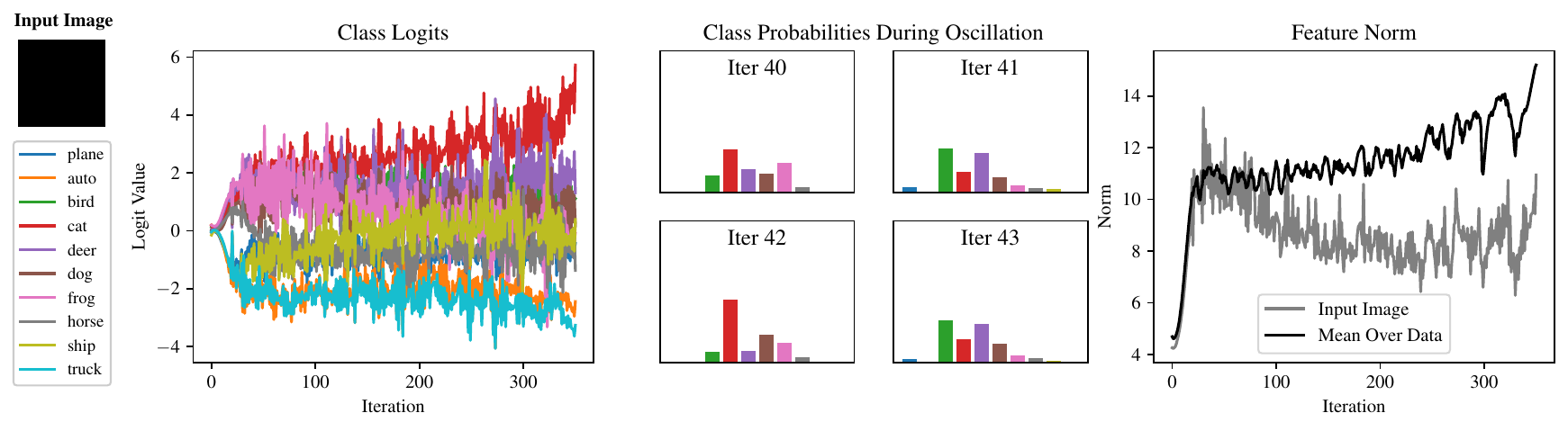}
    \caption{VGG-11-BN on a black color block.}
\end{figure}

\clearpage
\subsection{VGG-11-BN with Small Learning Rate to Approximate Gradient Flow}
\label{app:logit-track-small-lr}

Here we see that oscillation is a valuable regularizer, preventing the network from continuously upweighting opposing signals. As described in the main body, stepping too far in one direction causes an imbalanced gradient between the two opposing signals. Since the group which now has a larger loss is also the one which suffers from the use of the feature, the network is encouraged to downweight its influence. If we use a very small learning rate to approximate gradient flow, this regularization does not occur and the feature norms grow continuously. This leads to over-reliance on these features, suggesting that failing to downweight opposing signals is a likely cause of the poor generalization of networks trained with gradient flow.

The following plots depict a VGG-11-BN trained with learning rate $.0005$ to closely approximate gradient flow. We compare this to the feature norms of the same network trained with gradient descent with learning rate $0.1$, which closely matches gradient flow until it becomes unstable.

\begin{figure}[h!]
    \centering
    \includegraphics[width=\linewidth]{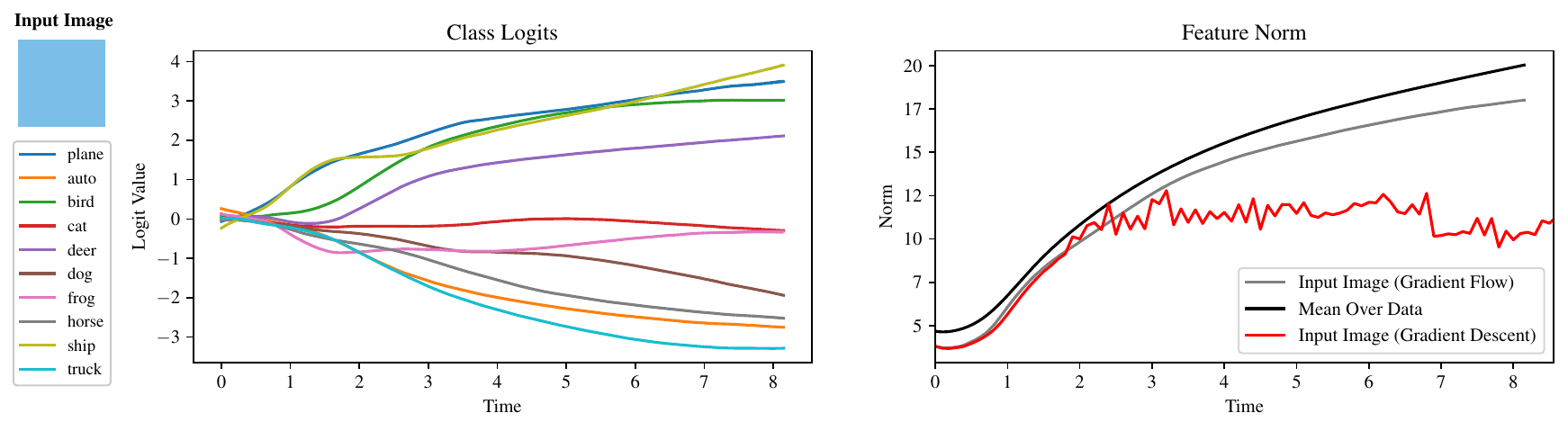}
    \caption{VGG-11-BN on a sky color block with learning rate 0.005 (approximating gradient flow) compared to 0.1.}
\end{figure}

\begin{figure}[h!]
    \centering
    \includegraphics[width=\linewidth]{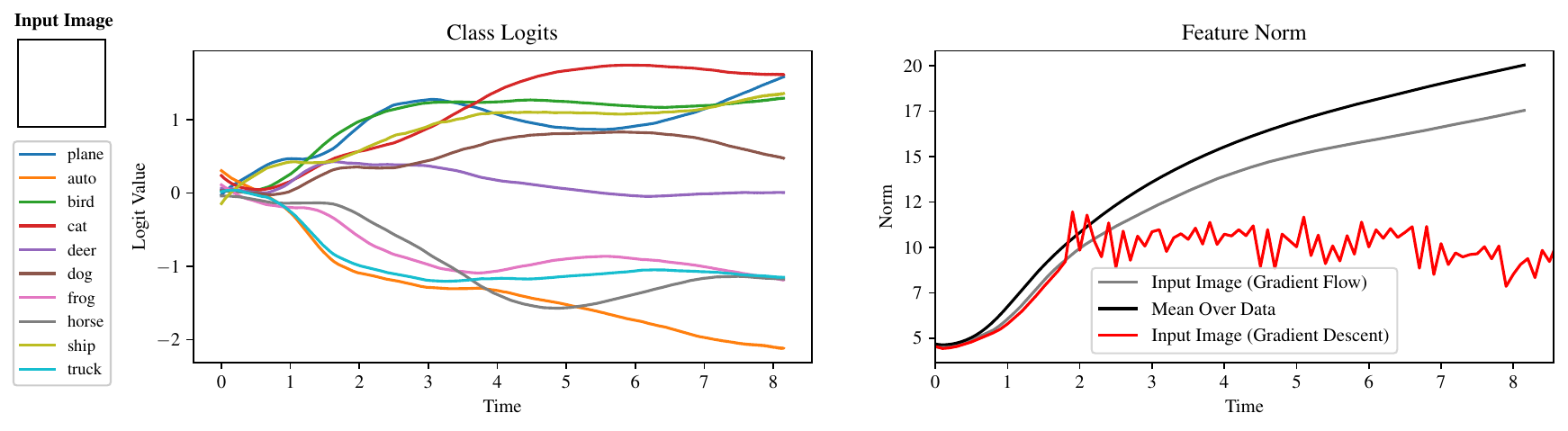}
    \caption{VGG-11-BN on a white color block with learning rate 0.005 (approximating gradient flow) compared to 0.1.}
\end{figure}

\begin{figure}[h!]
    \centering
    \includegraphics[width=\linewidth]{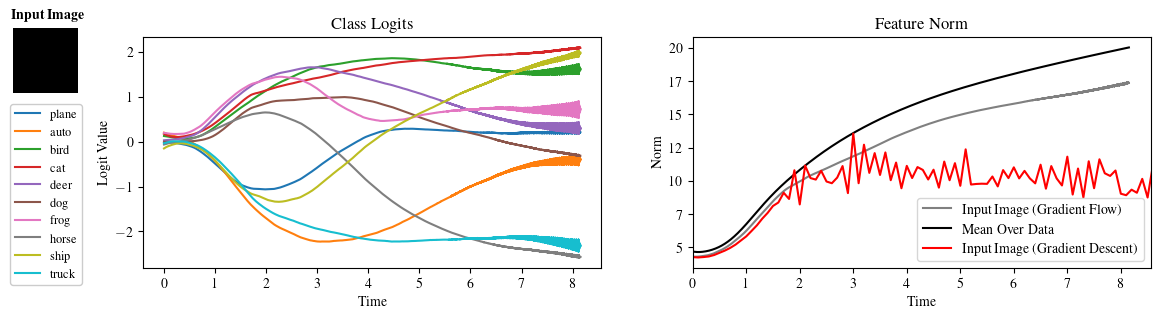}
    \caption{VGG-11-BN on a black color block with learning rate 0.005 (approximating gradient flow) compared to 0.1.}
\end{figure}

\begin{figure}[h!]
    \centering
    \includegraphics[width=\linewidth]{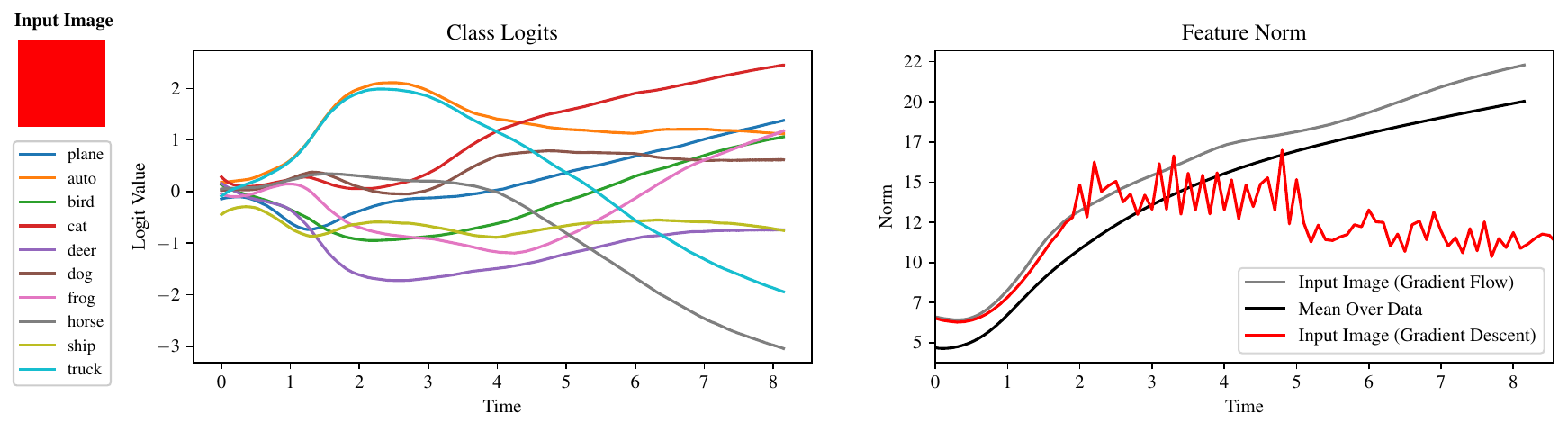}
    \caption{VGG-11-BN on a red color block with learning rate 0.005 (approximating gradient flow) compared to 0.1.}
\end{figure}

\begin{figure}[h!]
    \centering
    \includegraphics[width=\linewidth]{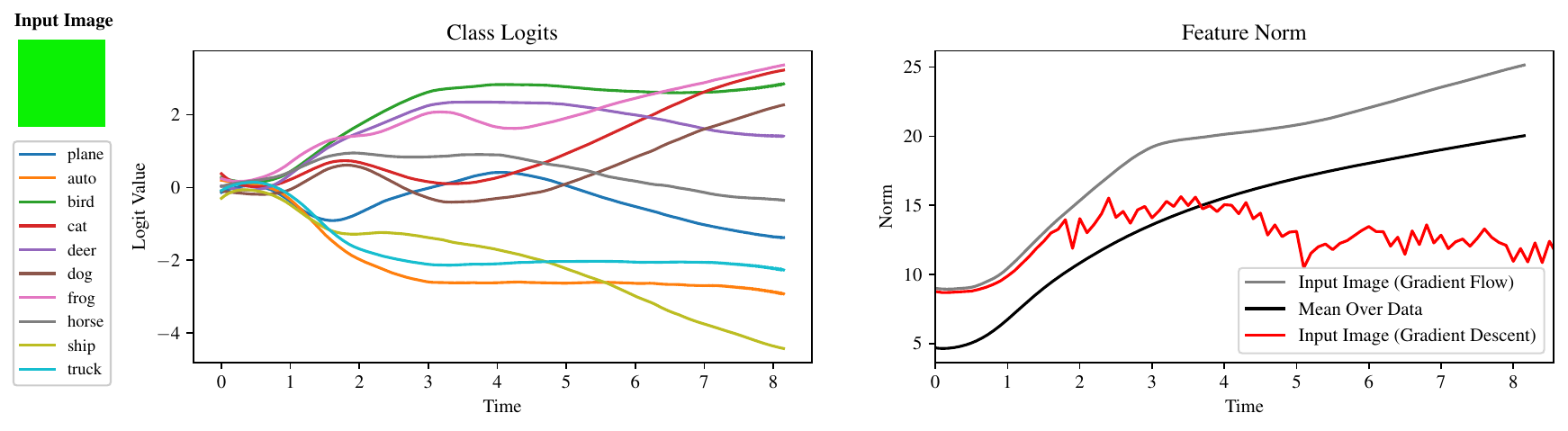}
    \caption{VGG-11-BN on a green color block with learning rate 0.0005 (approximating gradient flow) compared to 0.1.}
\end{figure}

\begin{figure}[ht!]
    \centering
    \begin{subfigure}{.1\linewidth}
        \includegraphics[width = \linewidth]{figures/png/Grass_Input.png}
    \end{subfigure}
    \begin{subfigure}{.89\linewidth}
        \includegraphics[width = \linewidth]{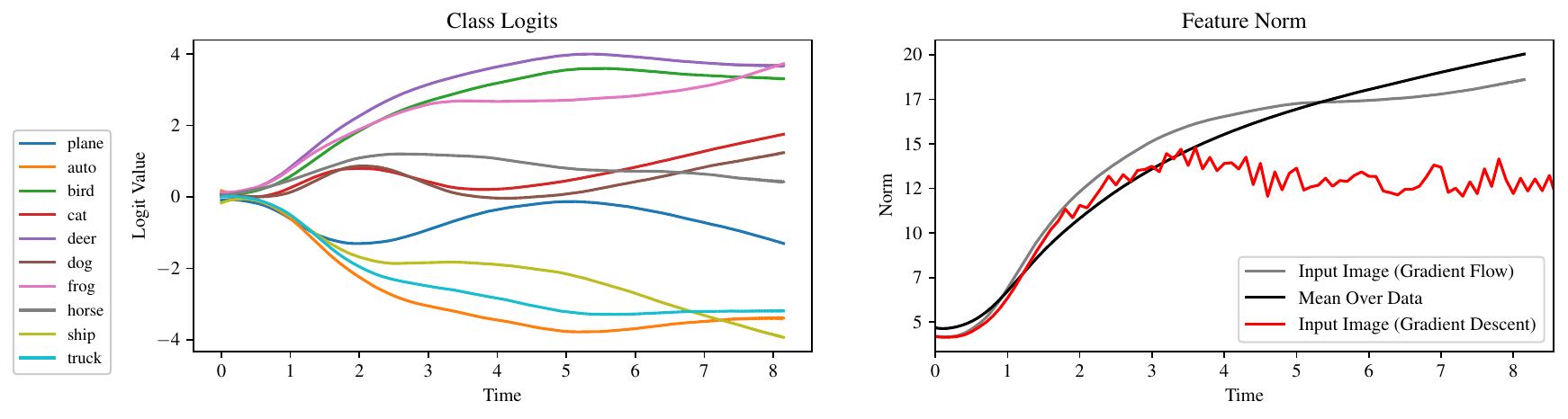}
    \end{subfigure}
    \caption{VGG-11-BN on an image with mostly grass texture with learning rate 0.0005 (approximating gradient flow) compared to 0.1.}
\end{figure}

\clearpage

\subsection{ResNet-18 Trained with Full-Batch Adam}
\label{app:logit-track-adam}

Finally, we plot the same figures for a ResNet-18 trained with full-batch Adam. We see that Adam consistently and quickly reduces the norm of these features, especially for more complex features such as texture, and that it also quickly reaches a point where oscillation ends. Note when comparing to plots above that the maximum iteration on the x-axis differs.

\begin{figure}[ht!]
    \centering
    \includegraphics[width=\linewidth]{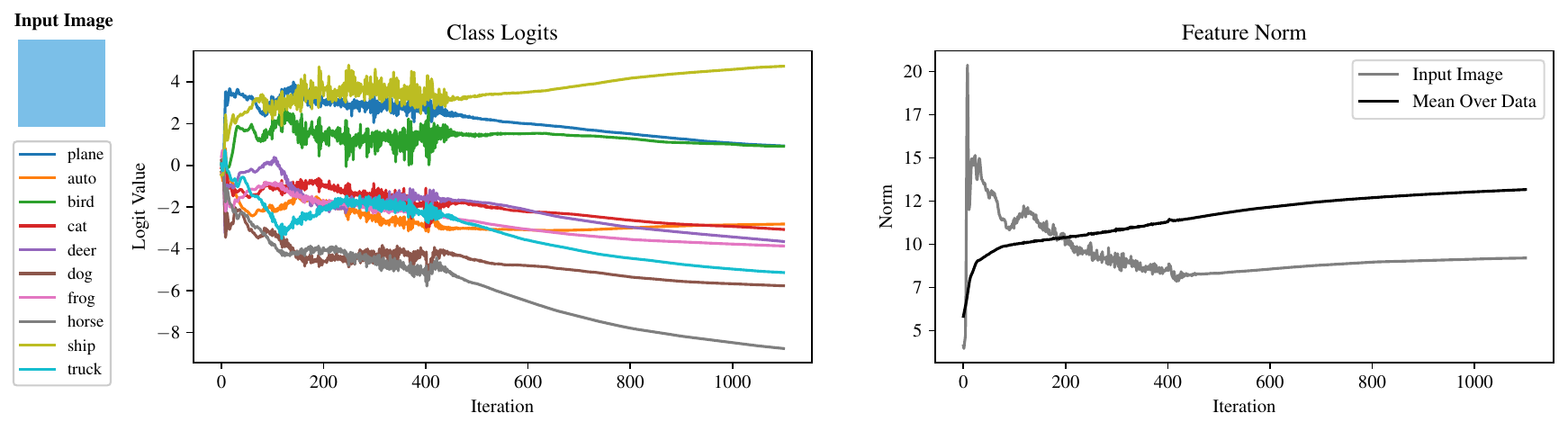}
    \caption{ResNet-18 on a sky color block trained with Adam.}
\end{figure}

\begin{figure}[ht!]
    \centering
    \includegraphics[width=\linewidth]{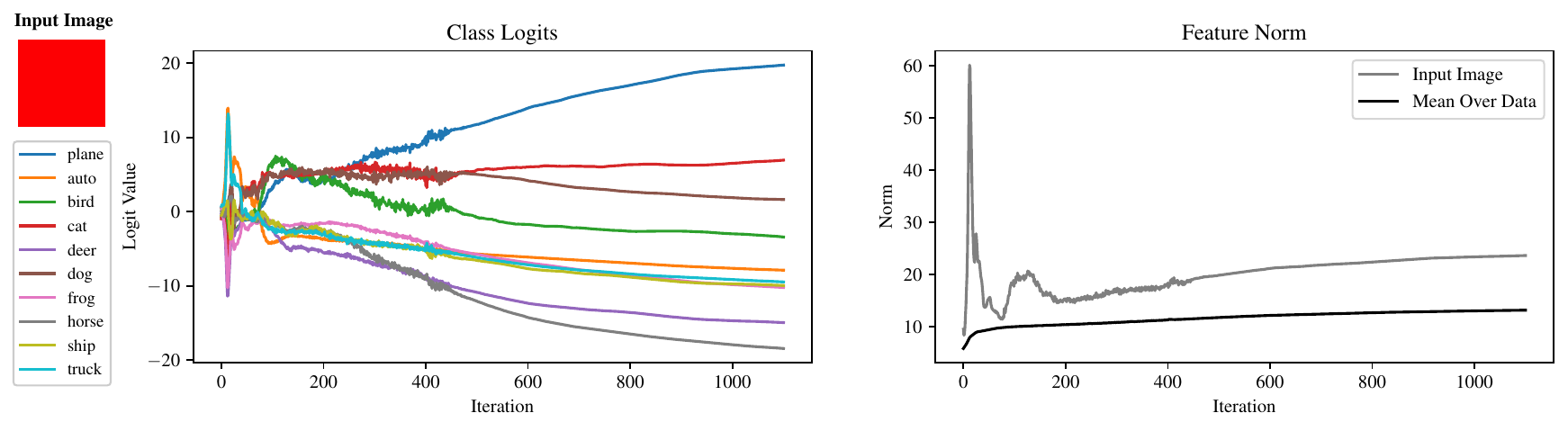}
    \caption{ResNet-18 on a red color block trained with Adam.}
\end{figure}

\begin{figure}[ht!]
    \centering
    \includegraphics[width=\linewidth]{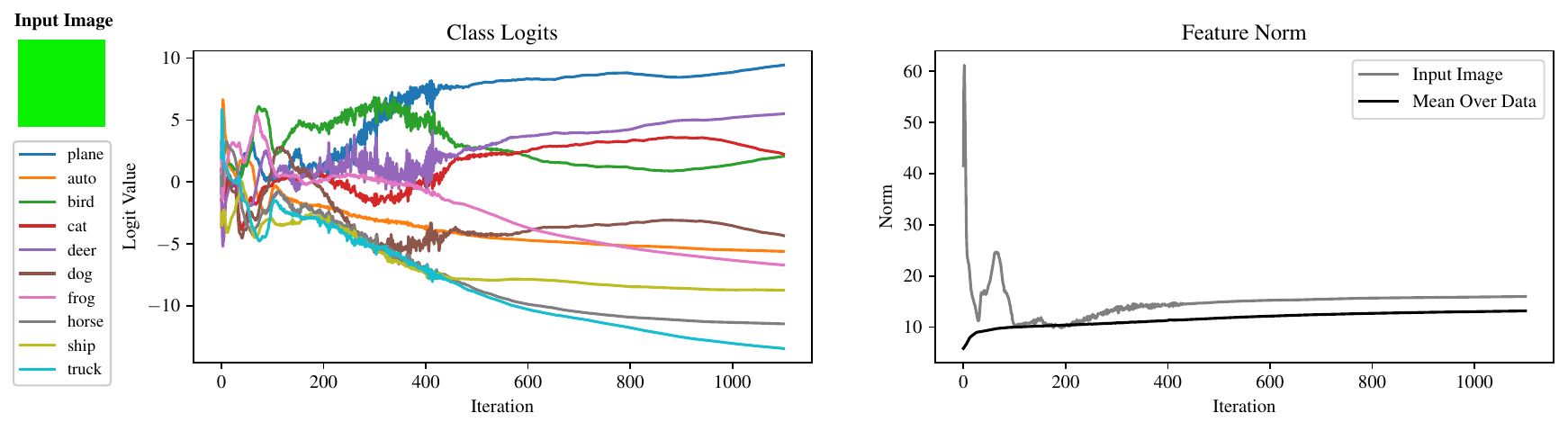}
    \caption{ResNet-18 on a green color block trained with Adam.}
\end{figure}

\begin{figure}[ht!]
    \centering
    \begin{subfigure}{.1\linewidth}
        \includegraphics[width = \linewidth]{figures/png/Grass_Input.png}
    \end{subfigure}
    \begin{subfigure}{.89\linewidth}
        \includegraphics[width = \linewidth]{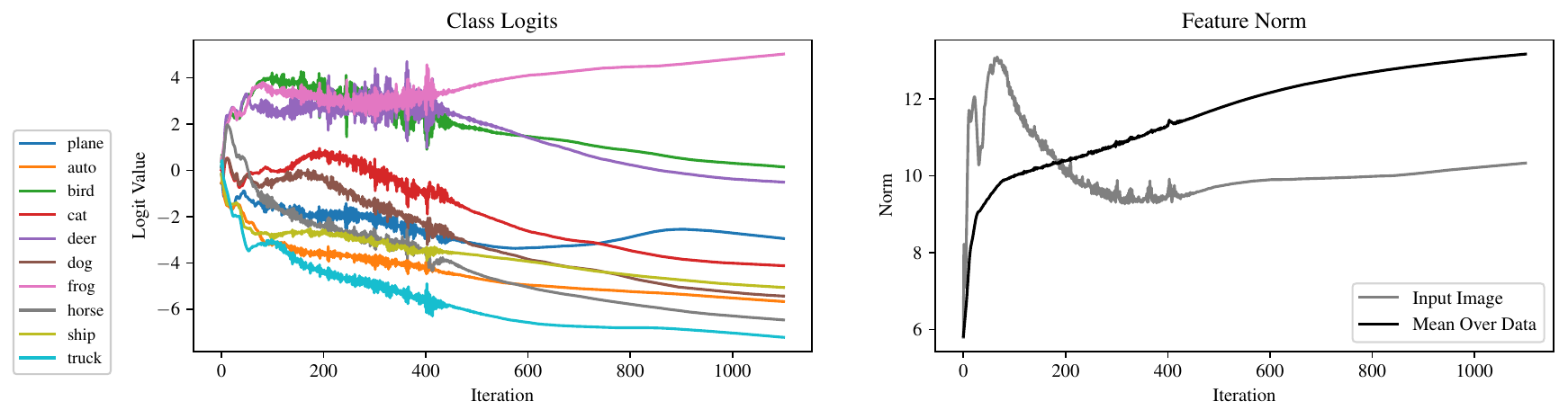}
    \end{subfigure}
    \caption{ResNet-18 on an image with mostly grass texture trained with Adam.}
\end{figure}

\clearpage
\section{Tracking the Amount of Curvature in each Parameter Layer}
\label{app:sharpness-location}

Here we plot the ``fraction of curvature'' of different architectures at various training steps. Recall the fraction of curvature is defined with respect to the top eigenvector of the loss Hessian. We partition this vector by network layer and evaluate each sub-vector's squared norm. This represents that layer's contribution to the overall curvature. To keep the plots readable, we omit layers whose fraction is never greater than $0.01$ at any training step (including the intermediate ones not plotted), though we always include the last linear layer. The total number of layers is 45, 38, 106, and 39 for the ResNet, VGG-11, ViT, and NanoGPT respectively.

\begin{figure}[ht!]
    \centering
    \includegraphics[width=\linewidth]{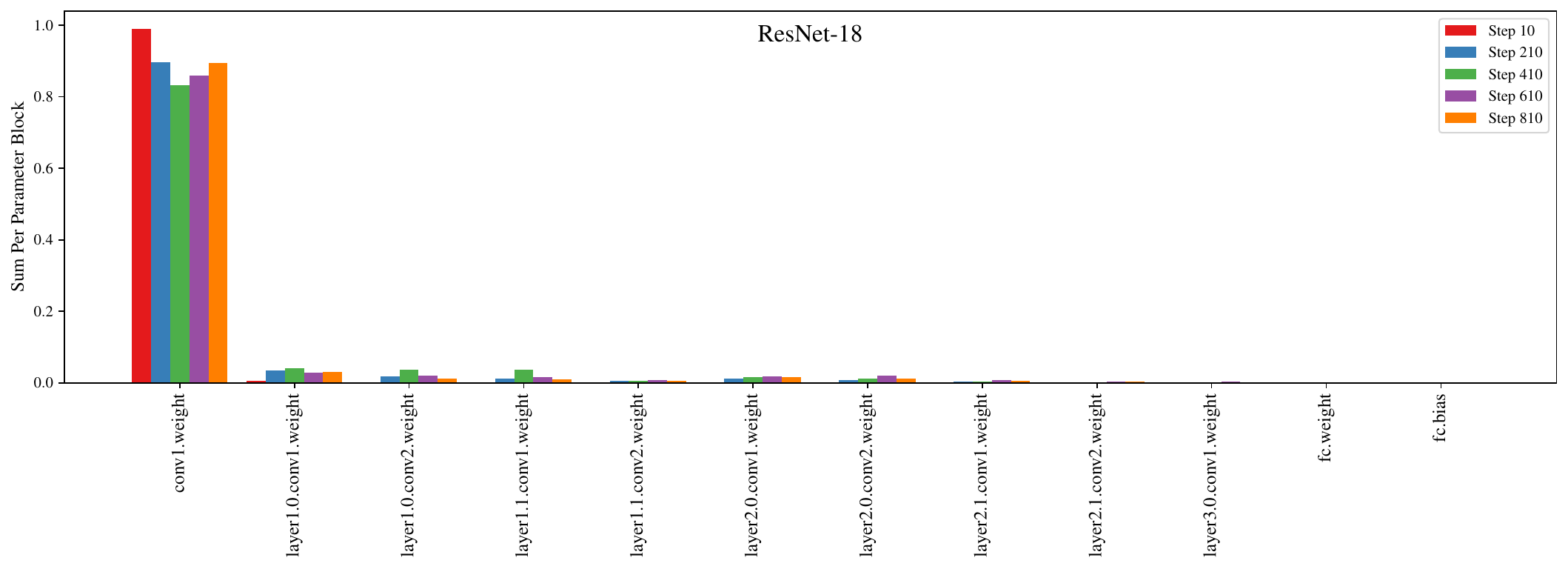}
    \caption{Sum of squared entries of the top eigenvector of the loss Hessian which lie in each parameter layer of a ResNet-18 throughout training.}
\end{figure}

\begin{figure}[ht!]
    \centering
    \includegraphics[width=\linewidth]{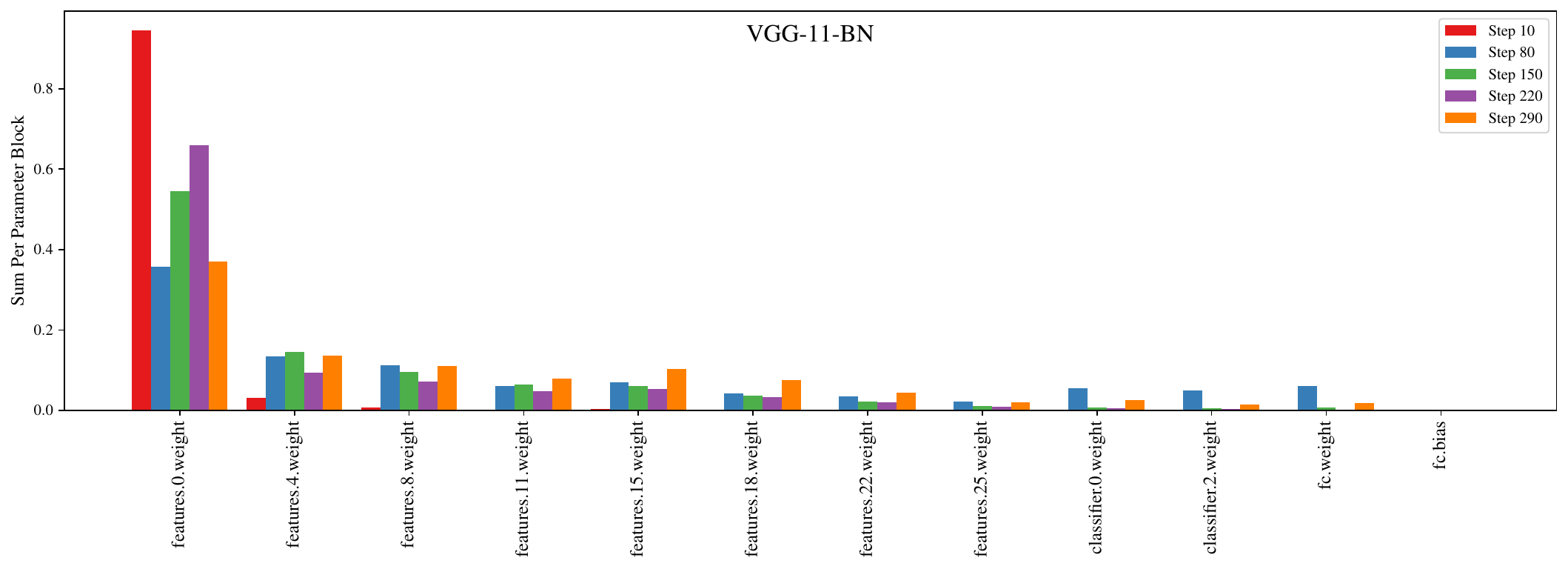}
    \caption{Sum of squared entries of the top eigenvector of the loss Hessian which lie in each parameter layer of a VGG-11-BN throughout training.}
\end{figure}

\begin{figure}[ht!]
    \centering
    \includegraphics[width=\linewidth]{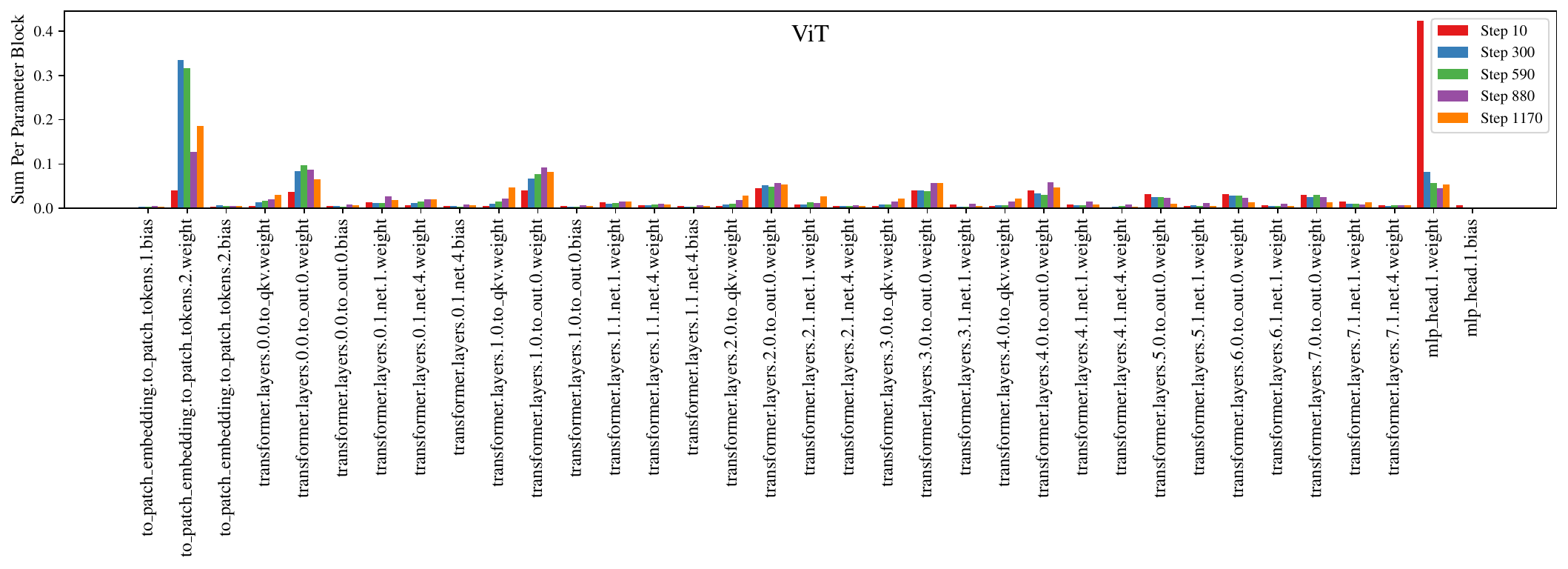}
    \caption{Sum of squared entries of the top eigenvector of the loss Hessian which lie in each parameter layer of a ViT throughout training.}
\end{figure}

\begin{figure}[ht!]
    \centering
    \includegraphics[width=\linewidth]{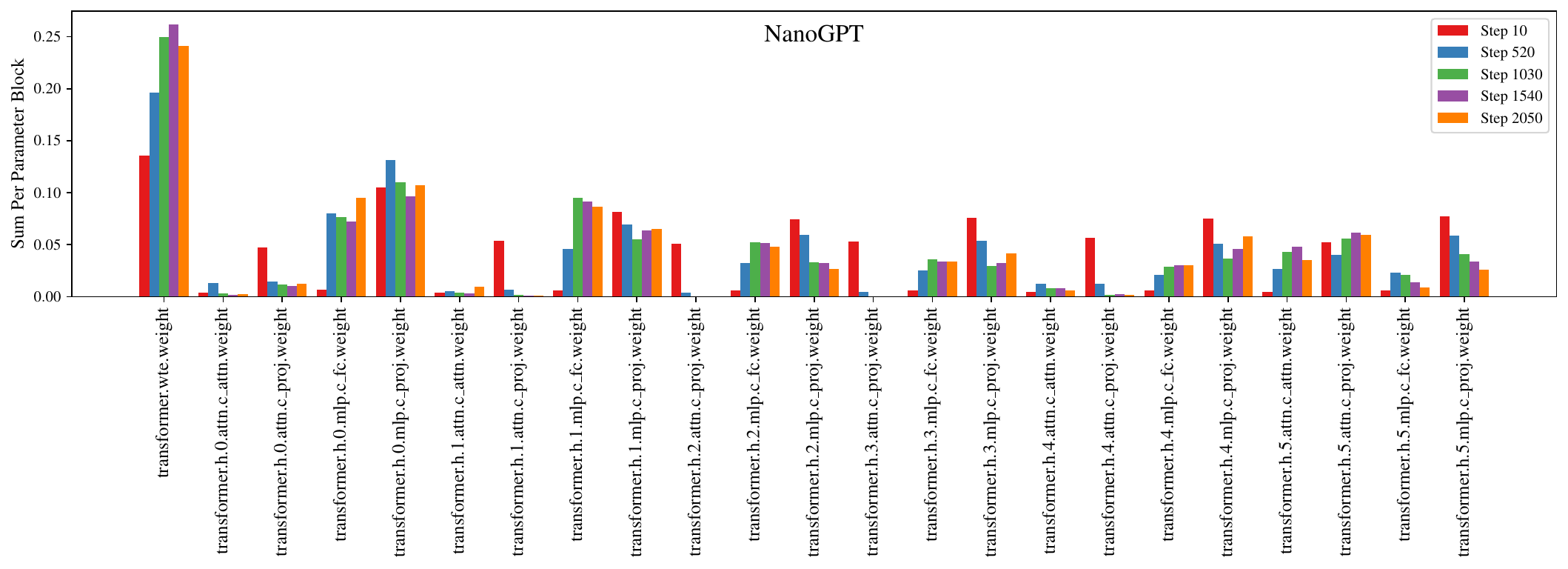}
    \caption{NanoGPT (6 layers, 6 head per layer, embedding dimension 384) trained on the default shakespeare character dataset in the NanoGPT repository. Due to difficulty calculating the true top eigenvector, we approximate it with the exponential moving average of the squared gradient.}
\end{figure}

\clearpage
\section{Additional Figures}
\label{sec:addl-figs}

\begin{figure}[ht]
    \centering
    \includegraphics[width=.5\linewidth]{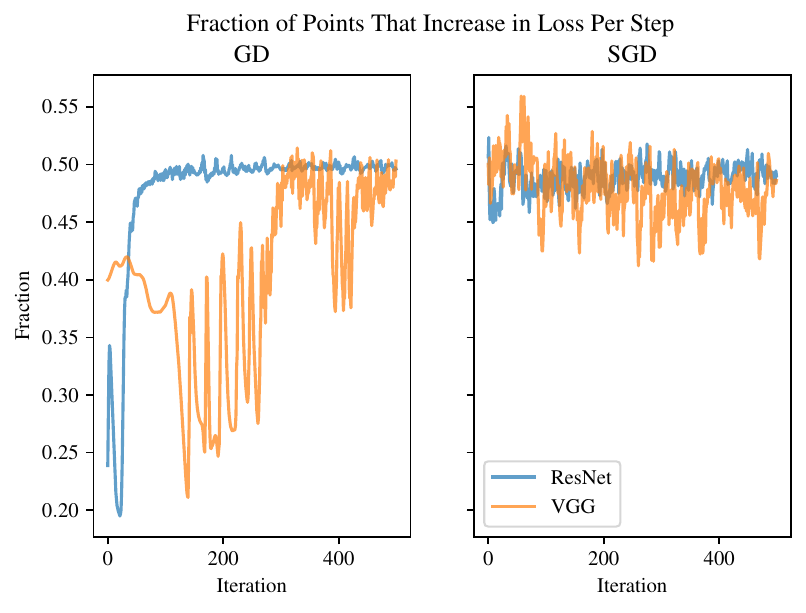}
    \caption{The fraction of overall training points which increase in loss on any given step. For both SGD and GD, it hovers around 0.5 (VGG without batchnorm takes a long time to reach the edge of stability). Though the outliers have much higher amplitude in their loss change, many more images contain \emph{some} small component of the features they exemplify (or are otherwise slightly affected by the weight oscillations), and so these points also oscillate in loss at a smaller scale.}
    \label{fig:frac-loss-increase}
\end{figure}

\begin{figure}[ht]
    \centering
    \includegraphics[width=.7\linewidth]{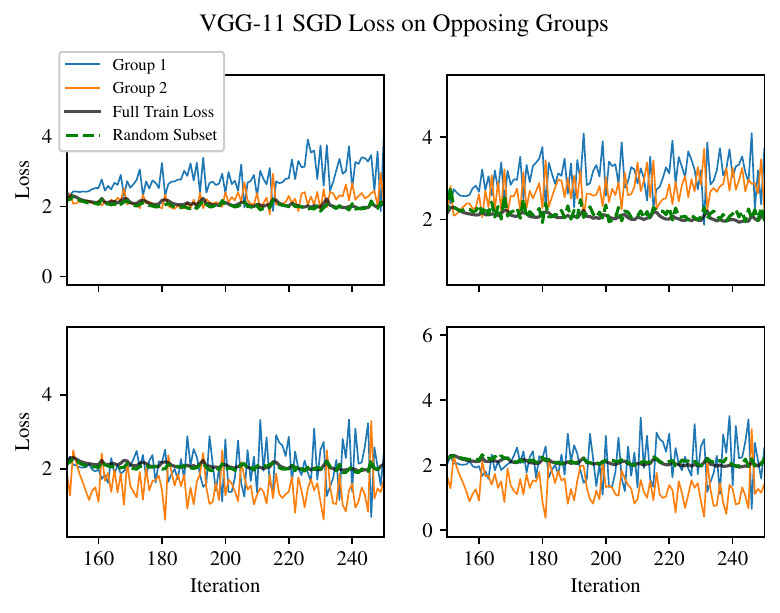}
    \caption{We reproduce \cref{fig:vgg-sgd} without batch normalization.}
    \label{fig:vggbn-sgd-groups}
\end{figure}

\begin{figure}[ht!]
    \centering
    \begin{subfigure}{.45\linewidth}
        \includegraphics[width = \linewidth]{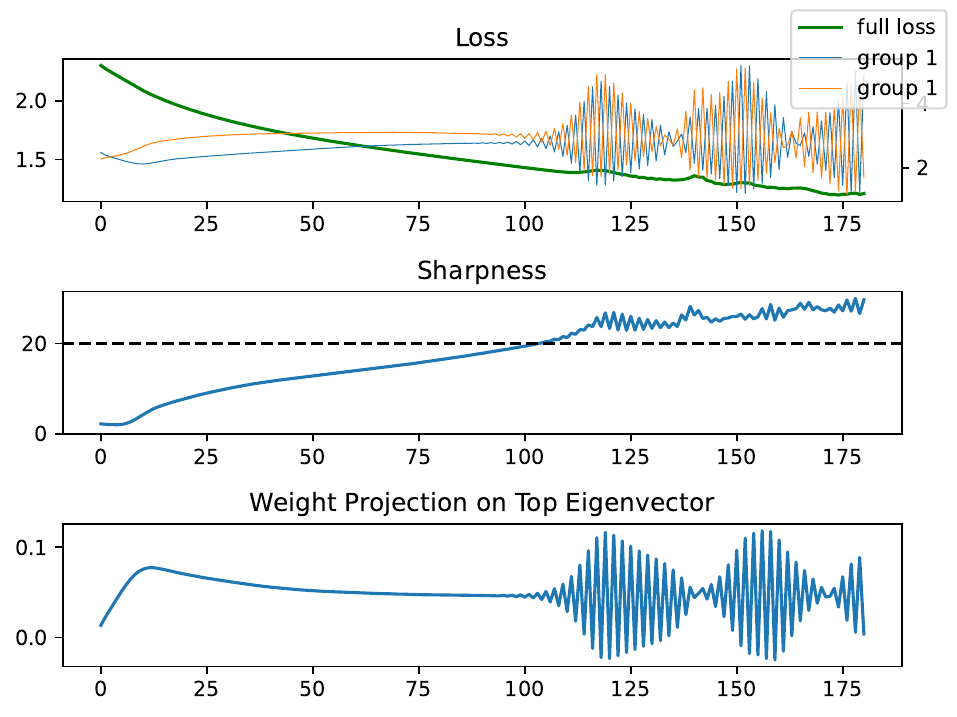}
        \caption{A 3-layer ReLU MLP trained on a 5k-subset of CIFAR-10.}
    \end{subfigure}
    \begin{subfigure}{.45\linewidth}
        \includegraphics[width = \linewidth]{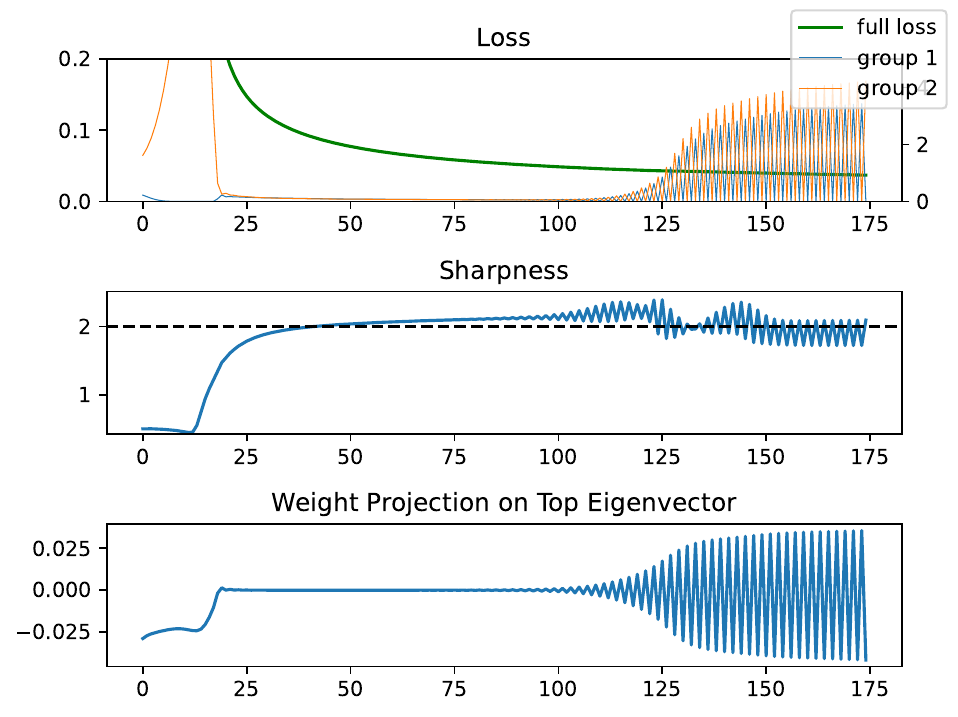}
        \caption{Our model: a 2-layer linear network trained on mostly Gaussian data with opposing signals.}
    \end{subfigure}
    \caption{We compare a small ReLU MLP on a subset of CIFAR-10 to our simple model of linear regression with a two-layer network.}
    \label{fig:verify-model}
\end{figure}

\clearpage
\section{Comparing our Variant of SGD to Adam}
\label{app:adam-vs-sgd}

\begin{algorithm}
\caption{SplitSGD}
\label{alg:splitsgd}
\begin{algorithmic}
\State \textbf{input:} Initial parameters $\theta_0$, SGD step size $\eta_1$, SignSGD step size $\eta_2$, momentum $\beta$, dampening $\tau$, threshold $r$.
\State \textbf{initialize:} $m_0 = \mathbf{0}$.
\For{$t \gets 1,\ldots,T$}
\State $g_t \gets \nabla_\theta L_t(\theta_{t-1})$ \Comment{Get stochastic gradient}
\State $m_t \gets \beta m_{t-1} + (1 - \tau) g_t$ \Comment{Update momentum with dampening}
\State $\hat m_t \gets m_t / (1 - \tau^t)$ \Comment{Debias}
\State $v_{\textrm{mask}} \gets \mathbf{1}\{|\hat m_t| \leq r\}$ \Comment{Split parameters by threshold}
\State $\theta_t \gets \theta_{t-1} - \eta_1 (\hat m_t \odot v_{\textrm{mask}}) - \eta_2  (\textrm{sign}(\hat m_t) \odot (1-v_{\textrm{mask}})) $ \Comment{unmasked SGD, masked SignSGD}
\EndFor
\end{algorithmic}
\end{algorithm}

As described in the main text, we find that simply including dampening and taking a fixed step size on gradients above a certain threshold results in performance matching that of Adam for the experiments we tried. We found that setting this threshold equal to the $q=.1$ quantile of the first gradient worked quite well---this was about \texttt{1e-4} for the ResNet-56/110 and \texttt{1e-6} for GPT-2.

Simply to have something to label it with, we name the method SplitSGD, because it performs SGD and SignSGD on different partitions of the parameters. The precise method is given above in \cref{alg:splitsgd}. We reiterate that we are not trying to suggest a new method---our goal is only to demonstrate the insight gained from knowledge of opposing signals' influence on NN optimization. For all plots, $\beta$ represents the momentum parameter and $\tau$ is dampening. Adam has a single parameter $\beta_1$ which represents both simultaneously, which we fix at 0.9, and we do the same for SplitSGD by setting $\beta = \tau = 0.9$. As in \cref{alg:splitsgd}, we let $\eta_1$ refer to the learning rate for standard SGD on the parameters with gradient below the magnitude threshold, and $\eta_2$ refers to the learning rate for the remainder which are optimized with SignSGD.

\subsection{SplitSGD on ResNet}
\begin{figure}[h!]
    \centering
    \includegraphics[width=\linewidth]{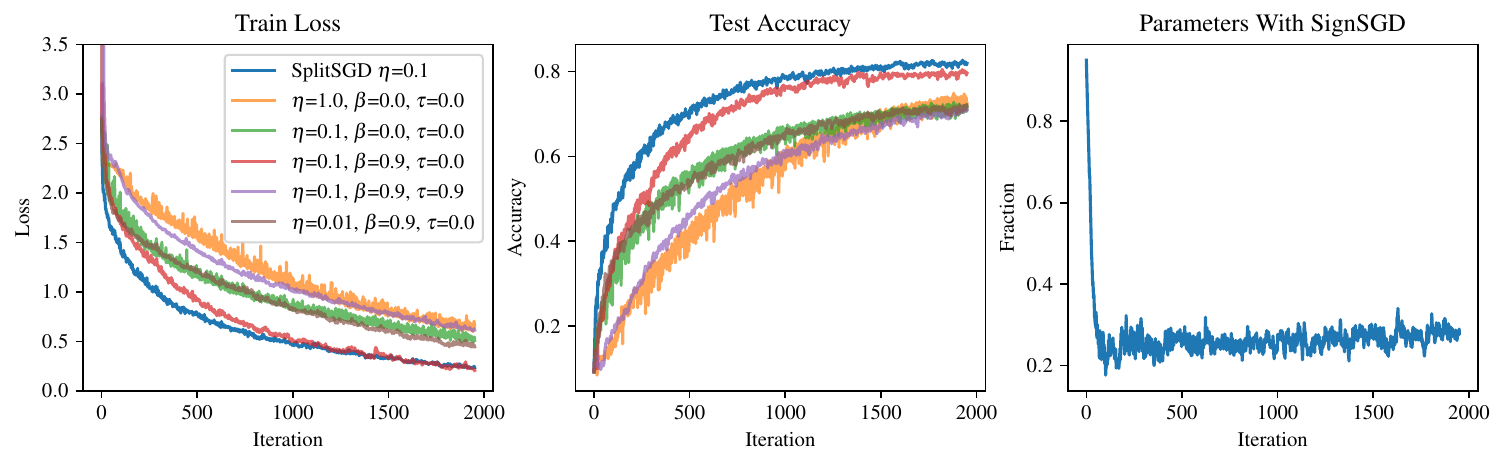}
    \caption{Standard SGD with varying learning rates and momentum/dampening parameters on a ResNet-56 on CIFAR-10, with one run of SplitSGD for comparison. Omitted SGD hyperparameter combinations performed much worse. Notice that SGD is extremely sensitive to hyperparameters. Rightmost plot is the fraction of parameters with fixed step size by SplitSGD.}
    \label{fig:sgdvssplit}
\end{figure}
We begin with a comparison on ResNets trained on CIFAR-10. \cref{fig:sgdvssplit} compares SplitSGD to standard versions of SGD with varying momentum and dampening on a ResNet-56. As expected, SGD is extremely sensitive to hyperparameters, particularly the learning rate, and even the best choice in a grid search underperforms SplitSGD. Furthermore, the rightmost plot depicts the fraction of parameters for which SplitSGD takes a fixed-size signed step. This means that after the first few training steps, 70-80\% of the parameters are being optimized simply with standard SGD (with $\beta = \tau = 0.9$).

\begin{figure}[h!]
    \centering
    \includegraphics[width=\linewidth]{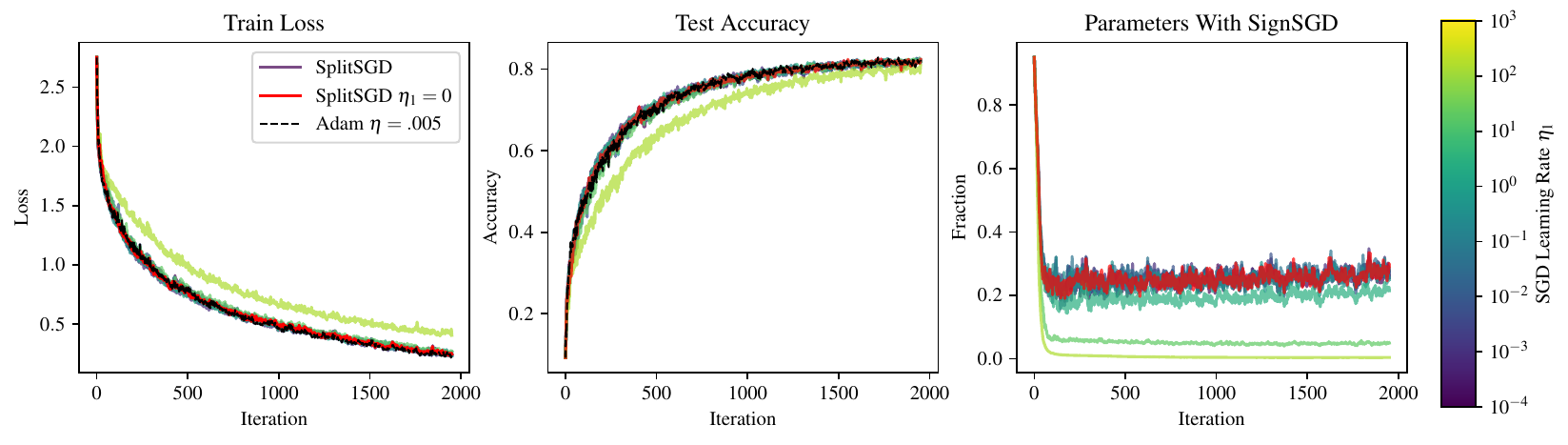}
    \caption{SplitSGD with varying SGD learning rates $\eta_1$ versus Adam on a ResNet-56 on CIFAR-10. The SignSGD learning rate is fixed at $\eta_2 = .001$; Adam uses $\eta = .005$, which was found to be the best performing choice via oracle selection grid search. The rightmost plot is the fraction of parameters with fixed step size by SplitSGD---that is, 1 minus this value is the fraction of parameters taking a regular gradient step with step size as given in the legend. This learning rate ranges over several orders of magnitude, is used for \textasciitilde 70-80\% of parameters, and can even be set to 0, with no discernible difference in performance.}
    \label{fig:splitvsadam_varysgdlr}
\end{figure}
Next, \cref{fig:splitvsadam_varysgdlr} plots SplitSGD with varying $\eta_1$ and $\eta_2$ fixed at .001. This is compared to Adam with learning rate .005, which was chosen via oracle grid search. Even though the SGD learning rate $\eta_1$ ranges over \emph{seven orders of magnitude} and is used for \textasciitilde 70-80\% of parameters, we see  no real difference in the train loss or test accuracy of SplitSGD. In fact, we find that we can even eliminate it completely! This suggests that for most of parameters and most of training, it is only a small fraction of parameters in the entire network which are influencing the overall performance. We posit a deeper connection here to the ``hidden'' progress described in grokking \citep{barak2022hidden, nanda2023progress}---if the correct subnetwork and its influence on the output grows slowly during training, that behavior will not be noticeable until the dominating signals are first downweighted.

\cref{fig:adamvssplit1,fig:adamvssplit2} depict the train loss and test accuracy of Adam and SplitSGD for varying learning rates (the standard SGD learning rate $\eta_1$ is fixed at 0.1). We see that SplitSGD is at least as robust as Adam to learning rate choice, if not more. The results also suggest that SplitSGD benefits from a slightly smaller learning rate than Adam, which we attribute to the fact that it will \emph{always} take step sizes of that fixed size, whereas the learning rate for Adam represents an upper bound on the step size for each parameter.

\begin{figure}[h!]
    \centering
    \includegraphics[width=\linewidth]{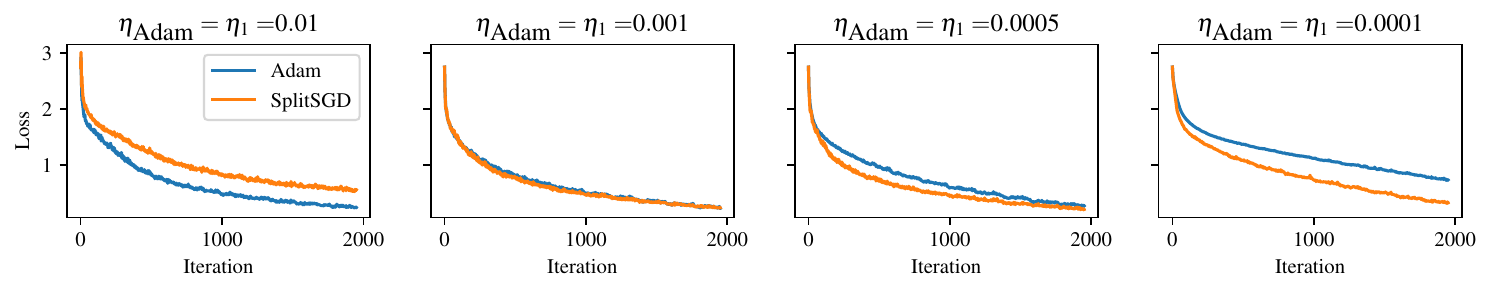}
    \caption{Train loss of Adam and SplitSGD for varying learning rates. The regular SGD step size for SplitSGD is fixed at 0.1. SplitSGD seems at least as robust to choice of learning rate as Adam, and it appears to benefit from a slightly smaller learning rate because it cannot adjust per-parameter.}
    \label{fig:adamvssplit1}
\end{figure}
\begin{figure}[h!]
    \centering
    \includegraphics[width=\linewidth]{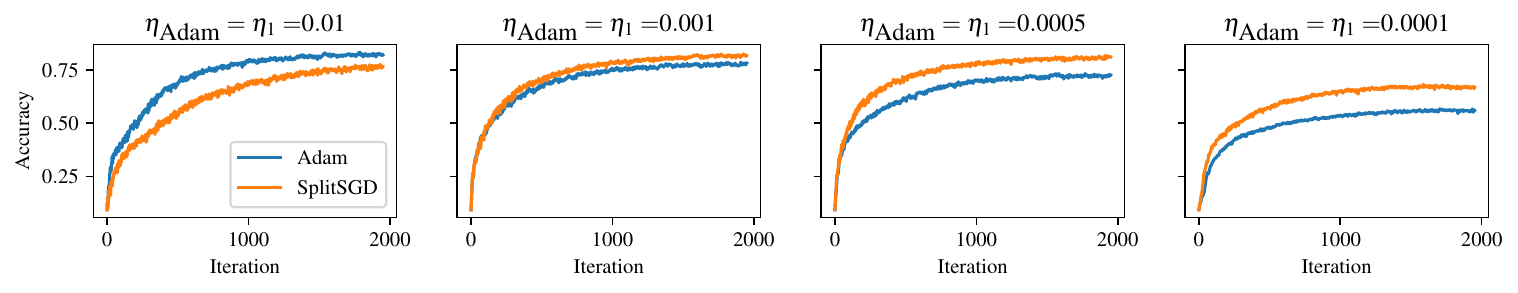}
    \caption{Test accuracy of Adam and SplitSGD for varying learning rates. The regular SGD step size for SplitSGD is fixed at 0.1. SplitSGD seems at least as robust to choice of learning rate as Adam, and it appears to benefit from a slightly smaller learning rate because it cannot adjust per-parameter.}
    \label{fig:adamvssplit2}
\end{figure}

We repeat these experiments with a ResNet-110, with similar findings. \cref{fig:resnet110-adamvssgd} compares the train loss and test accuracy of SGD with $\beta=0.9, \tau=0$ to Adam, and again the sensitivity of this optimizer to learning rate is clear. \cref{fig:resnet110-ablatedampen} compares Adam to SplitSGD (both with fixed-step learning rate .0003) but ablates the use of dampening: we find that the fixed-size signed steps appear to be more important for early in training, while dampening is helpful for maintaining performance later. It is not immediately clear what causes this bifurcation, nor if it will necessarily transfer to attention models.

Finally, \cref{fig:resnet110-adamvssplit} compares Adam to the full version of SplitSGD; we see essentially the same performance, and furthermore SplitSGD maintains its robustness to the choice of standard SGD learning rate.

\begin{figure}[b!]
    \centering
    \begin{subfigure}{.48\linewidth}
        \includegraphics[width = \linewidth]{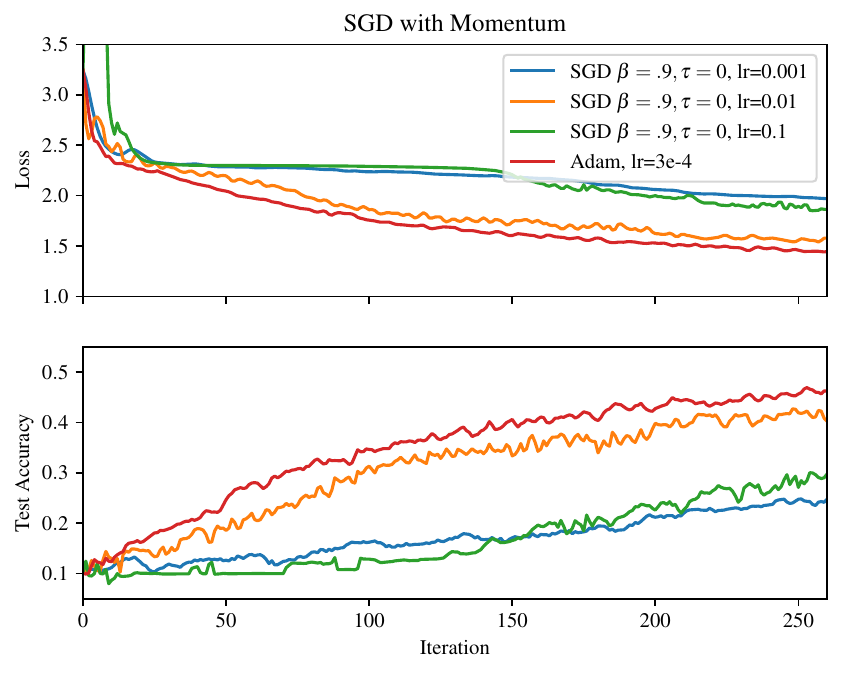}
        \caption{Adam versus standard SGD with Momentum. SGD remains extremely sensitive to choice of learning rate.}
        \label{fig:resnet110-adamvssgd}
    \end{subfigure}
    \begin{subfigure}{.48\linewidth}
        \includegraphics[width = \linewidth]{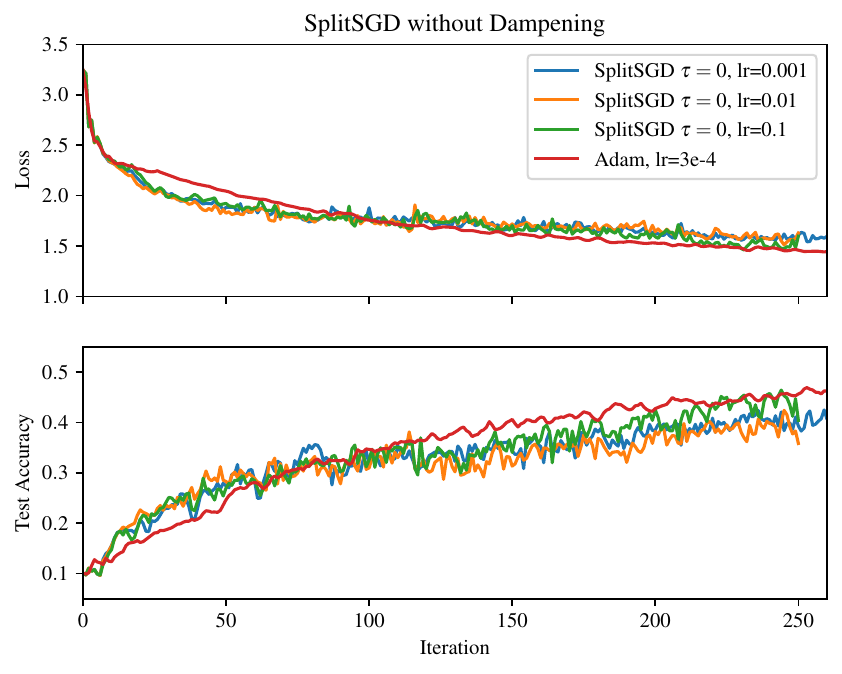}
        \caption{Adam vs. SplitSGD with $\tau=0$. Fixed-size learning rate for both is .0003.}
        \label{fig:resnet110-ablatedampen}
    \end{subfigure}
\end{figure}

\begin{figure}[h!]
    \centering
    \begin{subfigure}{.45\linewidth}
        \includegraphics[width = \linewidth]{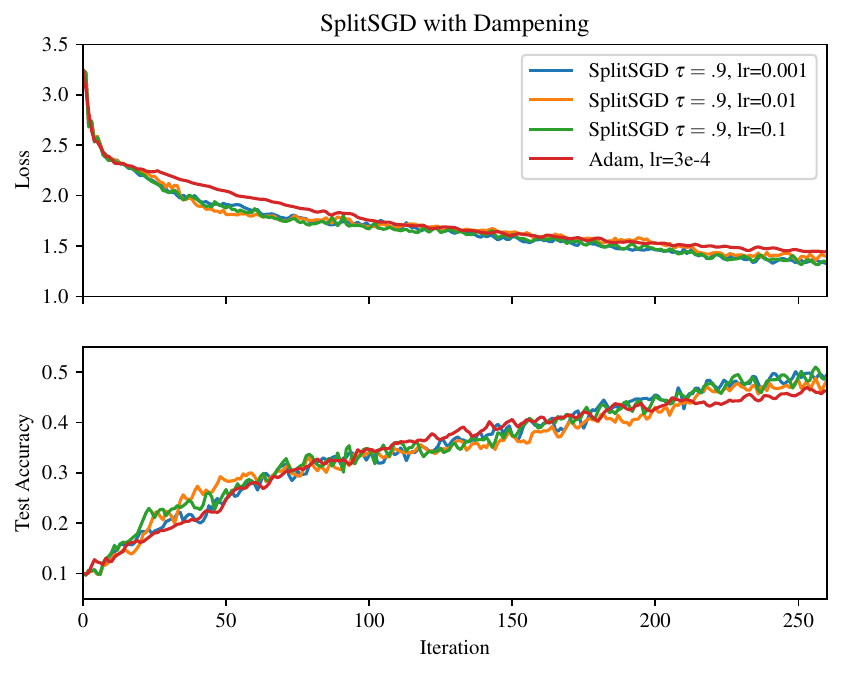}
        \caption{Adam vs. SplitSGD with $\tau=0.9$.  Fixed-size learning rate for both is .0003.}
        \label{fig:resnet110-adamvssplit}
    \end{subfigure}
    \begin{subfigure}{.48\linewidth}
        \includegraphics[width = \linewidth]{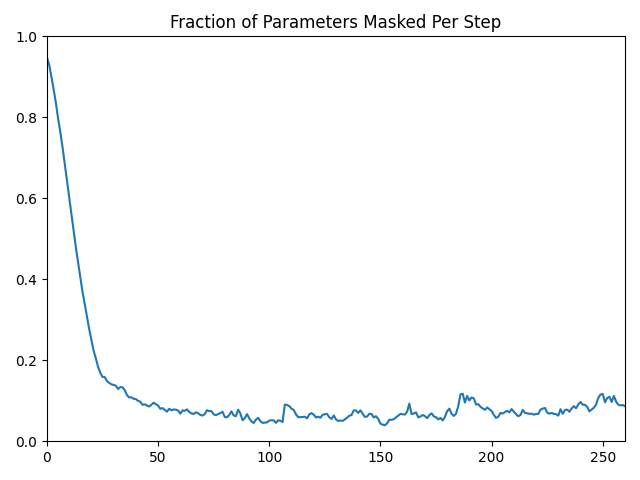}
        \caption{The fraction of parameters for which a fixed-size signed step was taken for each gradient step.}
        \label{fig:resnet110-frac-masked}
    \end{subfigure}
\end{figure}

\subsection{SplitSGD on GPT-2}
\label{appsec:split-gpt2}

For the transformer, we use the public nanoGPT repository which trains GPT-2 on the OpenWebText dataset. As a full training run would be too expensive, we compare only for the early stage of optimization. All hyperparameters are the defaults from that repository, with the SGD learning rate $\eta_1$ set equal to the other learning rate $\eta_2$. We observe that not only do the two methods track each other closely in training loss, it appears that they experience \emph{exactly} the same oscillations. Though we do not track the parameters themselves, this suggests that these two methods follow very similar optimization trajectories as well, which we believe is an intriguing possibility worth further study.

\begin{figure}[ht!]
    \centering
    \begin{subfigure}{.48\linewidth}
        \includegraphics[width = \linewidth]{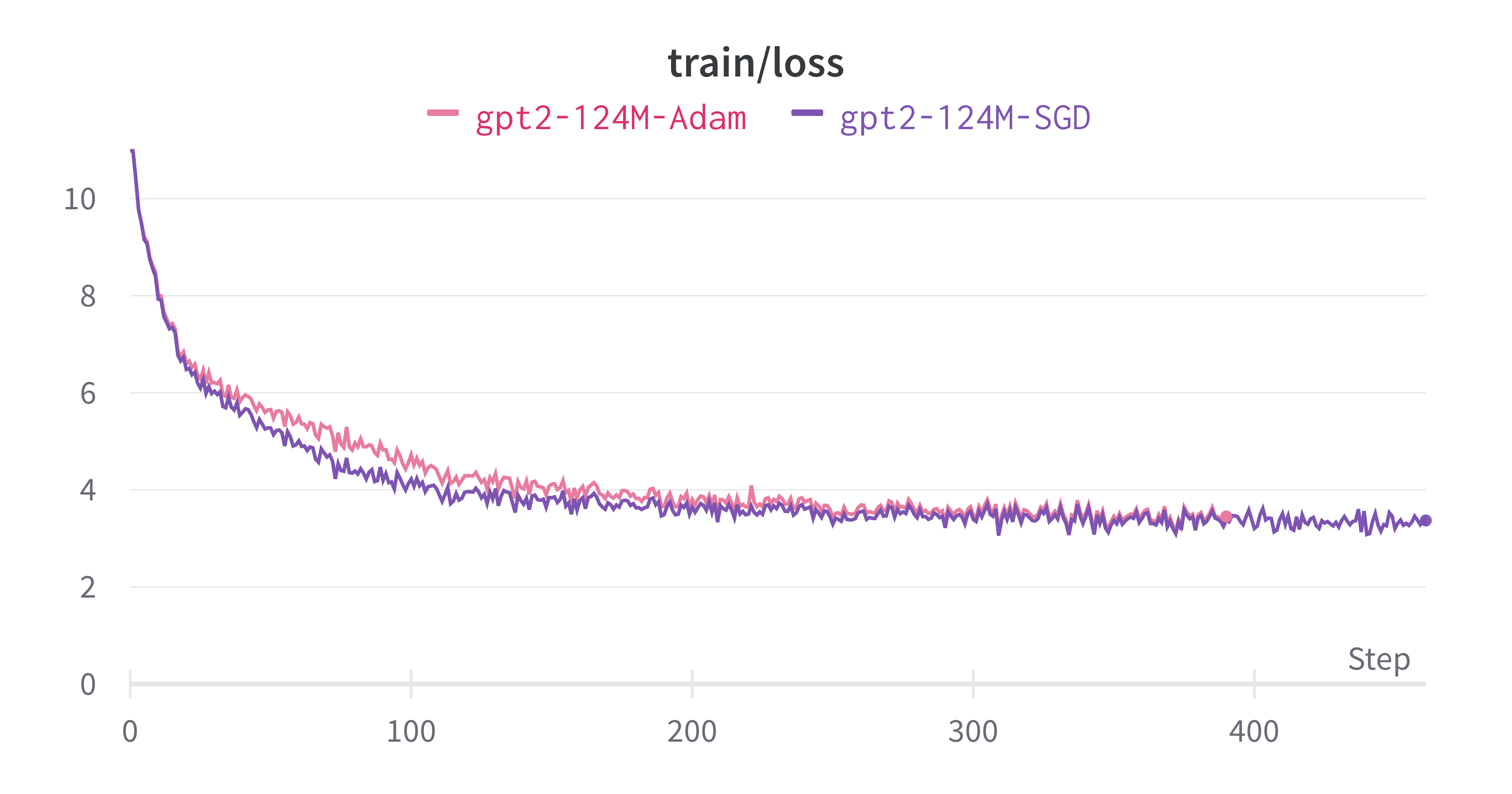}
    \end{subfigure}
    \begin{subfigure}{.48\linewidth}
        \includegraphics[width = \linewidth]{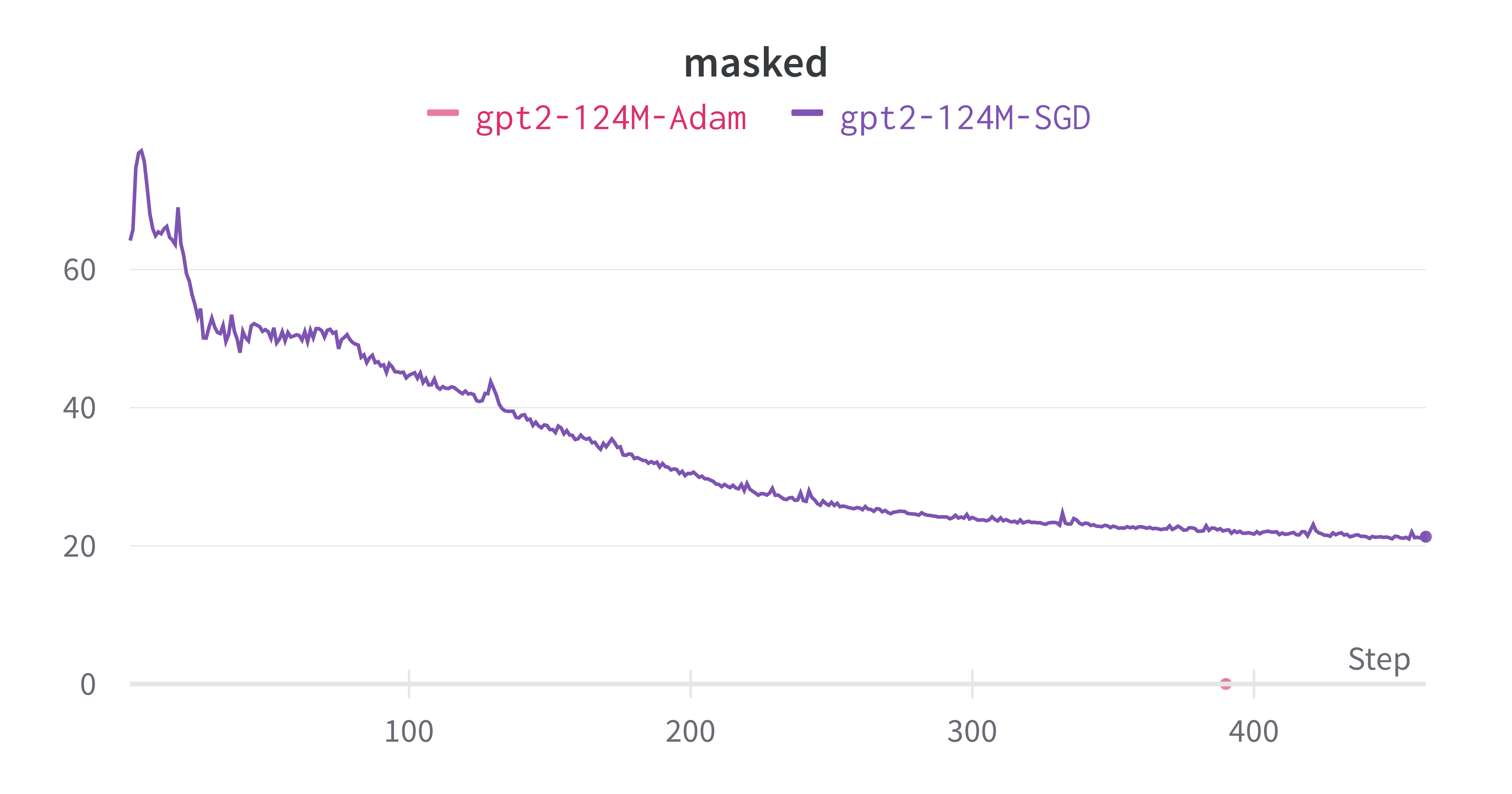}
    \end{subfigure}
    \caption{Adam versus SplitSGD on the initial stage of training GPT-2 on the OpenWebText dataset, and the fraction of parameters with a fixed-size signed step. All hyperparameters are the defaults from the nanoGPT repository. Observe that not only is their performance similar, they appear to have \emph{exactly} the same loss oscillations.}
    \label{fig:gpt-masksgd}
\end{figure}

\clearpage

\newpage
\section{Proofs of Theoretical Results}
\label{app:proofs}

Before we begin the analysis, we must identify the quantities of interest during gradient flow and the system of equations that determines how they evolve.

We start by writing out the loss:
\begin{align}
    2 L(\theta) &= \E[(c(b^\top x + b_o^\top x_o) - (\beta^\top x + d_2^{-1/2} \mathbf{1}^\top |x_o|))^2] \\
    &= \E[((cb - \beta)^\top x)^2] + \E[((cb_o - d_2^{-1/2} \sign(x_o) \mathbf{1})^\top x_o)^2] \\
    &= \|cb - \beta\|^2 + \frac{p}{2} \left( \left(\sqrt{\frac{\alpha}{p}} (cb_o - 1)\right)^2 + \left(\sqrt{\frac{\alpha}{p}} (cb_o + 1)\right)^2 \right) \\
    &= \|cb - \beta\|^2 + \alpha (c^2 \|b_o\|^2 + 1).
\end{align}

This provides the gradients
\begin{align}
    \nabla_b L &= c (cb - \beta), \\
    \nabla_{b_o} L &= \alpha c^2 b_o, \\
    \nabla_c L &= b^\top (cb - \beta) + \alpha \|b_o\|^2 c.
\end{align}
We will also make use of the Hessian to identify its top eigenvalue; it is given by
\begin{align}
    \nabla^2_\theta L(\theta) &= \begin{bmatrix} 
    c^2 I_{d_1} & \mathbf{0}_{d_1\times d_2} & 2cb \\
    \mathbf{0}_{d_2\times d_1} & \alpha c^2 I_{d_2} & 2c\alpha b_o \\
    2cb^\top & 2c\alpha b_o^\top & \|b\|^2 + \alpha \|b_o\|^2
    \end{bmatrix}.
\end{align}
The maximum eigenvalue $\lambda_{\max}$ at initialization is upper bounded by the maximum row sum of this matrix, and thus $\lambda_{\max} \leq 3 \frac{d_1 +\alpha d_2}{d_1+d_2} < 3\alpha$. Clearly, we also have $\lambda_{\max} \geq \alpha$. 

We observe that tracking the precise vectors $b, b_o$ are not necessary to uncover the dynamics when optimizing this loss. First, let us write $b := \epsilon \frac{\beta}{\|\beta\|} + \delta v$, where $v$ is the direction of the rejection of $b$ from $\beta$ (i.e., $\beta^\top v = 0$) and $\delta$ is its norm. Then we have the gradients
\begin{align}
    \nabla_\epsilon L &= (\nabla_\epsilon b)^\top (\nabla_b L) \\
    &= \frac{\beta}{\|\beta\|}^\top \left(c^2 \left(\epsilon \frac{\beta}{\|\beta\|} + \delta v \right) - c\beta\right) \\
    &= c^2\epsilon - c \|\beta\|, \\
    \nabla_\delta L &= (\nabla_\delta b)^\top (\nabla_b L) \\
    &= v^\top \left(c^2 \left(\epsilon \frac{\beta}{\|\beta\|} + \delta v \right) - c\beta\right) \\
    &= c^2 \delta, \\
    \nabla_c L &= \left(\epsilon \frac{\beta}{\|\beta\|} + \delta v \right)^\top \left( c \left(\epsilon \frac{\beta}{\|\beta\|} + \delta v \right) - \beta \right) + \alpha \|b_o\|^2 c \\
    &= c (\epsilon^2 + \delta^2 + \alpha \|b_o\|^2) - \epsilon \|\beta\|.
\end{align}
Finally, define the scalar quantity $o := \|b_o\|^2$, noting that $\nabla_o L = 2b_o^\top \nabla_{b_o} L = 2 \alpha c^2 o$. Minimizing this loss via gradient flow is therefore characterized by the following ODE on four scalars:
\begin{align}
    \dt{\epsilon} &= -c^2 \epsilon + c \|\beta\|, \\
    \dt{\delta} &= -c^2 \delta, \\
    \dt{o} &= -2\alpha c^2 o, \\
    \dt{c} &= -c (\epsilon^2 + \delta^2 + \alpha o) + \epsilon \|\beta\|. \\
\end{align}
Furthermore, we have the boundary conditions
\begin{align}
    \epsilon(0) &= \sqrt{\frac{1}{d_1+d_2}}, \\
    \delta(0) &= \sqrt{\frac{d_1-1}{d_1+d_2}}, \\
    o(0) &= \frac{d_2}{d_1+d_2}, \\
    c(0) &= 1.
\end{align}
Given these initializations and dynamics, we make a few observations: (i) all four scalars are initialized at a value greater than 0, and remain greater than 0 at all time steps; (ii) $\delta$ and $o$ will decrease towards 0 monotonically, and $\epsilon$ will increase monotonically until $c\epsilon = \|\beta\|$; (iii) $c$ will be decreasing at initialization. Lastly, for conciseness later on we define the quantities 
\begin{align}
    r &:= (\epsilon(0)^2 + \delta(0)^2 + \alpha o(0)) = \frac{d_1 + \alpha d_2}{d_1+d_2}, \\
    k &:= \frac{d_2}{d_1}, \\
    m &:= \frac{d_1}{d_1+d_2} = \frac{1}{1 + k}.
\end{align}

Before we can prove the main results, we present a lemma which serves as a key tool for deriving continuously valid bounds on the scalars we analyze:

\begin{lemma}
    Consider a vector valued ODE with scalar indices $v_1, v_2, \ldots$, where each index is described over the time interval $[t_{\min}, t_{\max}]$ by the continuous dynamics $\dt{v_i(t)} = a_i(v_{-i}(t)) \cdot v_i(t) + b_i(v_{-i}(t))$ with $a_i \leq 0, b_i \geq 0$ for all $i, t$ ($v_{-i}$ denotes the vector $v$ without index $i$). That is, each scalar's gradient is an affine function of that scalar with a negative coefficient. Suppose we define continuous functions $\hat a_i, \hat b_i : \R \to \R$ such that $\forall i, t$, $\hat a_i(t) \leq a_i(v_{-i}(t))$ and $\hat b_i(t) \leq b_i(v_{-i}(t))$. Let $\hat v$ be the vector described by these alternate dynamics, with the boundary condition $\hat v_i(t_{\min}) = v_i(t_{\min})$ and $v_i(t_{\min}) \geq 0$ for all $i$ (if a solution exists). Then for $t\in [t_{\min}, t_{\max}]$ it holds that
    \begin{align}
        \hat v(t) &\leq v(t),
    \end{align}
    elementwise. If $\hat a_i, \hat b_i$ upper bound $a_i, b_i$, the inequality is reversed.
\end{lemma}
\begin{proof}
    Define the vector $w(t) := \hat v(t) - v(t)$. This vector has the dynamics 
    \begin{align}
        \dt{w_i} &= \dt{\hat v_i} - \dt{v_i} \\
        &= \hat a_i(t) \cdot \hat v_i(t) + \hat b_i(t) - a_i(v_{-i}(t)) \cdot v_i(t) - b_i(v_{-i}(t)) \\
        &\leq \hat a_i(t) \cdot \hat v_i(t) - a_i(v_{-i}(t)) \cdot v_i(t).
    \end{align}
    The result will follow by showing that $w(t) \leq \mathbf{0}$ for all $t \in [t_{\min}, t_{\max}]$ (this clearly holds at $t_{\min}$). Assume for the sake of contradiction there exists a time $t' \in (t_{\min}, t_{\max}]$ and index $i$ such that $w_i(t') > 0$ (let $i$ be the first such index for which this occurs, breaking ties arbitrarily). By continuity, we can define $t_0 := \max\;\{t \in [t_{\min}, t'] \;:\; w_i(t) \leq 0\}$. By definition of $t_0$ it holds that $w_i(t_0) = 0$ and $\forall \epsilon>0$, $w_i(t_0 + \epsilon) - w_i(t_0) = w_i(t_0 + \epsilon) > 0$, and thus $\dt{w_i(t_0)} > 0$. But by the definition of $w$ we also have
    \begin{align}
        \hat v_i(t_0) &= v_i(t_0) + w_i(t_0) \\
        &= v_i(t_0),
    \end{align}
    and therefore 
    \begin{align}
        \dt{w_i(t_0)} &\leq \hat a_i(t_0) \cdot \hat v_i(t_0) - a_i(v_{-i}(t_0)) \cdot v_i(t_0) \\
        &= \bigl( \hat a_i(t_0) - a_i(v_{-i}(t_0)) \bigr) \cdot v_i(t_0) \\
        &\leq 0,
    \end{align}
    with the last inequality following because $\hat a_i(t) \leq a_i(v_{-i}(t))$ and $v_i(t) > 0$ for all $i, t \in [t_{\min}, t_{\max}]$. Having proven both $\dt{w_i(t_0)} > 0$ and $\dt{w_i(t_0)} \leq 0$, we conclude that no such $t'$ can exist. The other direction follows by analogous argument.
\end{proof}

We make use of this lemma repeatedly and its application is clear so we invoke it without direct reference. We are now ready to prove the main results:

\subsection{Proof of \cref{thm:sharpness-decrease}}
\begin{proof}
At initialization, we have $\|\beta\| \geq \frac{d_1}{\sqrt{d_1+d_2}} \implies \|\beta\|\epsilon(0) \geq \frac{d_1}{d_1+d_2} = c(0)(\epsilon(0)^2 + \delta(0)^2)$. Therefore, we can remove these terms from $\dt{c}$ at time $t=0$, noting simple that $\dt{c} \geq -\alpha o c$. Further, so long as $c$ is still decreasing (and therefore less than $c(0) = 1$),
\begin{align}
    \dt{(\|\beta\|\epsilon-c(\epsilon^2 + \delta^2))} &\geq \dt{(\|\beta\|\epsilon-(\epsilon^2 + \delta^2))} \\
    &= (\|\beta\| - 2\epsilon) \dt{\epsilon} - 2\delta \dt{\delta} \\
    &= (\|\beta\| - 2\epsilon) (-c^2 \epsilon + \|\beta\|c) - 2\delta (-c^2 \delta) \\
    &= -c^2 (\epsilon \|\beta\| - 2 (\epsilon^2 + \delta^2) + c(\|\beta\|^2 - 2\epsilon) \\
    &\geq -c (\epsilon \|\beta\| - 2 (\epsilon^2 + \delta^2)) + c(\|\beta\|^2 - 2\epsilon) \\
    &= c (\|\beta\|^2 - 2\epsilon - \epsilon \|\beta\| + 2 (\epsilon^2 + \delta^2)) \\
    \label{eq:non-negative}
    &\geq c (\|\beta\|^2 - \epsilon (2 + \|\beta\|)).
\end{align}
Since $c >0$ at all times, this is non-negative so long as the term in parentheses is non-negative, which holds so long as $\epsilon \leq \frac{\|\beta\|^2}{\|\beta\|+2}$. Further, since $\epsilon c \leq \|\beta\|$ we have
\begin{align}
    \dt{\epsilon^2} &= 2\epsilon \dt{\epsilon} \\
    &= -2c^2 \epsilon^2 + 2\epsilon c \|\beta\| \\
    &\leq 2 \|\beta\|^2.
\end{align}
This implies $\epsilon(t)^2 \leq \epsilon(0)^2 + 2t\|\beta\|^2$. Therefore, for $t \leq \frac{\ln \nicefrac{\|\beta\|}{2}}{2\|\beta\|}$ we have $\epsilon(t)^2 \leq \frac{1}{d_1+d_2} + \|\beta\| \ln \nicefrac{\|\beta\|}{2} \leq \frac{\|\beta\|^4}{(\|\beta\|+2)^2}$ (this inequality holds for $\|\beta\| \geq 2$). This satisfies the desired upper bound.

Thus the term in \cref{eq:non-negative} is non-negative for all $t \leq \frac{\ln \nicefrac{\|\beta\|}{2}}{2\|\beta\|}$, and so we have $\dt{c} \geq -\alpha o c$ under the above conditions. Since the derivative of $o$ is negative in $c$, a lower bound on $\dt{c}$ gives us an upper bound on $\dt{o}$, which in turn maintains a valid lower bound on $\dt{c}$ This allows us to solve for just the ODE given by 
\begin{align}
    \dt{c^2} &= -2\alpha c^2 o ,\\
    \dt{o} &= -2\alpha c^2 o.
\end{align}
Recalling the initial values of $c^2, o$, The solution to this system is given by
\begin{align}
    c(t)^2 &= \frac{m}{1 - \frac{(1-m)}{\exp(2\alpha m t)}},\\
    o(t) &= \frac{m }{\frac{\exp(2\alpha m t)}{1-m} - 1} \\
    &= \frac{m}{\exp(2\alpha m t) (1 + k^{-1}) - 1}
\end{align}
Since these are bounds on the original problem, we have $c(t)^2 \geq m$ and $o(t)$ shrinks exponentially fast in $t$. In particular, note that under the stated condition $\sqrt{\alpha} \geq \frac{\|\beta\| \ln k}{m (\ln \nicefrac{\|\beta\|}{2})}$ (recalling $k := \frac{d_2}{d_1} > 1$), we have $\frac{\ln k}{2\sqrt{\alpha} m} \leq \frac{\ln \nicefrac{\|\beta\|}{2}}{2\|\beta\|}$. Therefore we can plug in this value for $t$, implying $o(t) \leq m \left(\frac{d_1}{d_2}\right)^{\sqrt{\alpha}} = m k^{-\sqrt{\alpha}}$ at some time before $t=\frac{\ln \nicefrac{\|\beta\|}{2}}{2\|\beta\|}$.

Now we solve for the time at which $\dt{c} \geq 0$. Returning to \cref{eq:non-negative}, we can instead suppose that $\epsilon \leq \frac{\|\beta\|^2 - \gamma}{\|\beta\|+2} \implies \|\beta\|^2 - \epsilon (2 + \|\beta\|) \geq \gamma$ for some $\gamma > 0$. If this quantity was non-negative and has had a derivative of at least $\gamma$ until time $t = \frac{\ln k}{2\sqrt{\alpha} m}$, then its value at that time must be at least $\frac{\gamma \ln k}{2\sqrt{\alpha} m}$. For $\dt{c}$ to be non-negative, we need this to be greater than $c(t)^2 \alpha o(t)$, so it suffices to have $\frac{\gamma \ln k}{2\sqrt{\alpha} m} \geq  \frac{\alpha m}{\exp(2\alpha m t) (1 + k^{-1}) - 1} \impliedby \gamma \ln k \geq  \frac{2\alpha^{3/2} m^2}{k^{\sqrt{\alpha}} (1 + k^{-1}) - 1} \impliedby \gamma \geq \frac{2\alpha^{3/2} m^2 k^{-\sqrt{\alpha}}}{\ln k}$. Observe that the stated lower bound on $\alpha$ directly implies this inequality.

Finally, note that $\|b\|^2 = \epsilon^2 + \delta^2$, and therefore
\begin{align}
    \dt{\|b\|^2} &= 2\epsilon \dt{\epsilon} + 2\delta \dt{\delta} \\
    &= -2c^2 (\epsilon^2 + \delta^2) + 2 c \epsilon \|\beta\|.
\end{align}
Since $c(0) = 1$ and $c\epsilon < \|\beta\|$, this means $\|b\|^2$ will also be decreasing at initialization. Thus we have shown that all relevant quantities will decrease towards 0 at initialization, but that by time $t = \frac{\ln k}{2\sqrt{\alpha} m}$, we will have $\dt{c} \geq 0$.
\end{proof}

\subsection{Proof of Proof of \cref{thm:sharpness-increase}}
\begin{proof}
Recall from the previous section that we have shown that at some time $t_1 \leq \frac{\ln k}{2\sqrt{\alpha} m}$, $c(t)^2$ will be greater than $m$ and increasing, and $o(t)$ will be upper bounded by $m k^{-\sqrt{\alpha}}$. Furthermore, $\epsilon(t)^2 \leq \frac{1}{d_1+d_2} + 2t\|\beta\|^2$. To show that the sharpness reaches a particular value, we must demonstrate that $c$ grows large enough before the point $c\epsilon \approx \|\beta\|$ where this growth will rapidly slow. To do this, we study the relative growth of $c$ vs. $\epsilon$.

Recall the derivatives of these two terms:
\begin{align}
    \dt{c} &= -(\epsilon^2 + \delta^2 + \alpha \outlierbeta^2) c + \|\beta\| \epsilon, \\
    \dt{\epsilon} &= -c^2 \epsilon + \|\beta\| c.
\end{align}
Considering instead their squares,
\begin{align}
    \dt{c^2} &= 2c \dt{c} \\
    &= -2 (\epsilon^2 + \delta^2 + \alpha \outlierbeta^2) c^2 + 2 \|\beta\| \epsilon c, \\
    \dt{\epsilon^2} &= 2\epsilon \dt{\epsilon} \\
    &= -2\epsilon^2 c^2 + 2 \|\beta\| \epsilon c.
\end{align}
Since $\delta, o$ decrease monotonically, we have $\dt{c^2} \geq -2(\epsilon^2 + \frac{d_1}{d_1+d_2} + \alpha m \left(\frac{d_1}{d_2}\right)^{\sqrt{\alpha}}) c^2 + 2\|\beta\| \epsilon$. Thus if we can show that 
\begin{align}
    \|\beta\| \epsilon c &\geq (\epsilon^2 + 2 (\frac{d_1}{d_1+d_2} + \alpha m \left(\frac{d_1}{d_2}\right)^{\sqrt{\alpha}})) c^2,
\end{align}
we can conclude that $\dt{c^2} \geq (\epsilon^2 c^2 + \|\beta\| \epsilon c) = \frac{1}{2} \dt{\epsilon^2}$---that is, that $c(t)^2$ grows at least half as fast as $\epsilon(t)^2$. And since $\delta, o$ continue to decrease, this inequality will continue to hold thereafter.

Simplifying the above desired inequality, we get
\begin{align}
    \|\beta\| \frac{\epsilon}{c} &\geq \epsilon^2 + 2 m (1 + \alpha k^{-\sqrt{\alpha}}).
\end{align}
Noting that $\frac{\epsilon}{c} \geq 1$ and $m = \frac{d_1}{d_1+d_2} \leq \frac{1}{2}$, and recalling the upper bound on $\epsilon(t)^2$, this reduces to proving
\begin{align}
    \|\beta\| &\geq \frac{1}{d_1+d_2} + 2t\|\beta\|^2 + 1 + \alpha k^{-\sqrt{\alpha}}.
\end{align}
Since this occurs at some time $t_1 \leq \frac{\ln k}{2\sqrt{\alpha}m}$, and since $m^{-1} = 1+k$, we get
\begin{align}
    \|\beta\| &\geq \frac{1}{d_1+d_2} + \frac{\|\beta\|^2 (1+k) \ln k}{\sqrt{\alpha}} + 1 + \alpha k^{-\sqrt{\alpha}}.
\end{align}
The assumed lower bound on $\sqrt{\alpha}$ means the sum of the first three terms can be upper bounded by a small $1+o(1)$ term (say, $\nicefrac{9}{5}$) and recalling $\|\beta\| \geq \nicefrac{24}{5}$ it suffices to prove
\begin{align}
    &\|\beta\| \geq \frac{9}{5} + \alpha k^{-\sqrt{\alpha}} \\
    \impliedby &\alpha k^{-\sqrt{\alpha}} \leq 3.
\end{align}
Taking logs,
\begin{align}
    \frac{2 \ln \sqrt{\alpha}}{\ln k} - \sqrt{\alpha} \leq \ln 3,
\end{align}
which is clearly satisfied for $\sqrt{\alpha} \geq 1 + k \ln k$. As argued above, this implies $\dt{c^2} \geq \frac{1}{2} \dt{\epsilon^2}$ by some time $t_2 \leq \frac{\ln k}{2\sqrt{\alpha} m}$.

Consider the time $t_2$ at which this first occurs, whereby $c(t_2)^2$ is growing by at least one-half the rate of $\epsilon(t_2)^2$. Here we note that we can derive an upper bound on $c$ and $\epsilon$ at this time using our lemma and the fact that
\begin{align}
    \dt{c} &\leq \|\beta\| \epsilon, \\
    \dt{\epsilon} &\leq \|\beta\| c.
\end{align}
The solution to this system implies
\begin{align}
    c(t_2) & \leq \frac{1}{2} \left( \frac{\exp(\|\beta\| t_2) - \exp(-\|\beta\| t_2)}{\sqrt{d_1+d_2}} + \exp(\|\beta\| t_2) + 1 \right) \\
    &\leq \frac{1}{2} \left( \exp(\|\beta\| t_2) \left(1 + \frac{1}{\sqrt{d_1+d_2}} \right) + 1 \right) \\
    &\leq \frac{1}{2} \left( \exp\left( \frac{\|\beta\| \ln k}{2\sqrt{\alpha} m}\right) \left(1 + \frac{1}{\sqrt{d_1+d_2}} \right) + 1 \right), \\
    \epsilon(t_2) &\leq \frac{1}{2} \left( \exp(\|\beta\| t_2) \left(1 + \frac{1}{\sqrt{d_1+d_2}} \right) + \frac{1}{\sqrt{d_1+d_2}} - 1 \right) \\
    &\leq \frac{1}{2} \left( \exp\left( \frac{\|\beta\| \ln k}{2\sqrt{\alpha} m}\right) \left(1 + \frac{1}{\sqrt{d_1+d_2}} \right) + \frac{1}{\sqrt{d_1+d_2}} - 1 \right)
\end{align}

Then for $\alpha > \left( \frac{\|\beta\| \ln k}{m (\ln \|\beta\| - \ln 2)} \right)^2$, the exponential term is upper bounded by $\frac{\sqrt{\|\beta\|}}{2}$, giving
\begin{align}
    c(t_2) & \leq \frac{1}{2} \left( \frac{\sqrt{\|\beta\|}}{2} \left(1 + \frac{1}{\sqrt{d_1+d_2}} \right) + 1 \right) \\
    &\leq \frac{\sqrt{\|\beta\|}}{2}, \\
    \epsilon(t_2) &\leq \frac{1}{2} \left( \frac{\sqrt{\|\beta\|}}{2} \left(1 + \frac{1}{\sqrt{d_1+d_2}} \right) + \frac{1}{\sqrt{d_1+d_2}} - 1 \right) \\
    &\leq \frac{\sqrt{\|\beta\|}}{2}.
\end{align}
We know that optimization will continue until $\epsilon^2 c^2 = \|\beta\|^2$, and also that $\dt{c^2} \geq \frac{1}{2} \dt{\epsilon^2}$. Since $c \leq \epsilon$, this implies that $\epsilon^2 \geq \|\beta\|$ before convergence. Suppose that starting from time $t_2$, $\epsilon^2$ grows until time $t'$ by an additional amount $s$. Then we have
\begin{align}
    s &= \epsilon(t')^2 - \epsilon(t_2)^2 \\
    &= \int_{t_2}^{t'} \dt{\epsilon(t)^2} \\
    &\leq \int_{t_2}^{t'} 2\dt{c(t)^2} \\
    &= 2 (c(t')^2 - c(t_2)^2).
\end{align}
In other words, $c^2$ must have grown by at least half that amount. Since $\epsilon(t_2)^2 \leq \frac{\|\beta\|}{4}$ and therefore $\epsilon(t')^2 \leq \frac{\|\beta\|}{4} + s$, even if $c(t')^2$ is the minimum possible value of $\frac{s}{2}$ we must have at convergence $\frac{s}{2} = c^2 = \frac{\|\beta\|^2}{\epsilon^2} \geq \frac{\|\beta\|^2}{\frac{\|\beta\|}{4} + s}$.
This is a quadratic in $s$ and solving tells us that we must have $s \geq \frac{5}{4} \|\beta\|$. Therefore, $c(t')^2 \geq \frac{5}{8} \|\beta\|$ is guaranteed to occur. Noting our derivation of the loss Hessian, this implies the sharpness must reach at least $\frac{5}{8} \alpha \|\beta\|$ for each dimension of $b_o$.
\end{proof}

\section{Additional Samples Under Various Architectures/Seeds}
\label{app:addl-images}

To demonstrate the robustness of our finding we train a ResNet-18, VGG-11, and a Vision Transformer for 1000 steps with full-batch GD, each with multiple random initializations. For each run, we identify the 24 training examples with the most positive and most negative change in loss from step $i$ to step $i+1$, for $i\in\{100, 250, 500, 750\}$. We then display these images along with their label (above) and the network's predicted label before and after the gradient step (below). The change in the network's predicted labels display a clear pattern, where certain training samples cause the network to associate an opposing signal with a new class, which the network then overwhelmingly predicts whenever that feature is present.

Consistent with our other experiments, we find that early opposing signals tend to be ``simpler'', e.g. raw colors, whereas later signals are more nuanced, such as the presence of a particular texture. We also see that the Vision Transformer seems to learn complex features earlier, and that they are less obviously aligned with human perception---this is not surprising since they process inputs in a fundamentally different manner than traditional ConvNets.

\begin{figure}[ht!]
    \centering
    \begin{subfigure}[b]{.475\linewidth}
        \includegraphics[width=\linewidth]{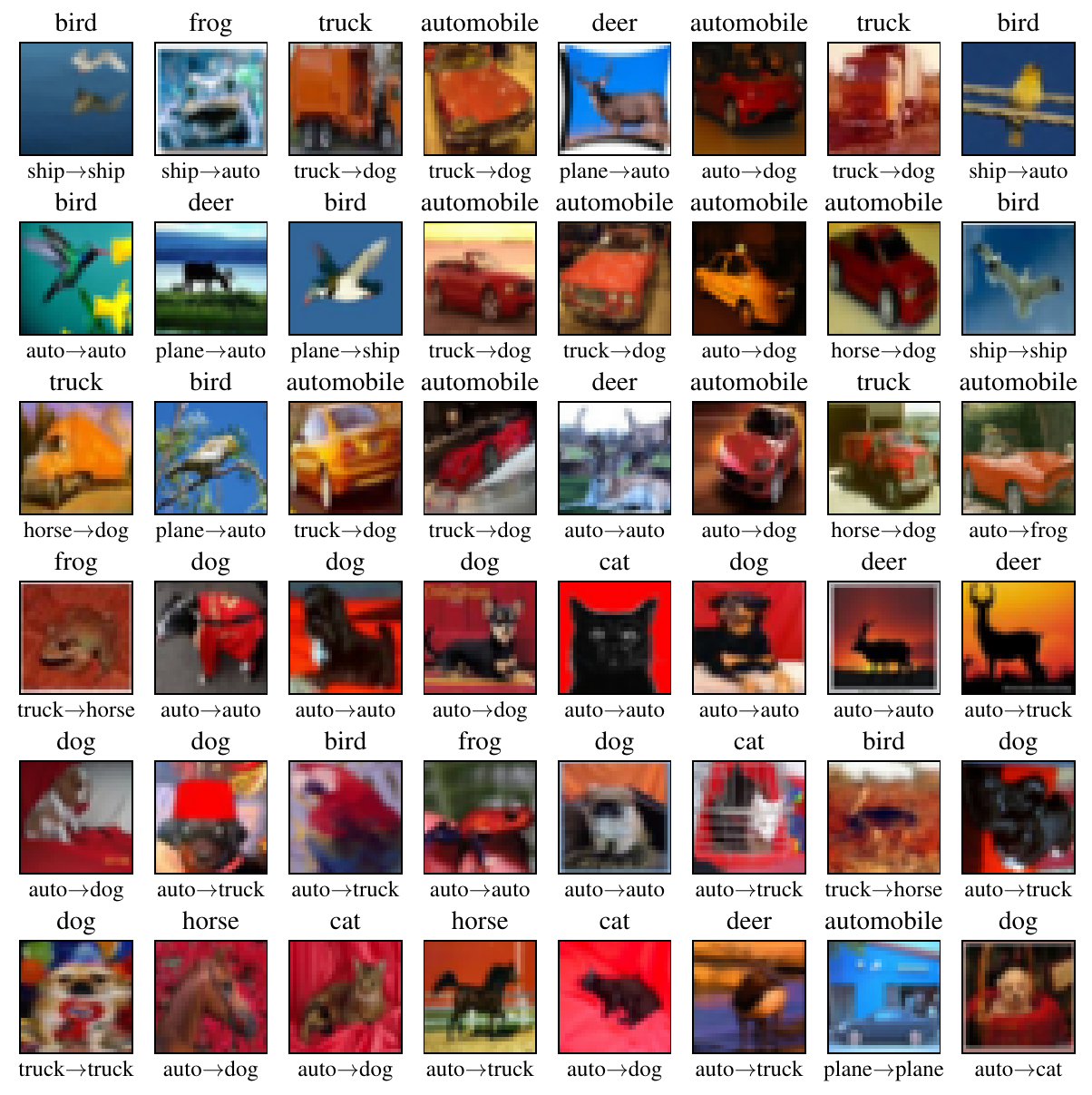}
        \caption{Step 100 to 101}
    \end{subfigure}
    \hfill
    \begin{subfigure}[b]{.475\linewidth}
        \includegraphics[width = \linewidth]{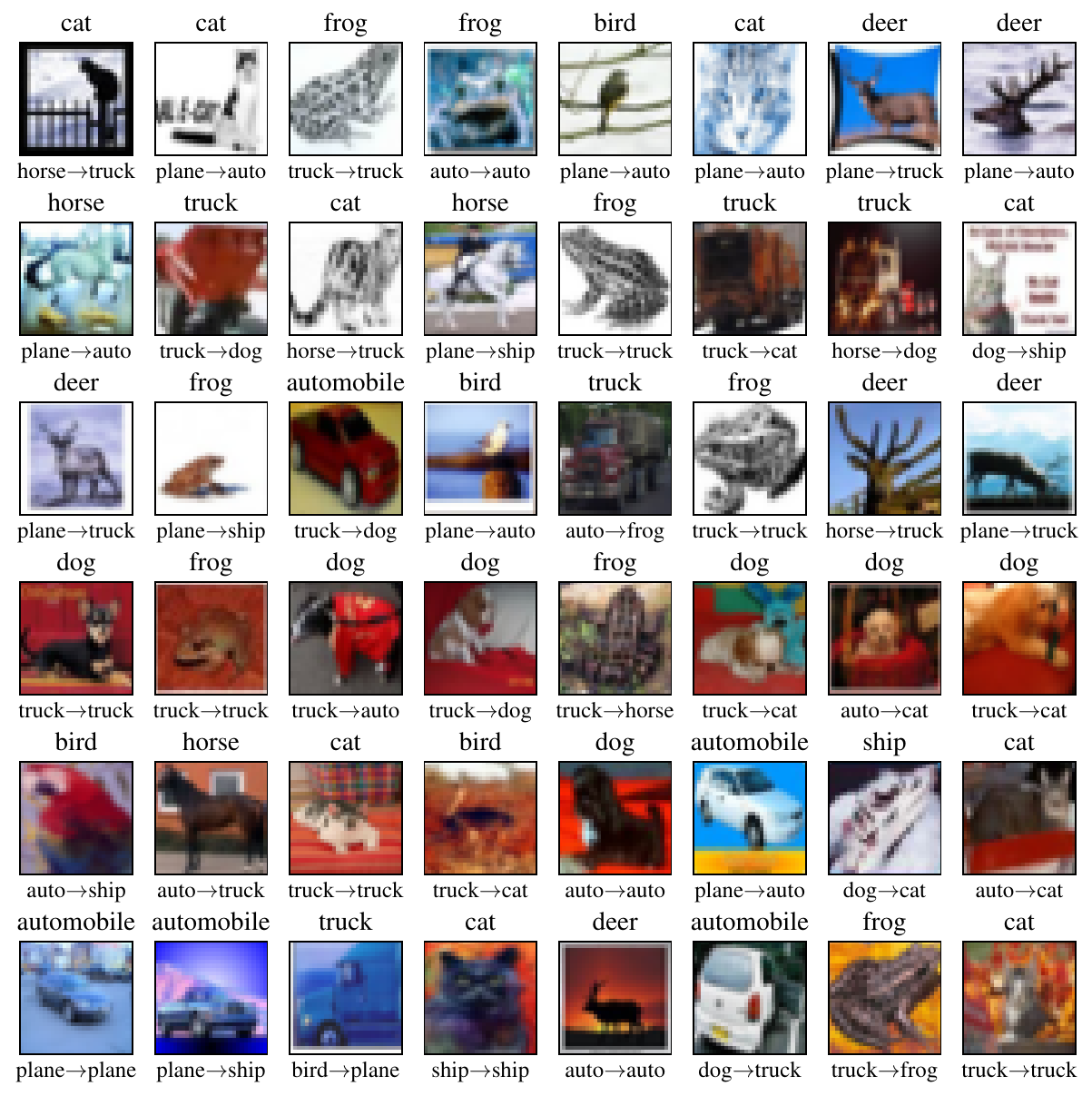}
        \caption{Step 250 to 251}
    \end{subfigure}
    \vskip\baselineskip
    \begin{subfigure}[b]{.475\linewidth}
        \includegraphics[width = \linewidth]{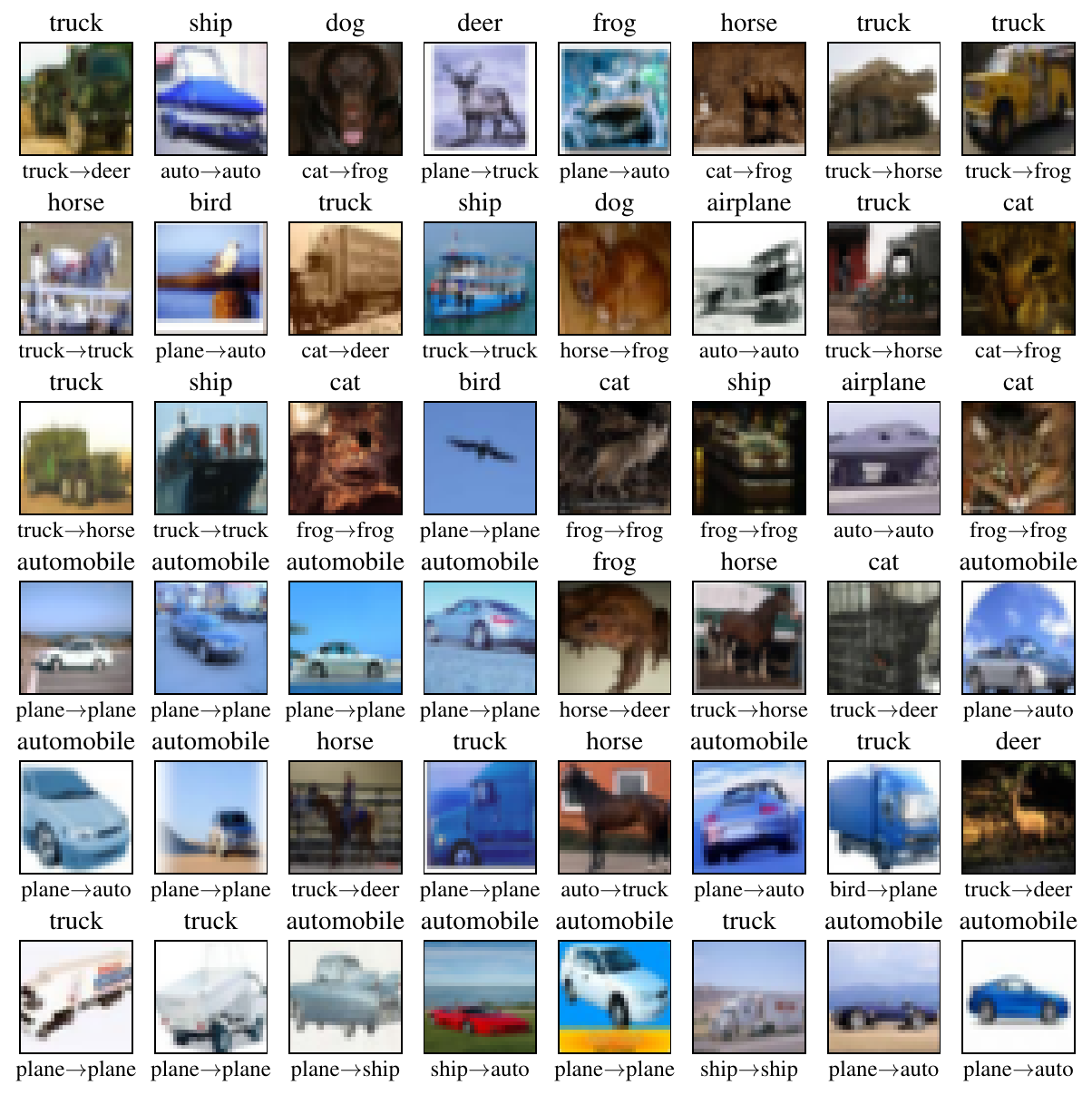}
        \caption{Step 500 to 501}
    \end{subfigure}
    \hfill
    \begin{subfigure}[b]{.475\linewidth}
        \includegraphics[width = \linewidth]{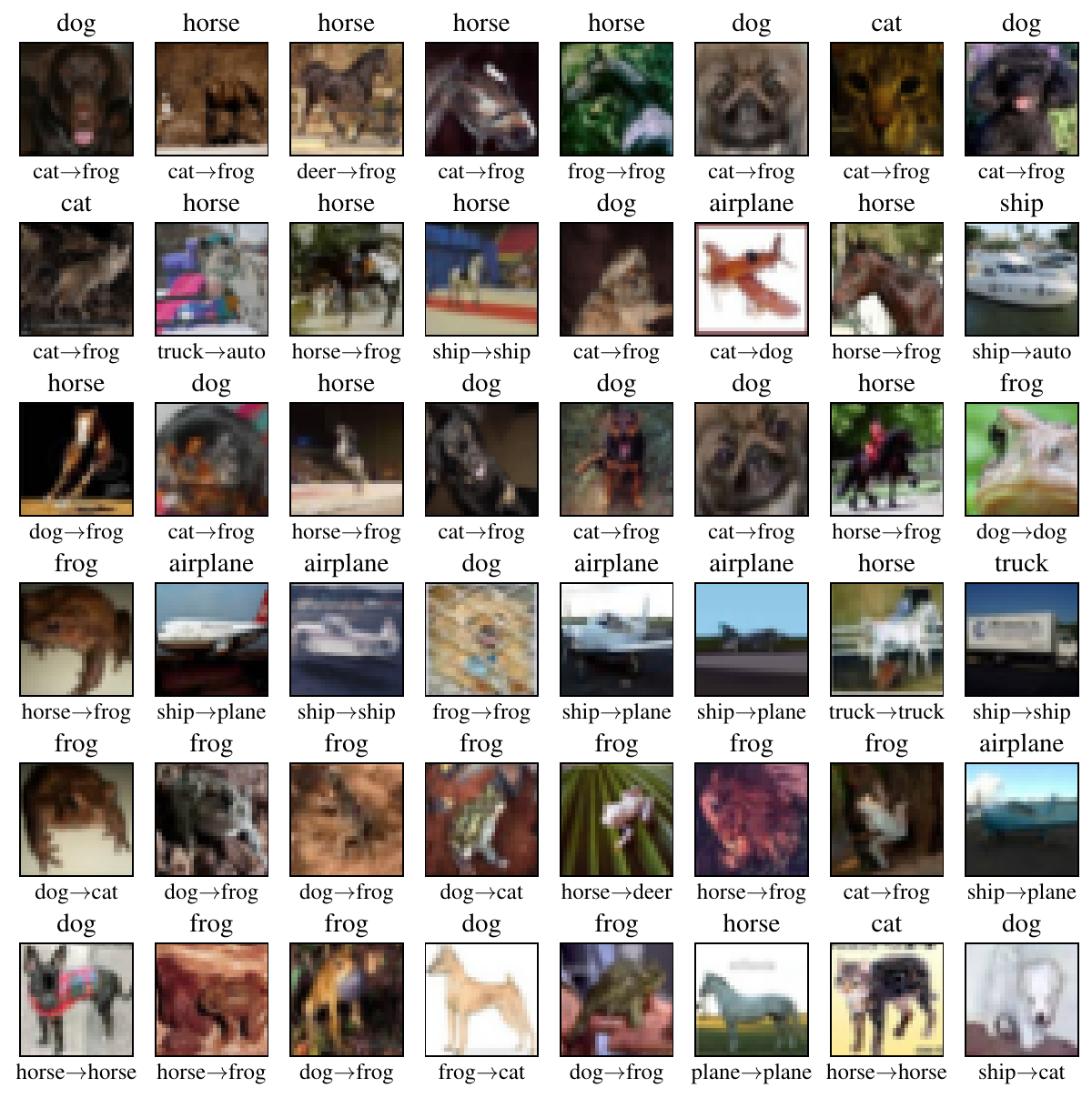}
        \caption{Step 750 to 751}
    \end{subfigure}
    \caption{\textbf{(ResNet-18, seed 1)} Images with the most positive (top 3 rows) and most negative (bottom 3 rows) change to training loss after steps 100, 250, 500, and 750. Each image has the true label (above) and the predicted label before and after the gradient update (below).}
\end{figure}

\begin{figure}[ht!]
    \centering
    \begin{subfigure}[b]{.475\linewidth}
        \includegraphics[width = \linewidth]{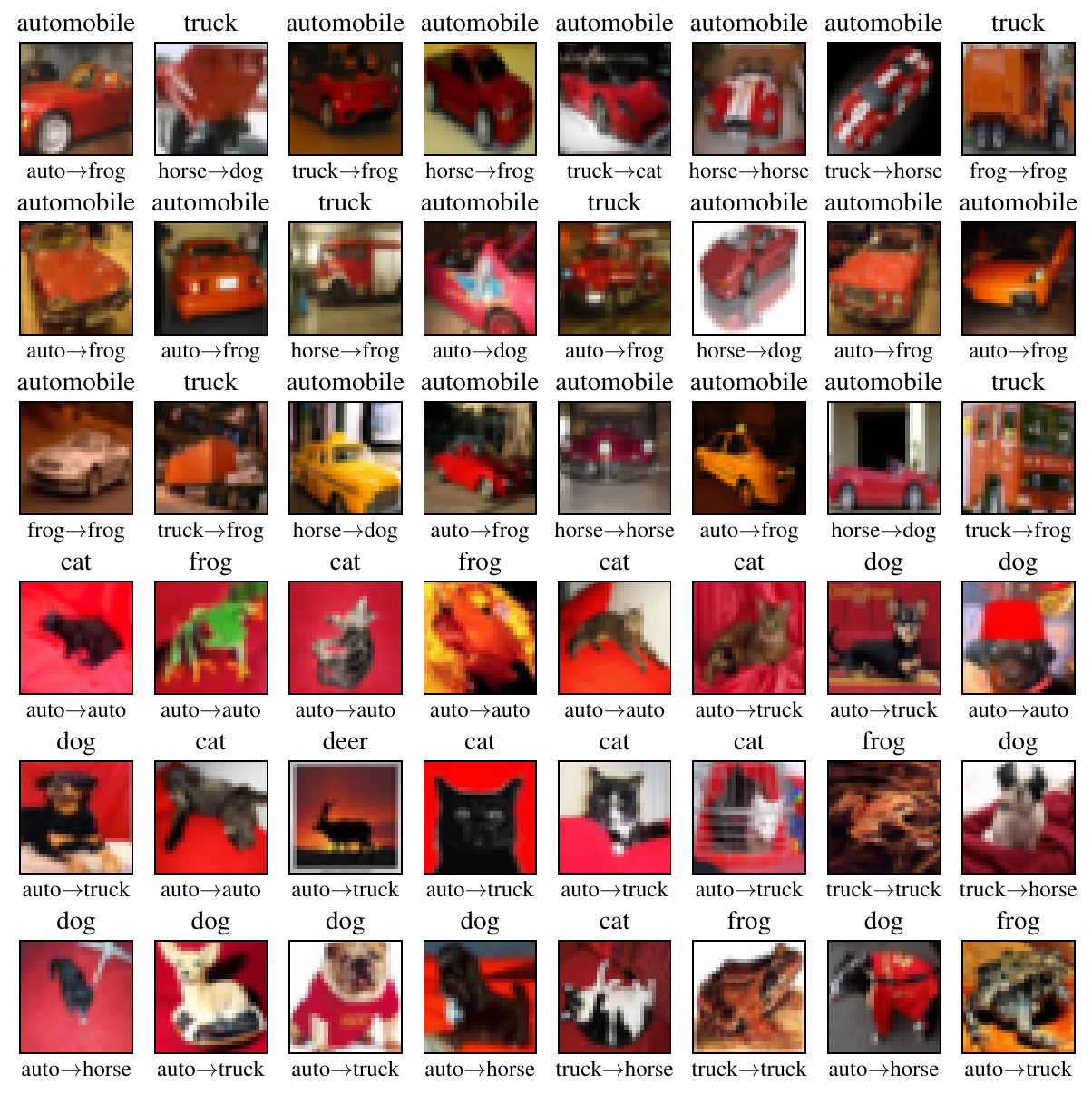}
        \caption{Step 100 to 101}
    \end{subfigure}
    \hfill
    \begin{subfigure}[b]{.475\linewidth}
        \includegraphics[width = \linewidth]{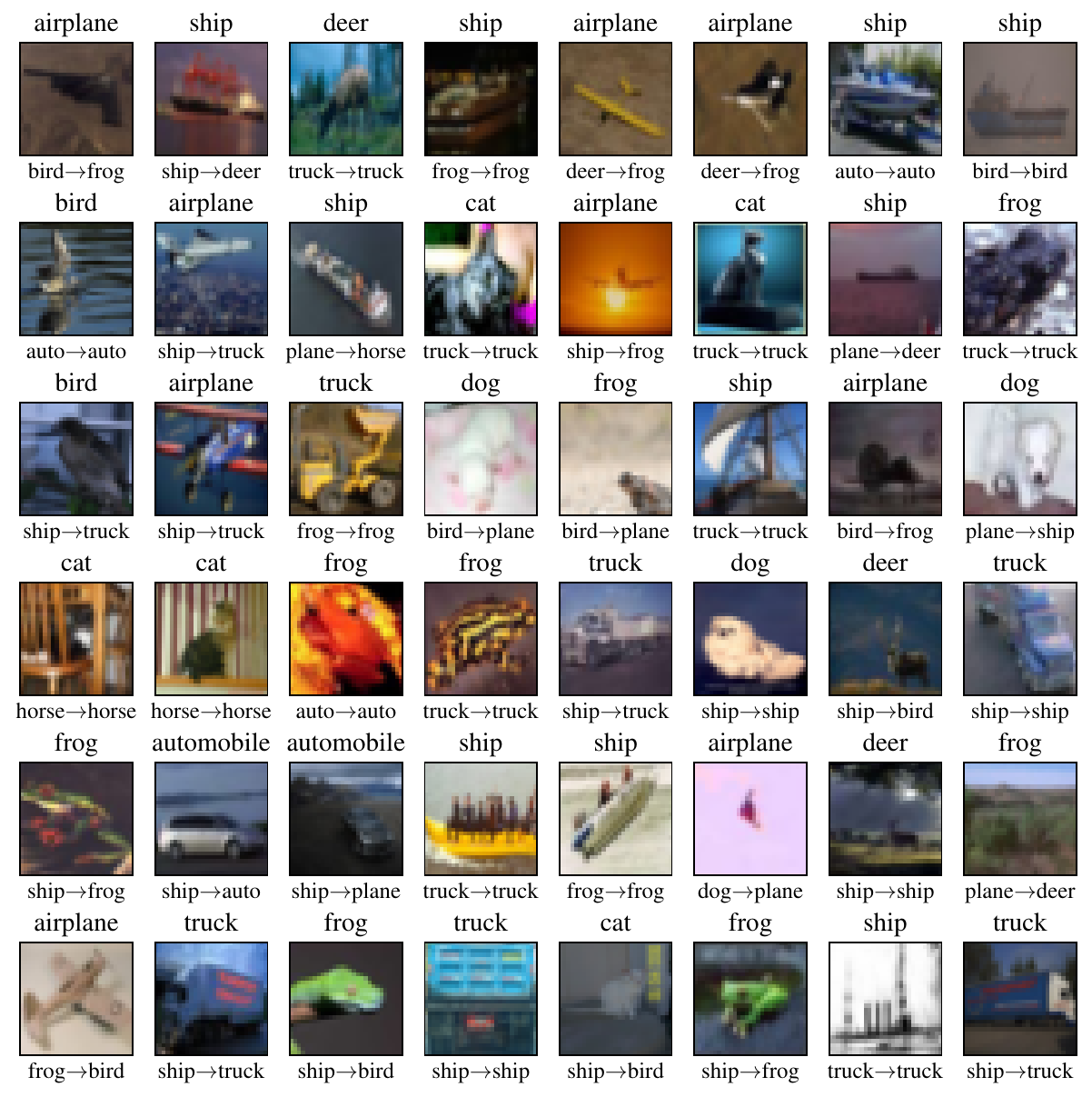}
        \caption{Step 250 to 251}
    \end{subfigure}
    \vskip\baselineskip
    \begin{subfigure}[b]{.475\linewidth}
        \includegraphics[width = \linewidth]{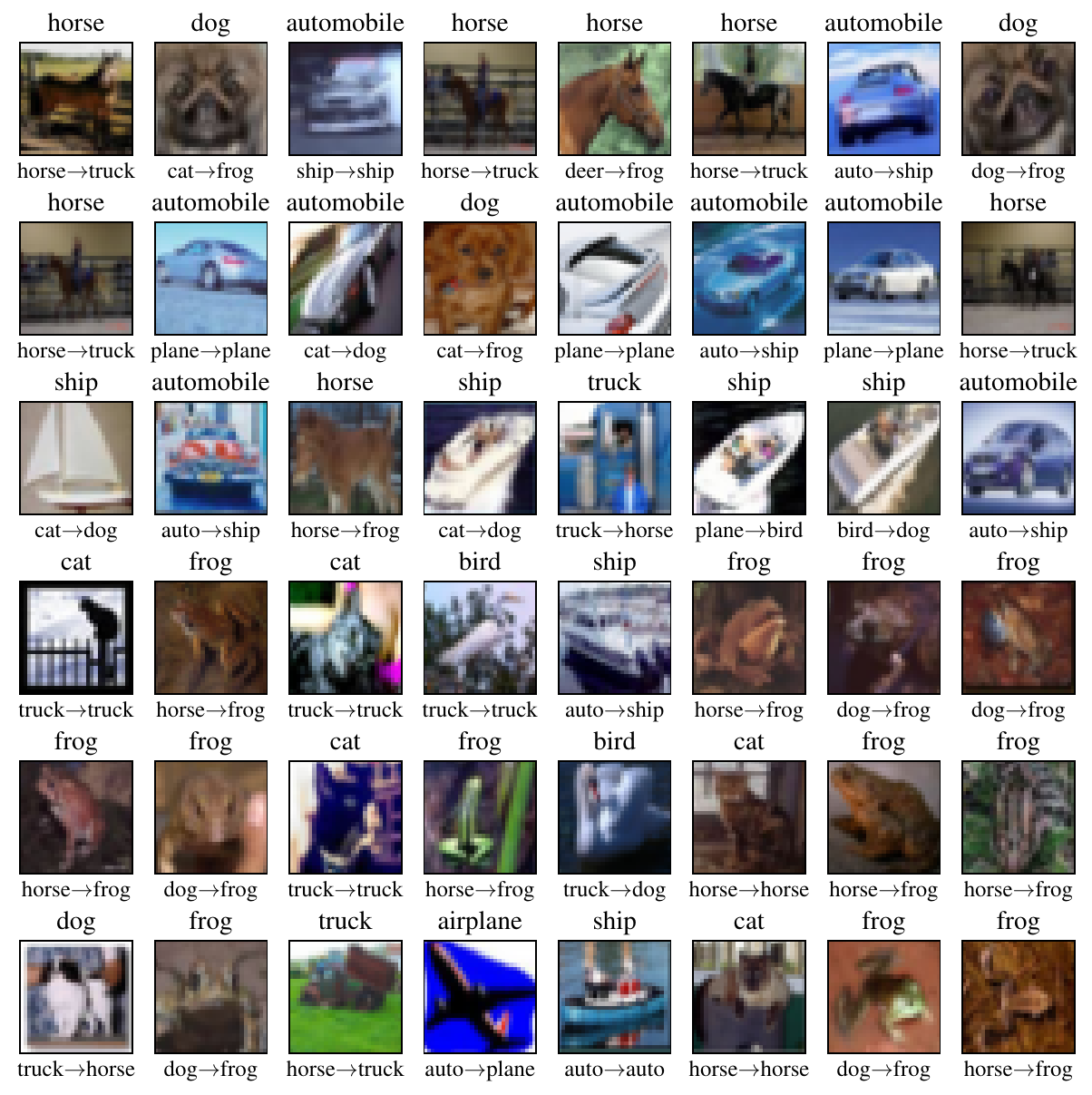}
        \caption{Step 500 to 501}
    \end{subfigure}
    \hfill
    \begin{subfigure}[b]{.475\linewidth}
        \includegraphics[width = \linewidth]{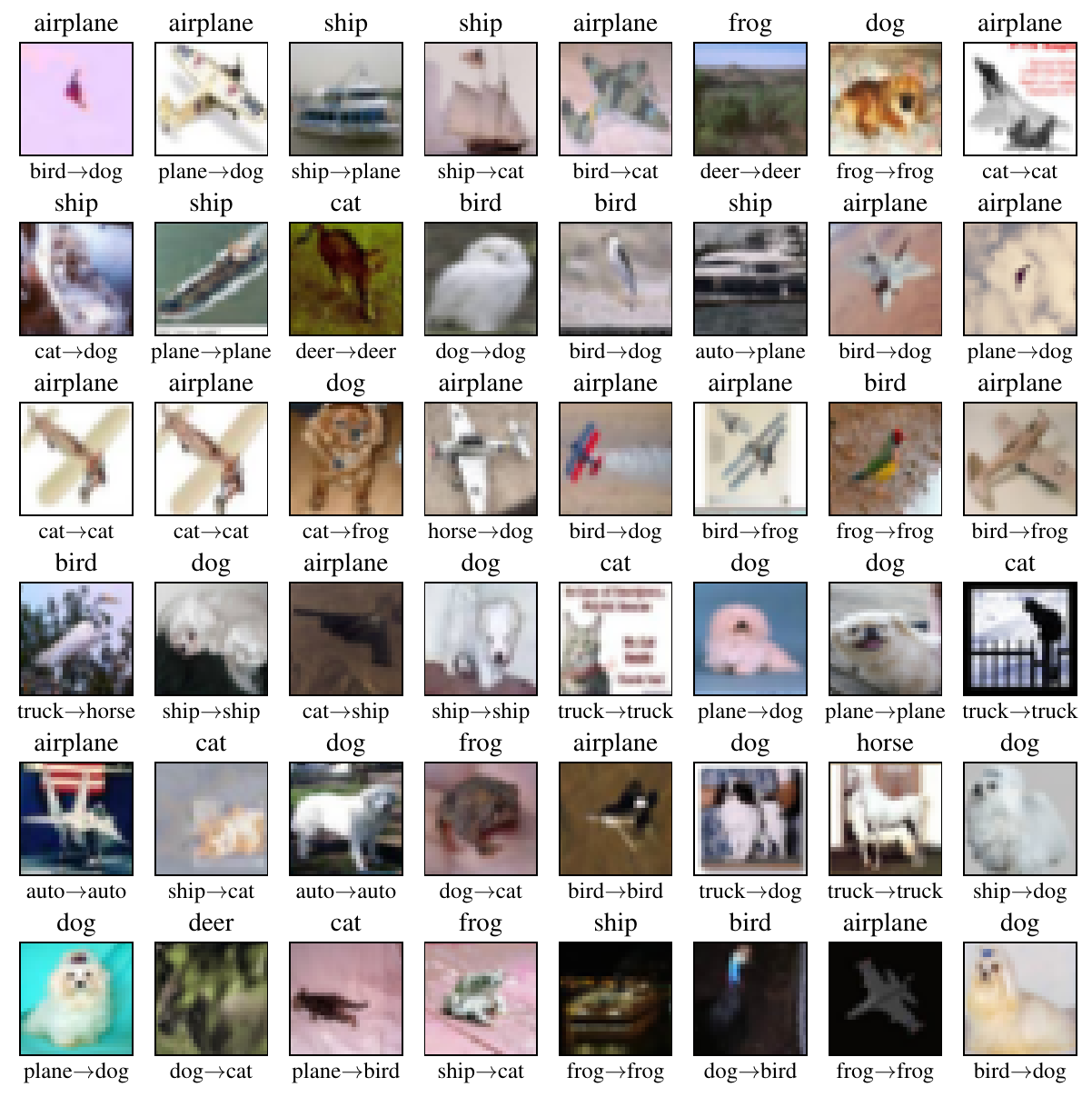}
        \caption{Step 750 to 751}
    \end{subfigure}
    \caption{\textbf{(ResNet-18, seed 2)} Images with the most positive (top 3 rows) and most negative (bottom 3 rows) change to training loss after steps 100, 250, 500, and 750. Each image has the true label (above) and the predicted label before and after the gradient update (below).}
\end{figure}

\begin{figure}[ht!]
    \centering
    \begin{subfigure}[b]{.475\linewidth}
        \includegraphics[width = \linewidth]{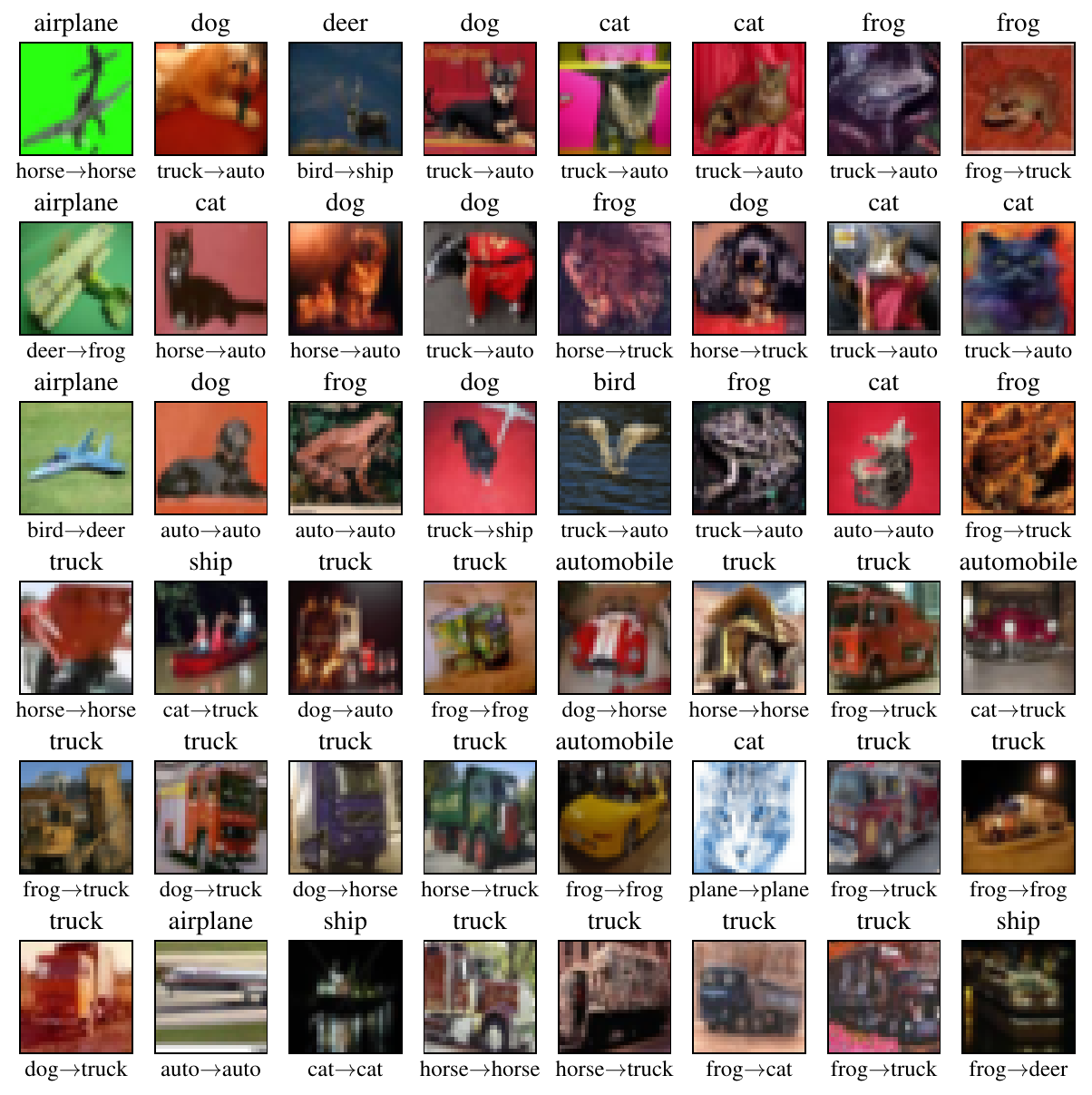}
        \caption{Step 100 to 101}
    \end{subfigure}
    \hfill
    \begin{subfigure}[b]{.475\linewidth}
        \includegraphics[width = \linewidth]{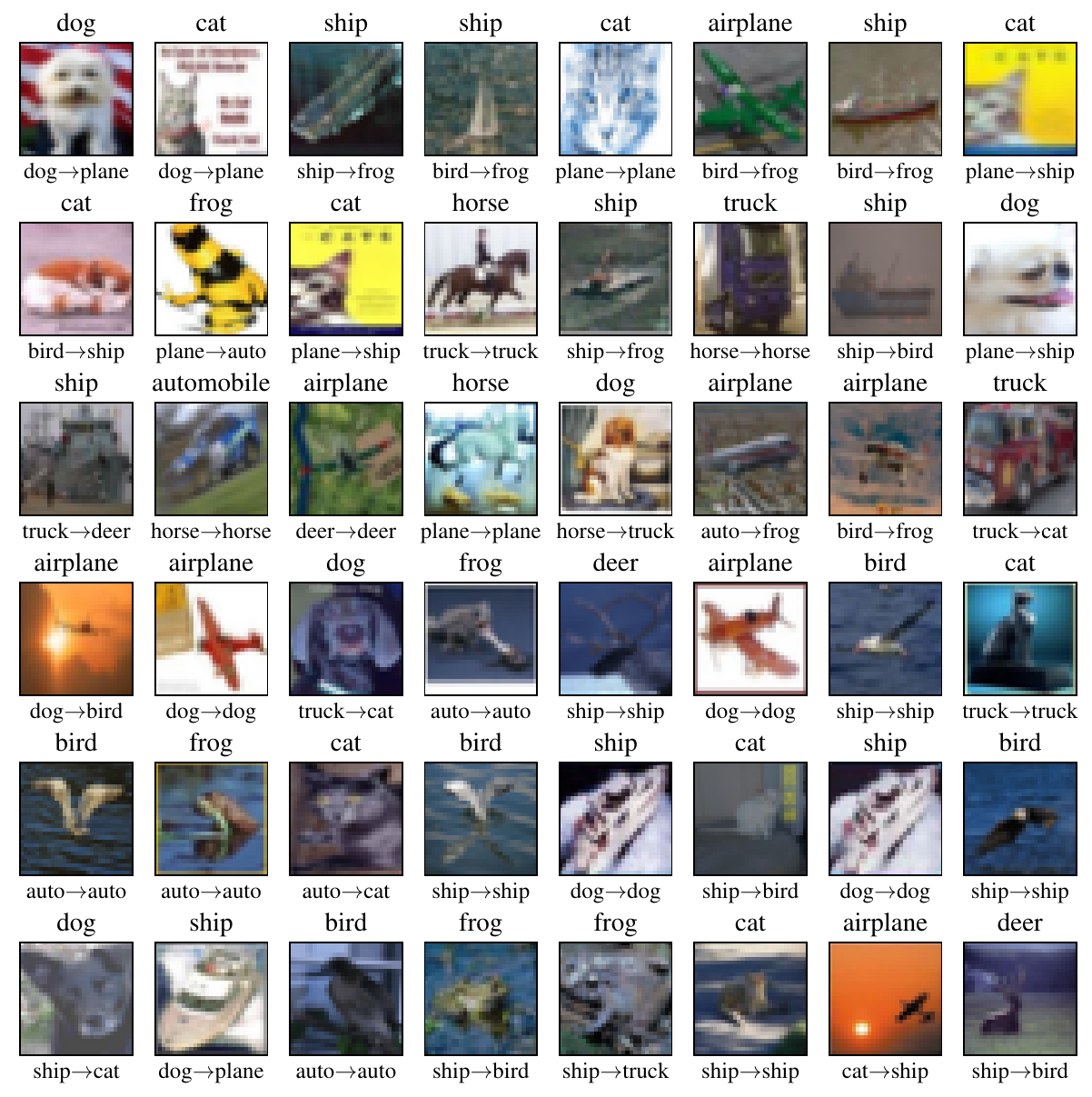}
        \caption{Step 250 to 251}
    \end{subfigure}
    \vskip\baselineskip
    \begin{subfigure}[b]{.475\linewidth}
        \includegraphics[width = \linewidth]{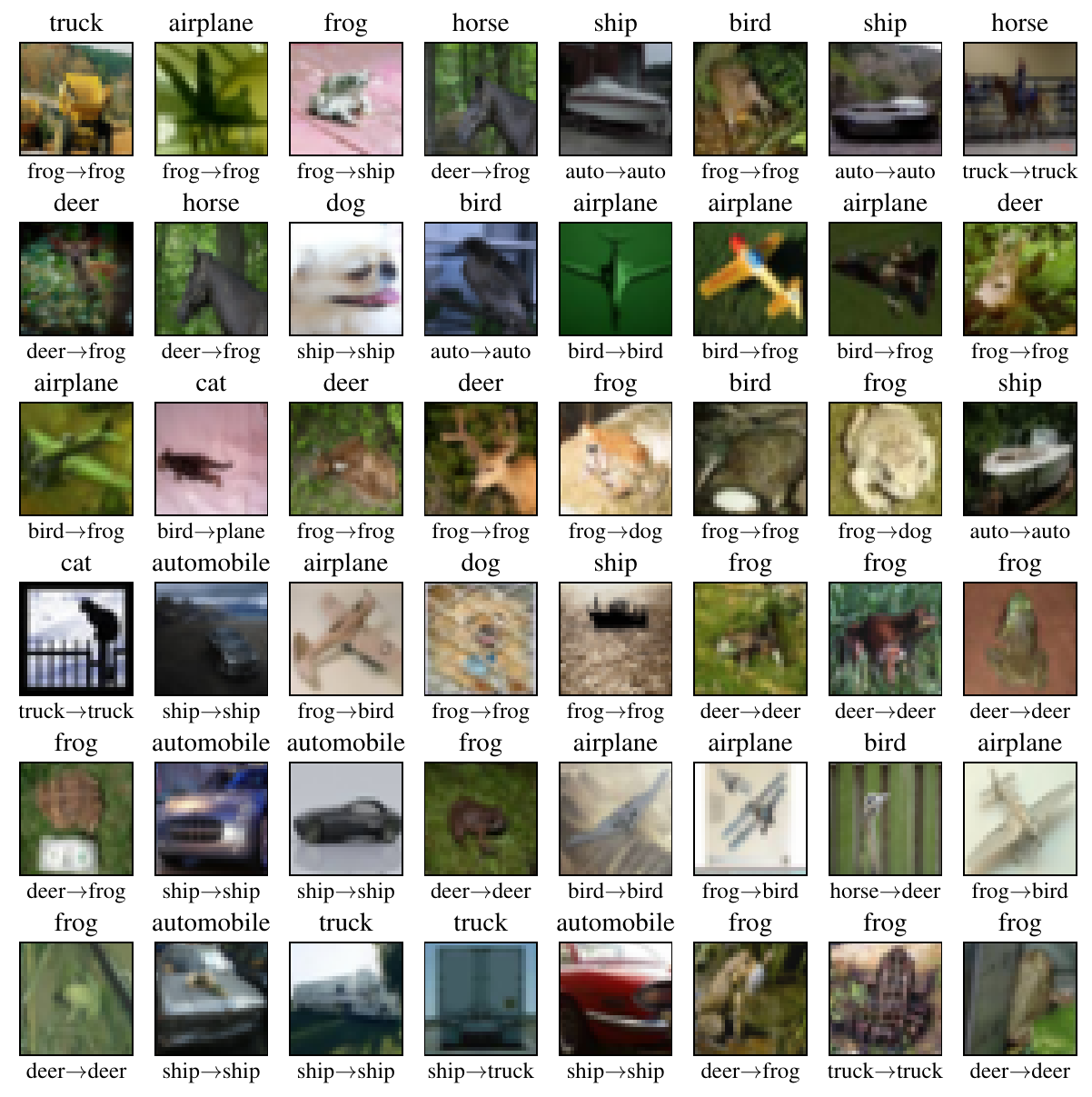}
        \caption{Step 500 to 501}
    \end{subfigure}
    \hfill
    \begin{subfigure}[b]{.475\linewidth}
        \includegraphics[width = \linewidth]{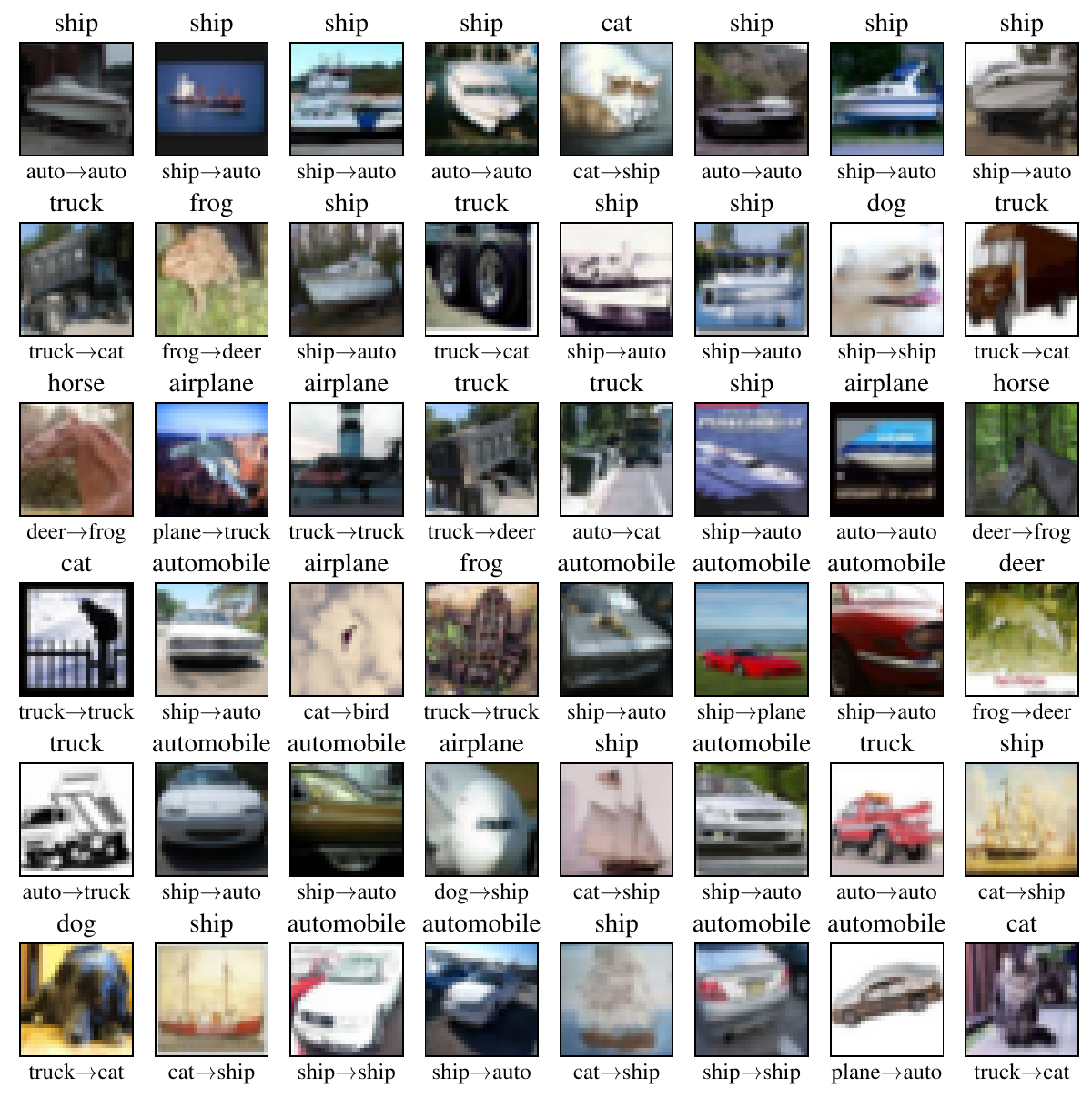}
        \caption{Step 750 to 751}
    \end{subfigure}
    \caption{\textbf{(ResNet-18, seed 3)} Images with the most positive (top 3 rows) and most negative (bottom 3 rows) change to training loss after steps 100, 250, 500, and 750. Each image has the true label (above) and the predicted label before and after the gradient update (below).}
\end{figure}

\begin{figure}[ht!]
    \centering
    \begin{subfigure}[b]{.475\linewidth}
        \includegraphics[width = \linewidth]{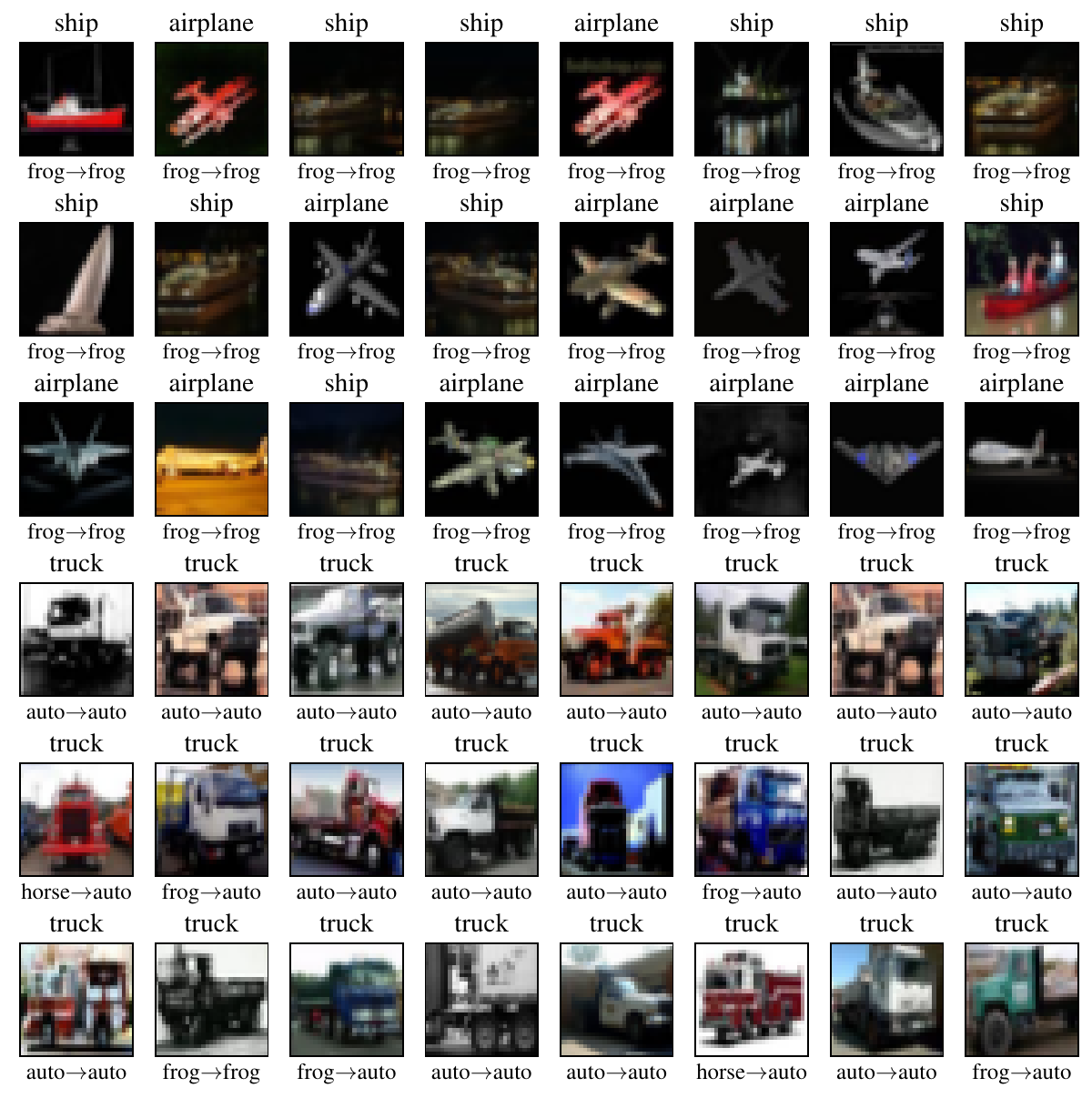}
        \caption{Step 100 to 101}
    \end{subfigure}
    \hfill
    \begin{subfigure}[b]{.475\linewidth}
        \includegraphics[width = \linewidth]{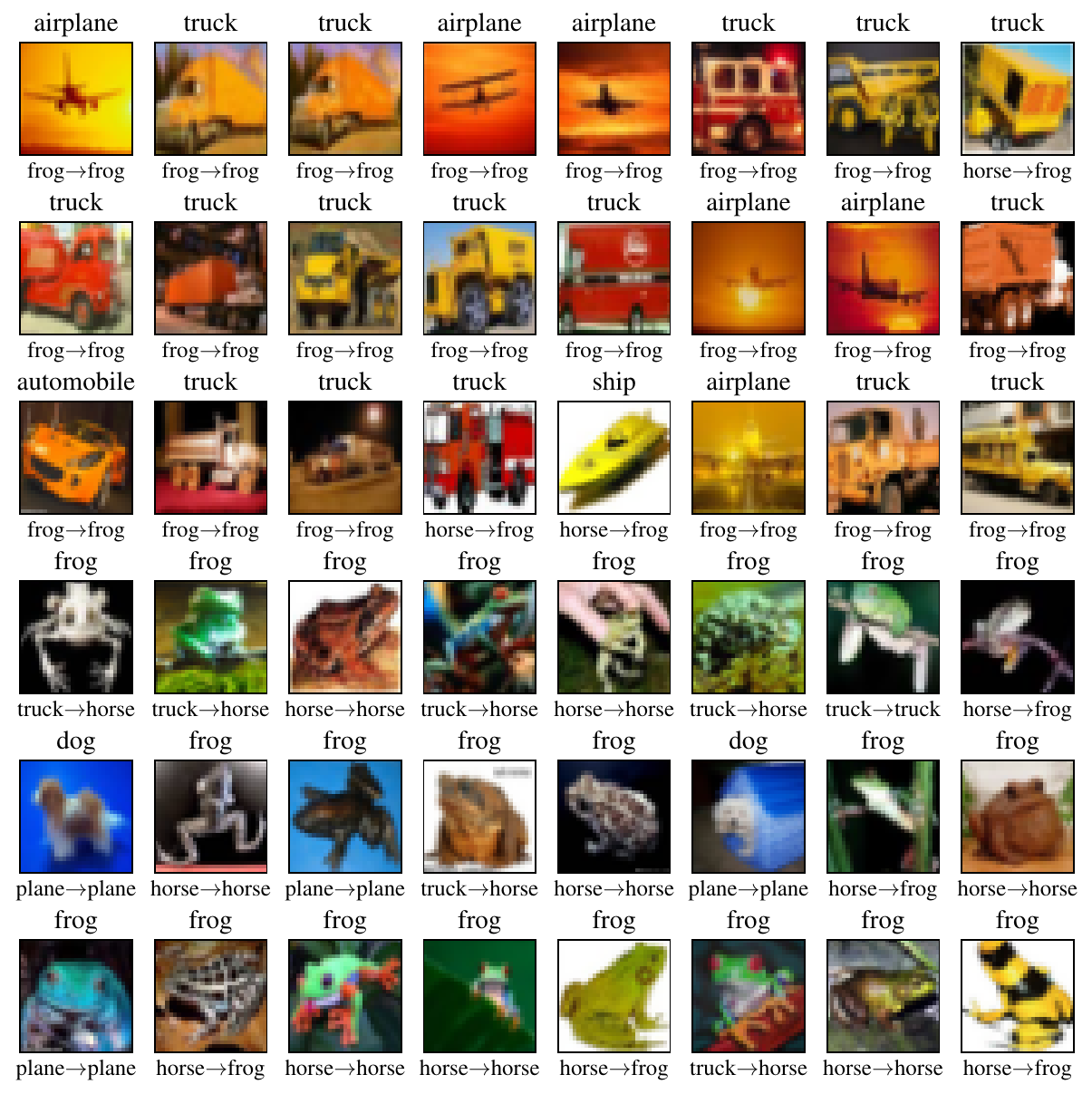}
        \caption{Step 250 to 251}
    \end{subfigure}
    \vskip\baselineskip
    \begin{subfigure}[b]{.475\linewidth}
        \includegraphics[width = \linewidth]{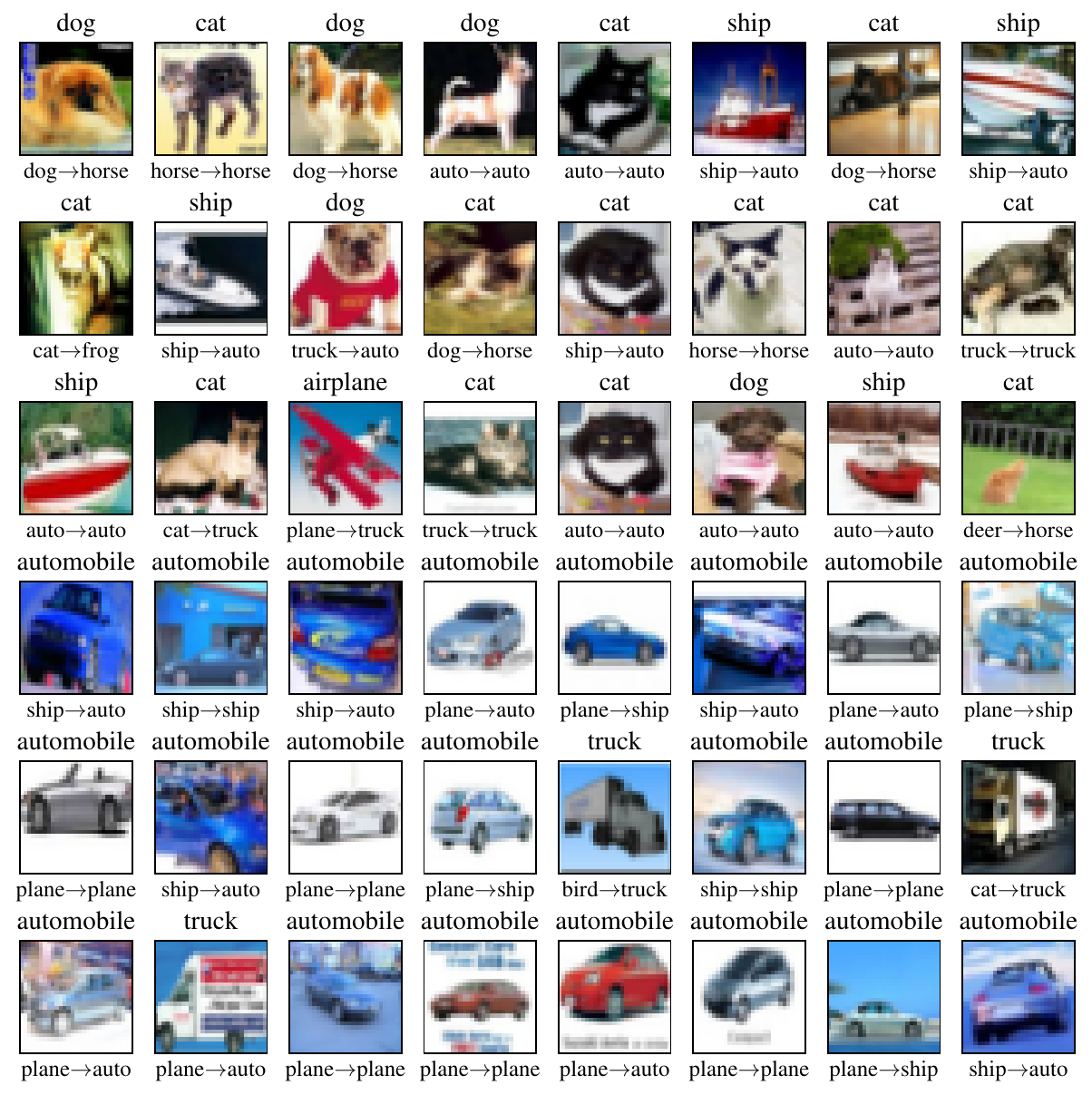}
        \caption{Step 500 to 501}
    \end{subfigure}
    \hfill
    \begin{subfigure}[b]{.475\linewidth}
        \includegraphics[width = \linewidth]{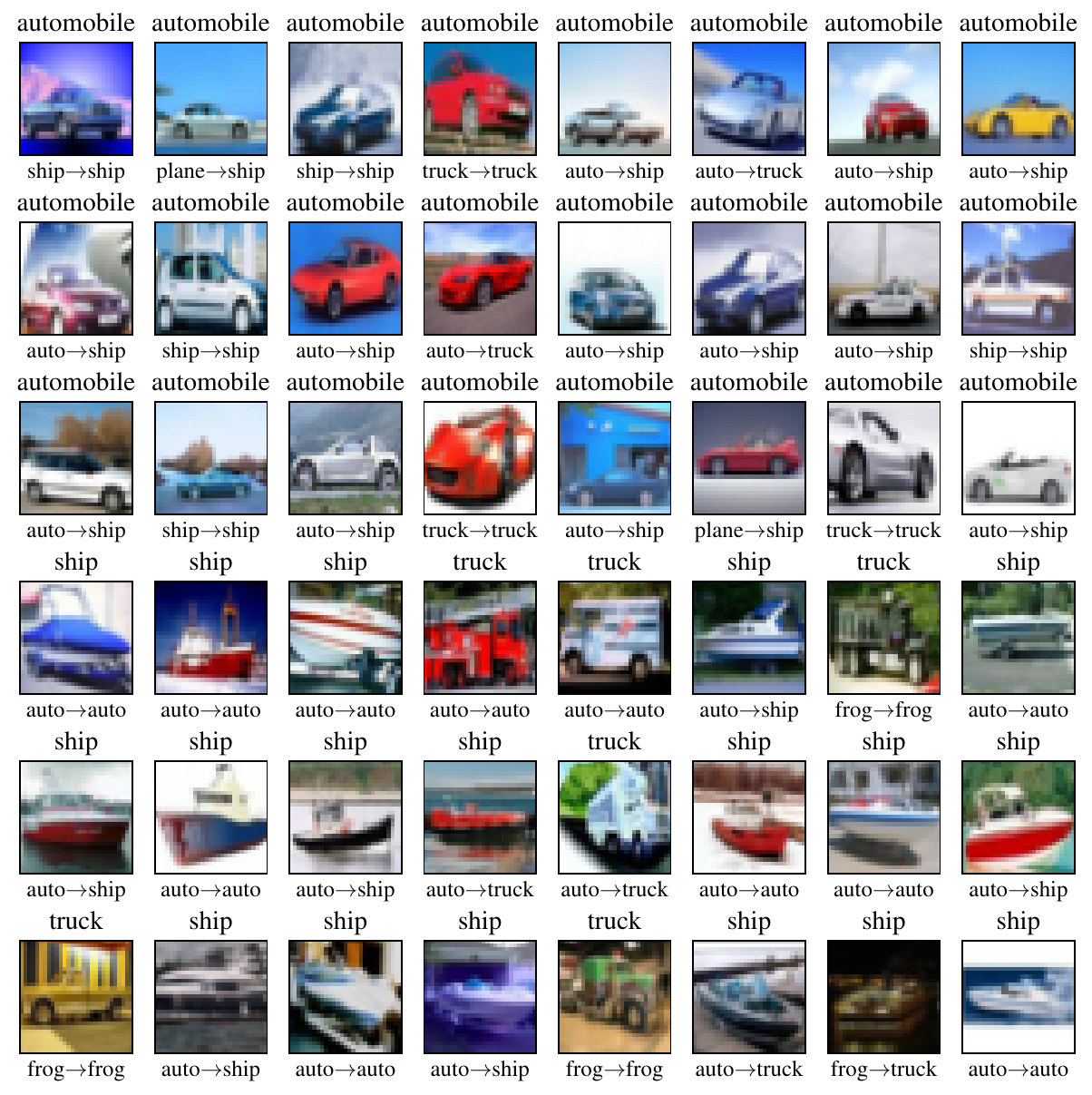}
        \caption{Step 750 to 751}
    \end{subfigure}
    \caption{\textbf{(VGG-11, seed 1)} Images with the most positive (top 3 rows) and most negative (bottom 3 rows) change to training loss after steps 100, 250, 500, and 750. Each image has the true label (above) and the predicted label before and after the gradient update (below).}
\end{figure}

\begin{figure}[ht!]
    \centering
    \begin{subfigure}[b]{.475\linewidth}
        \includegraphics[width = \linewidth]{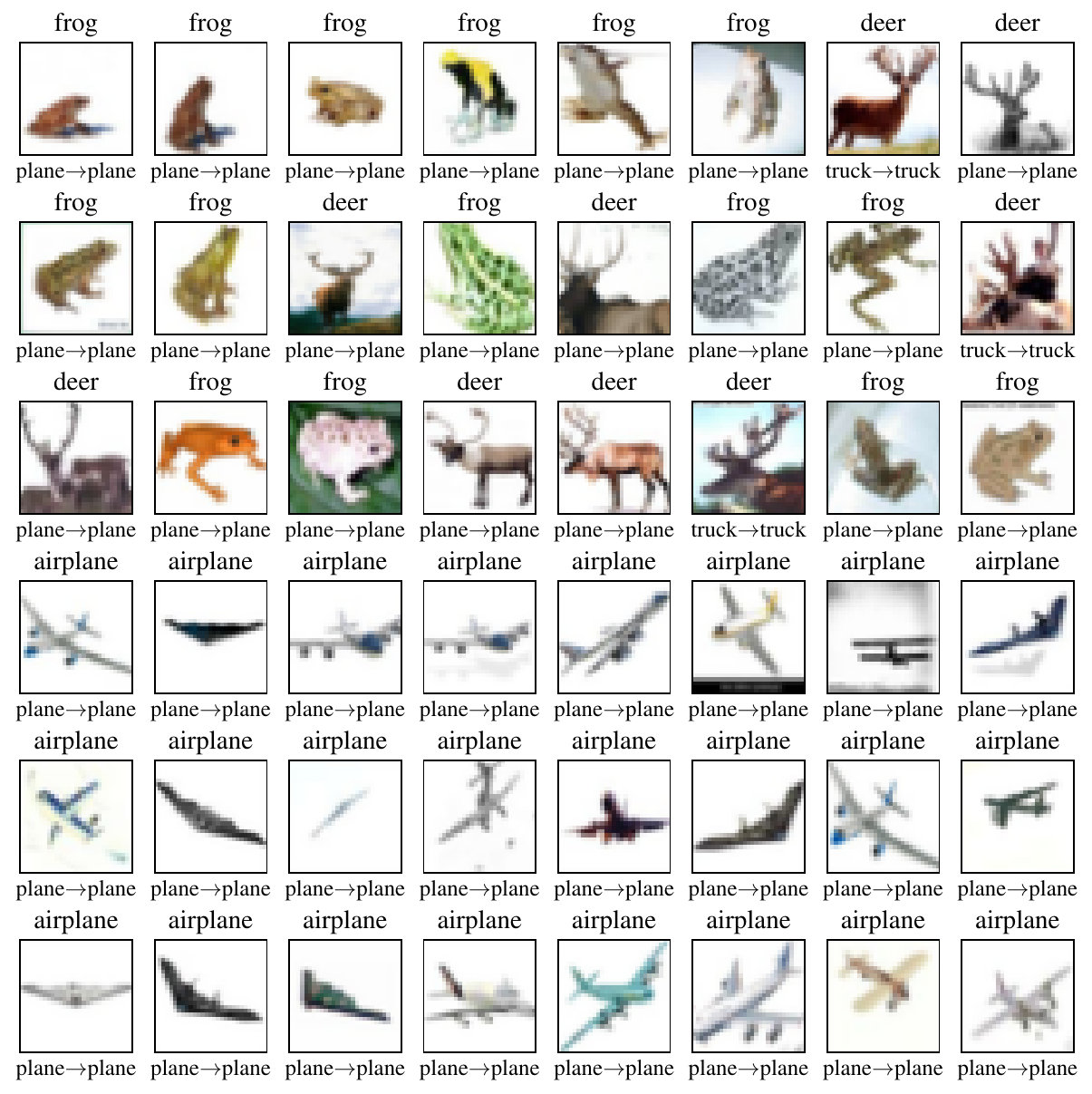}
        \caption{Step 100 to 101}
    \end{subfigure}
    \hfill
    \begin{subfigure}[b]{.475\linewidth}
        \includegraphics[width = \linewidth]{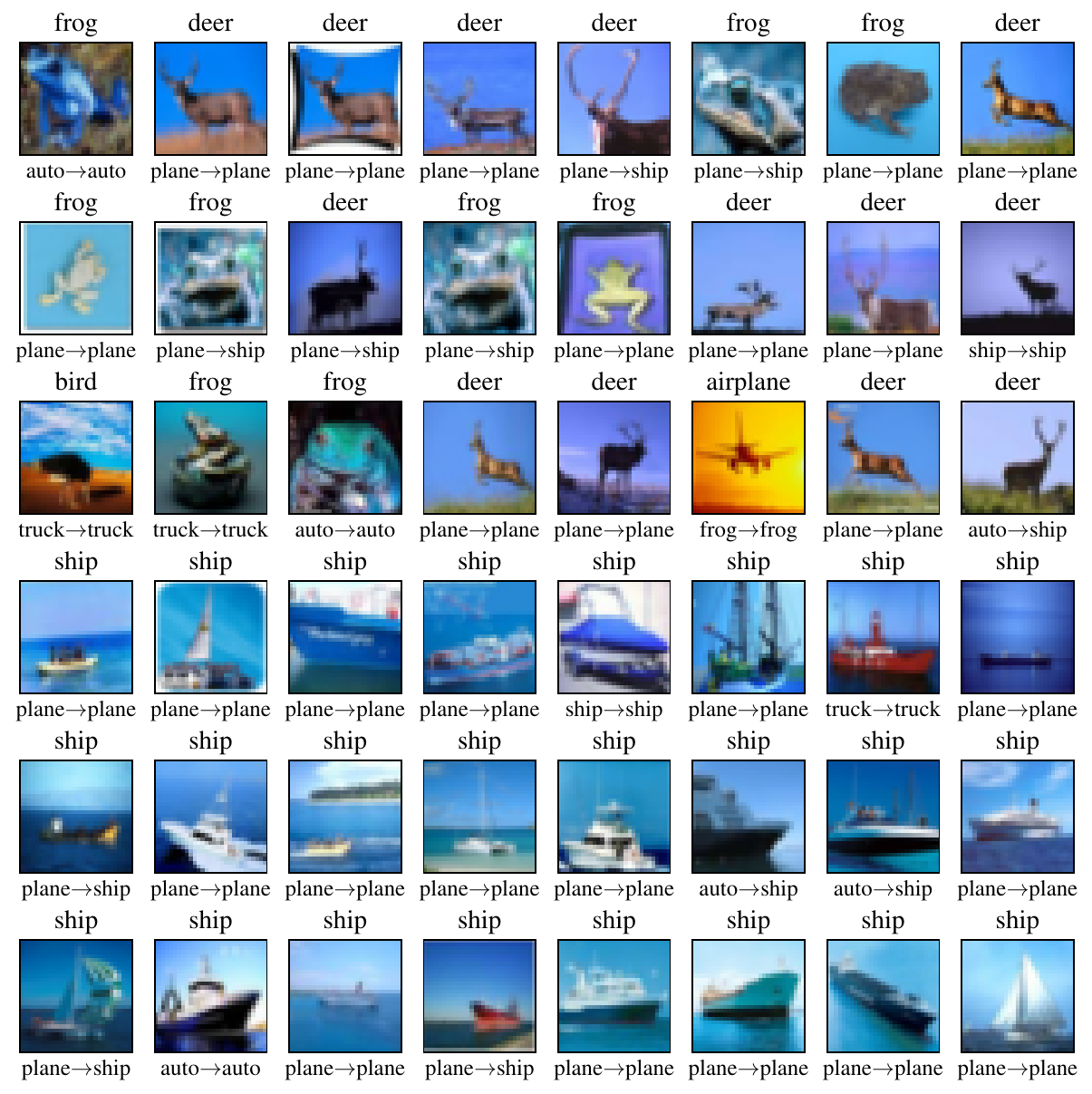}
        \caption{Step 250 to 251}
    \end{subfigure}
    \vskip\baselineskip
    \begin{subfigure}[b]{.475\linewidth}
        \includegraphics[width = \linewidth]{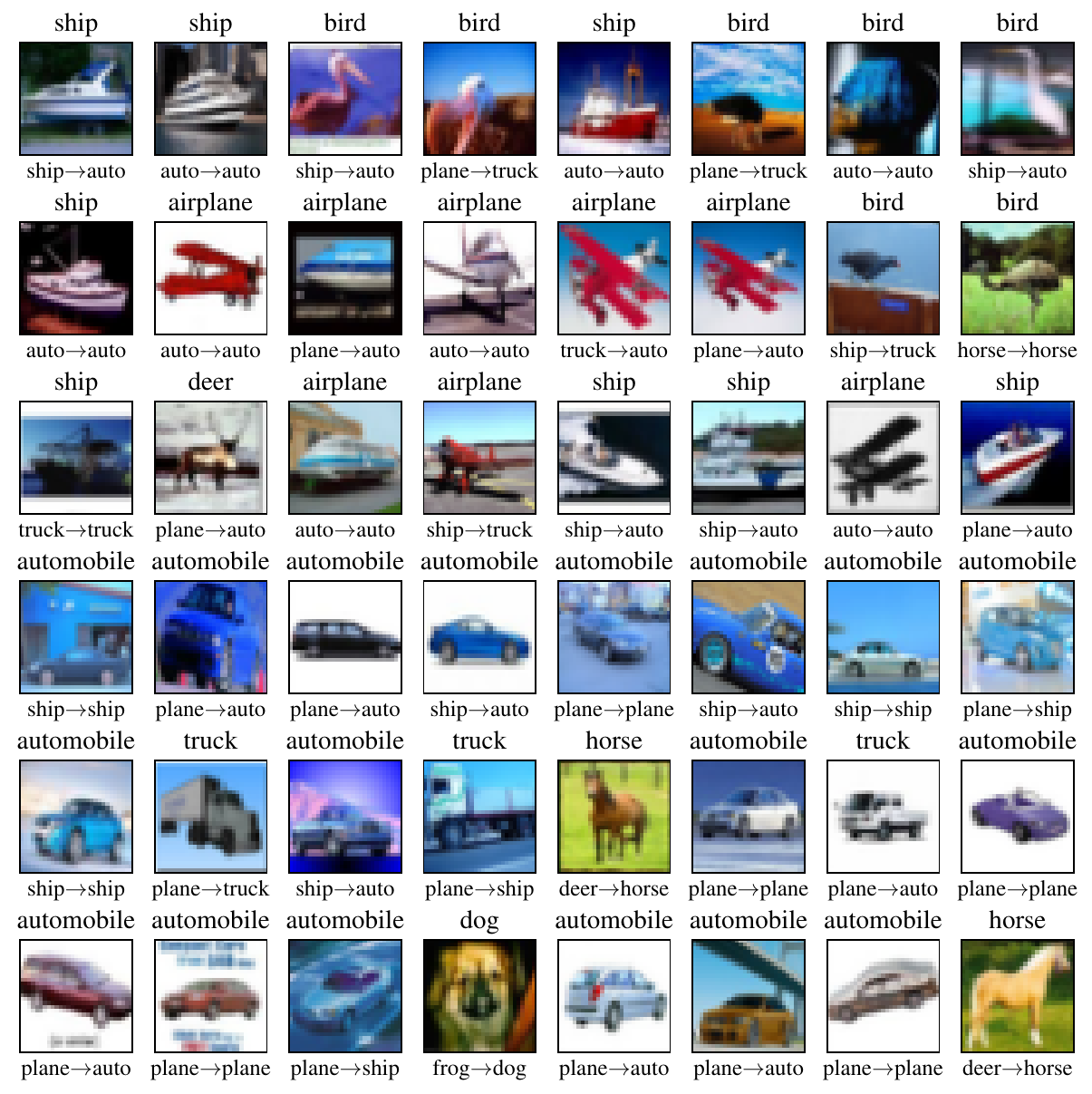}
        \caption{Step 500 to 501}
    \end{subfigure}
    \hfill
    \begin{subfigure}[b]{.475\linewidth}
        \includegraphics[width = \linewidth]{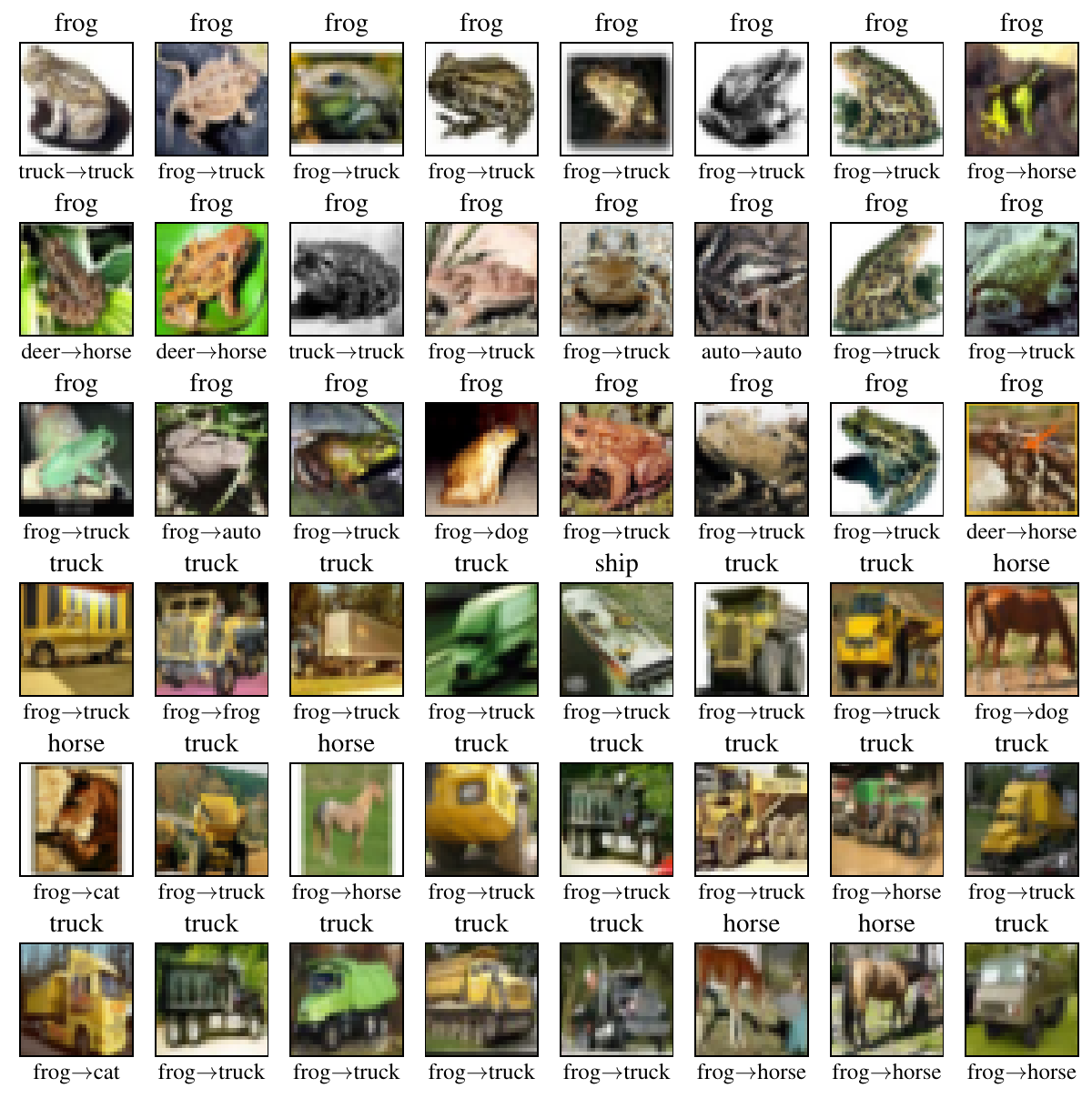}
        \caption{Step 750 to 751}
    \end{subfigure}
    \caption{\textbf{(VGG-11, seed 2)} Images with the most positive (top 3 rows) and most negative (bottom 3 rows) change to training loss after steps 100, 250, 500, and 750. Each image has the true label (above) and the predicted label before and after the gradient update (below).}
\end{figure}

\begin{figure}[ht!]
    \centering
    \begin{subfigure}[b]{.475\linewidth}
        \includegraphics[width = \linewidth]{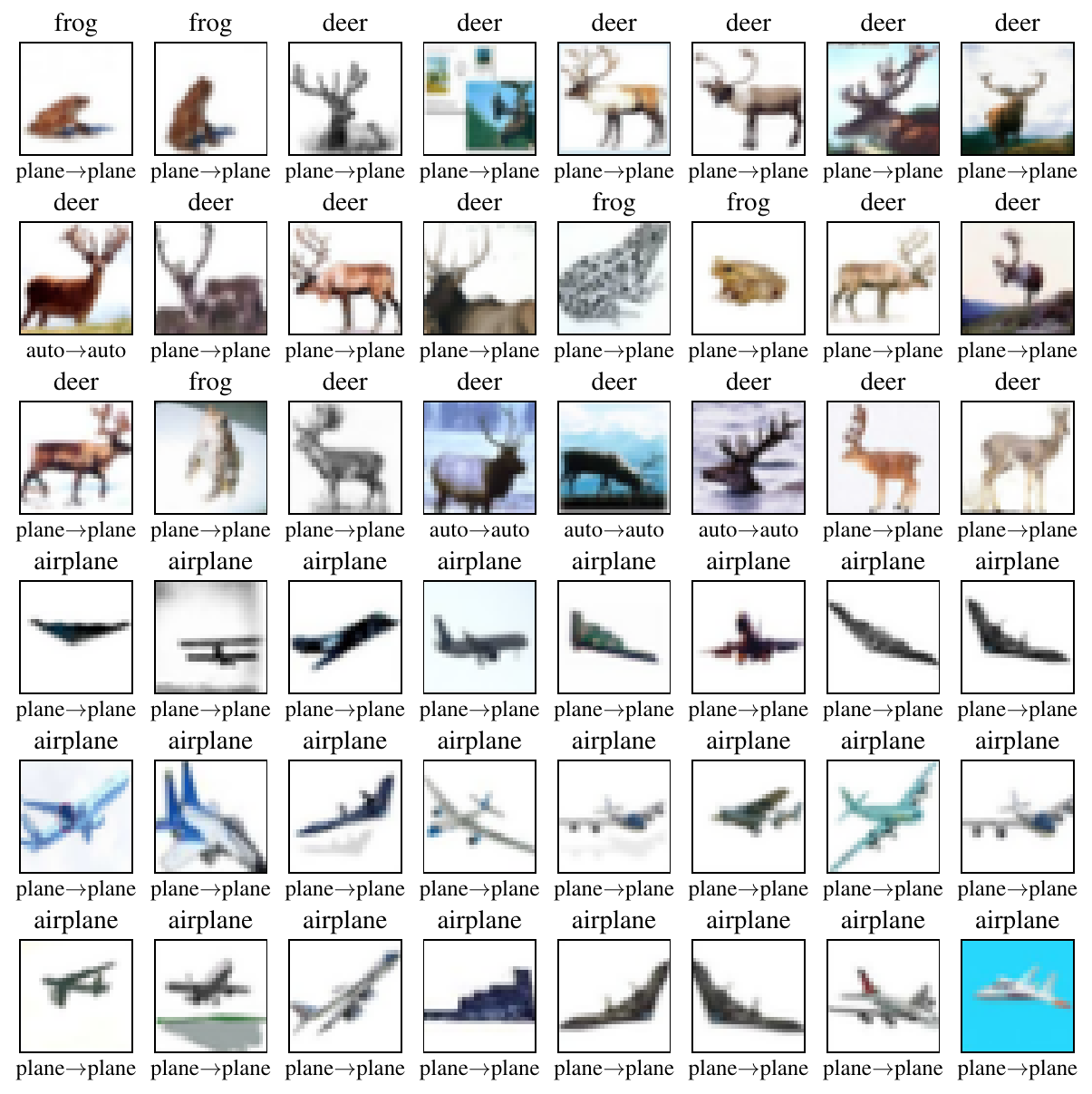}
        \caption{Step 100 to 101}
    \end{subfigure}
    \hfill
    \begin{subfigure}[b]{.475\linewidth}
        \includegraphics[width = \linewidth]{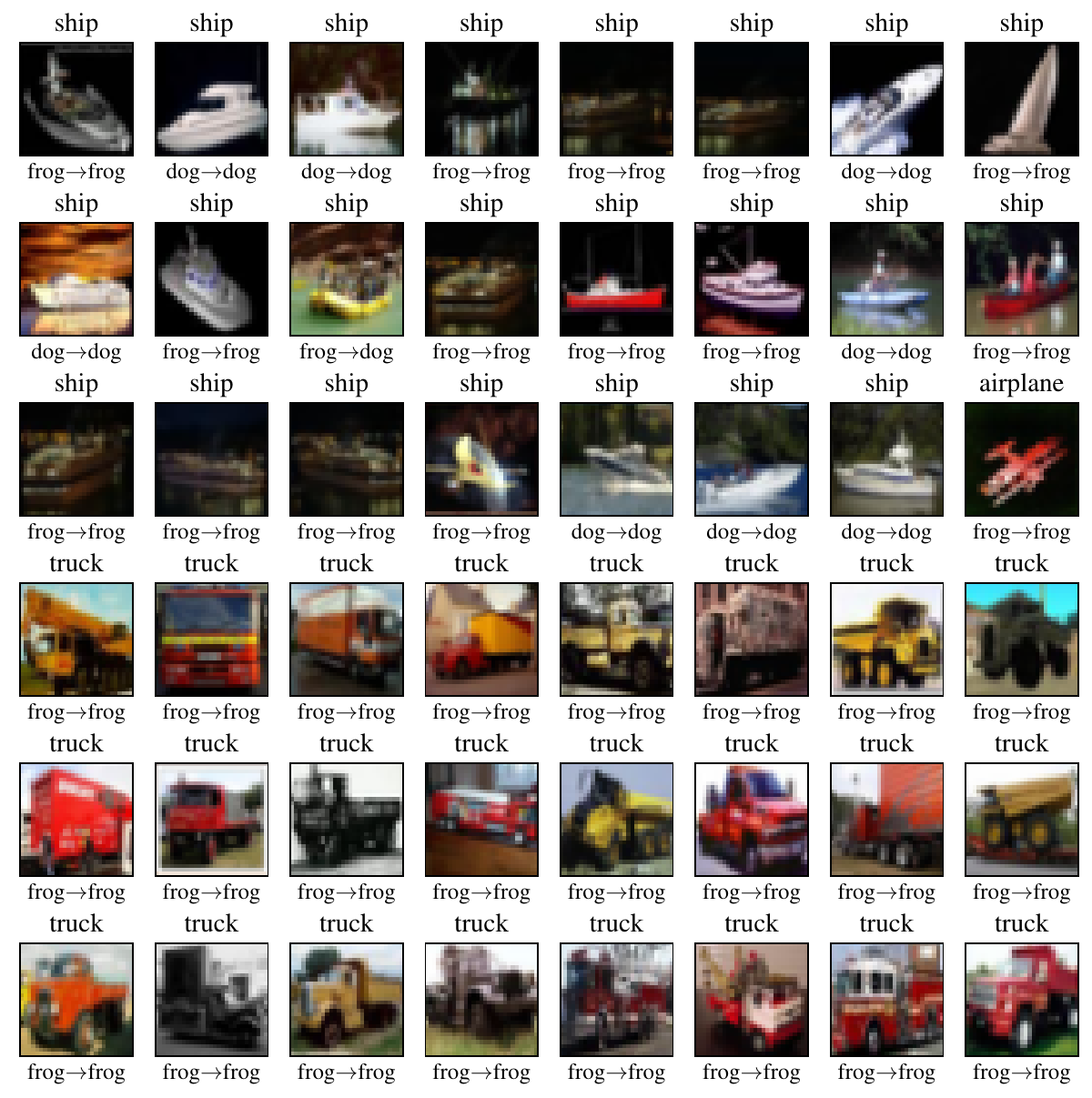}
        \caption{Step 250 to 251}
    \end{subfigure}
    \vskip\baselineskip
    \begin{subfigure}[b]{.475\linewidth}
        \includegraphics[width = \linewidth]{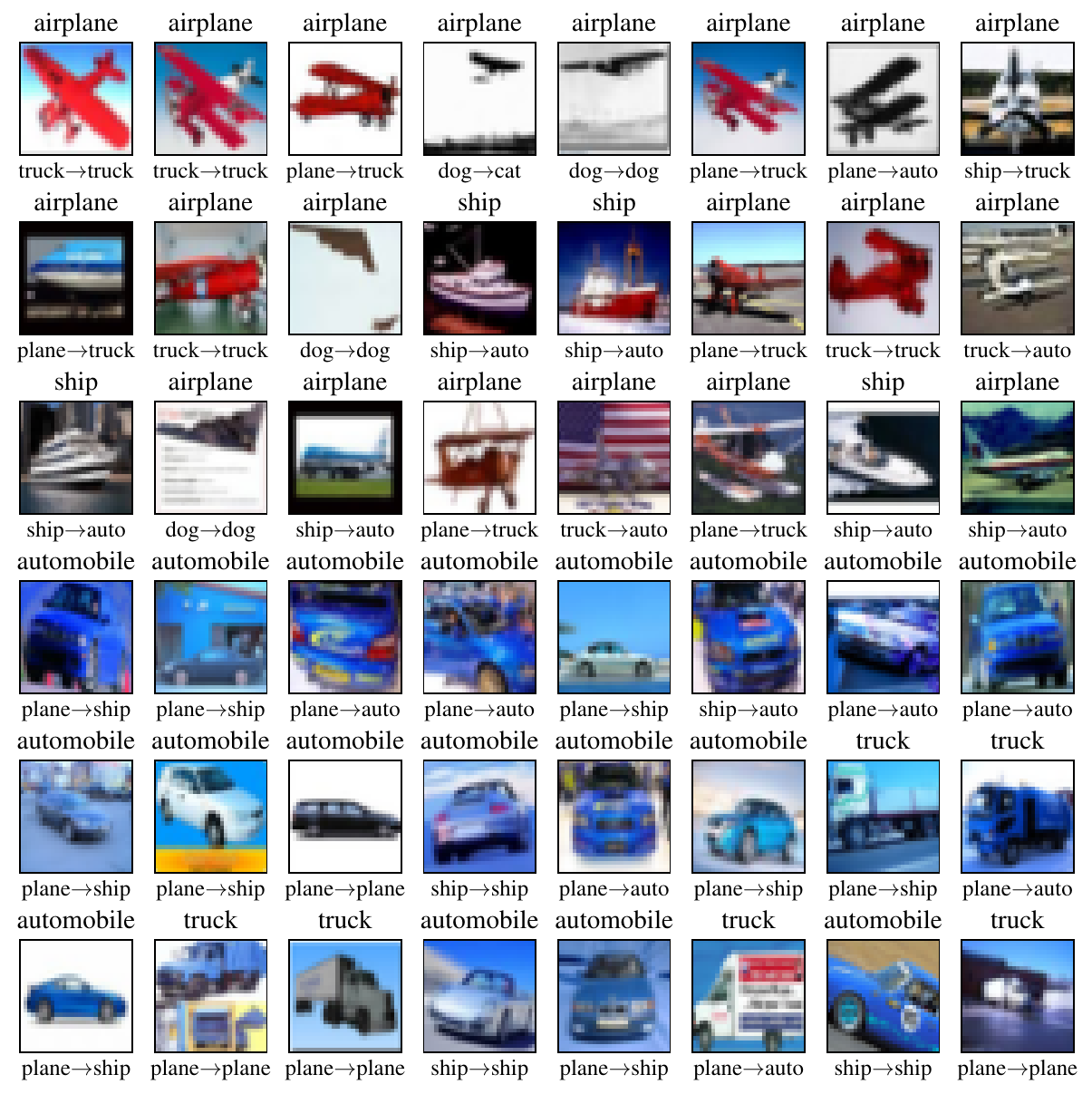}
        \caption{Step 500 to 501}
    \end{subfigure}
    \hfill
    \begin{subfigure}[b]{.475\linewidth}
        \includegraphics[width = \linewidth]{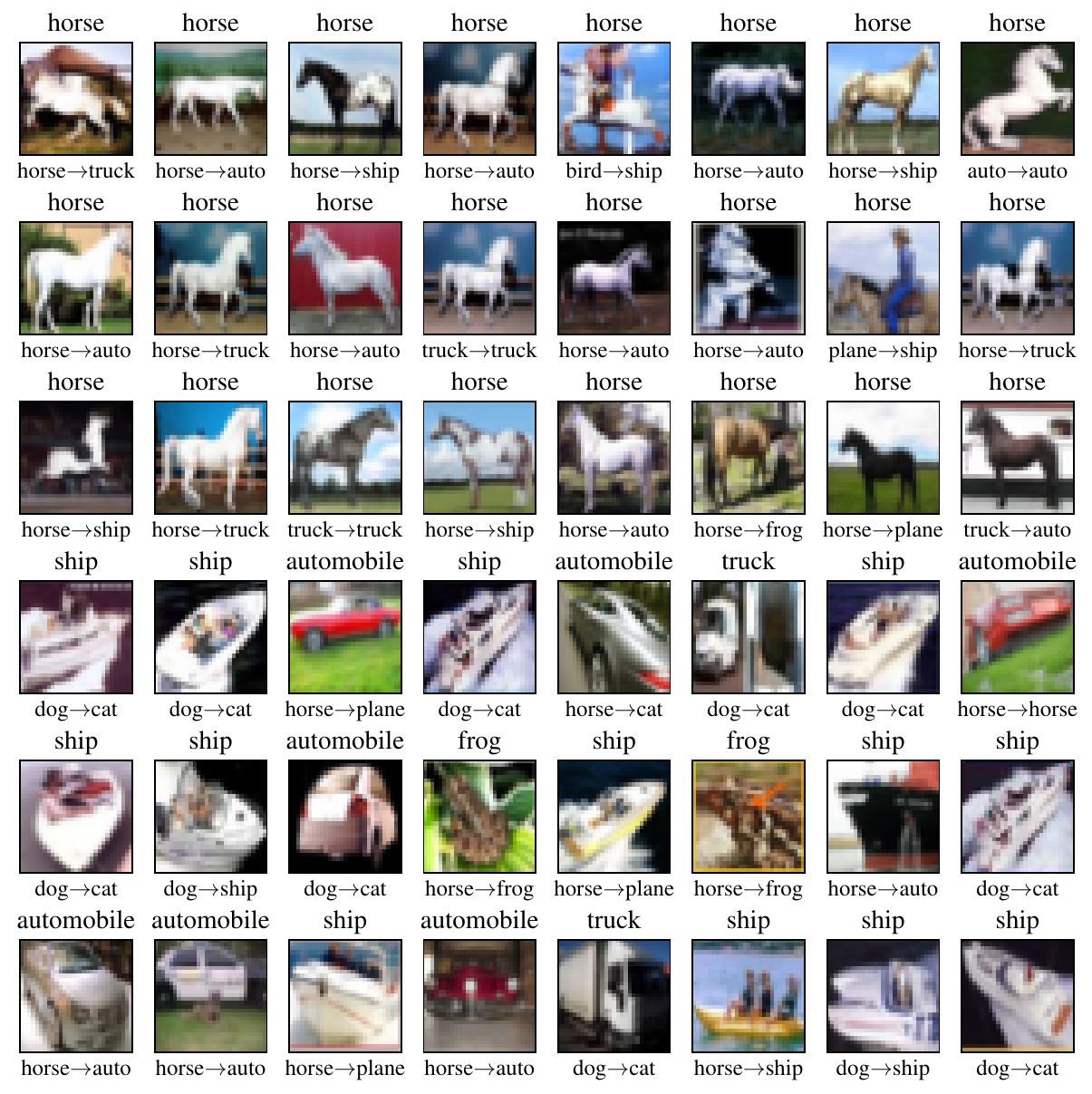}
        \caption{Step 750 to 751}
    \end{subfigure}
    \caption{\textbf{(VGG-11, seed 3)} Images with the most positive (top 3 rows) and most negative (bottom 3 rows) change to training loss after steps 100, 250, 500, and 750. Each image has the true label (above) and the predicted label before and after the gradient update (below).}
\end{figure}

\begin{figure}[ht!]
    \centering
    \begin{subfigure}[b]{.475\linewidth}
        \includegraphics[width = \linewidth]{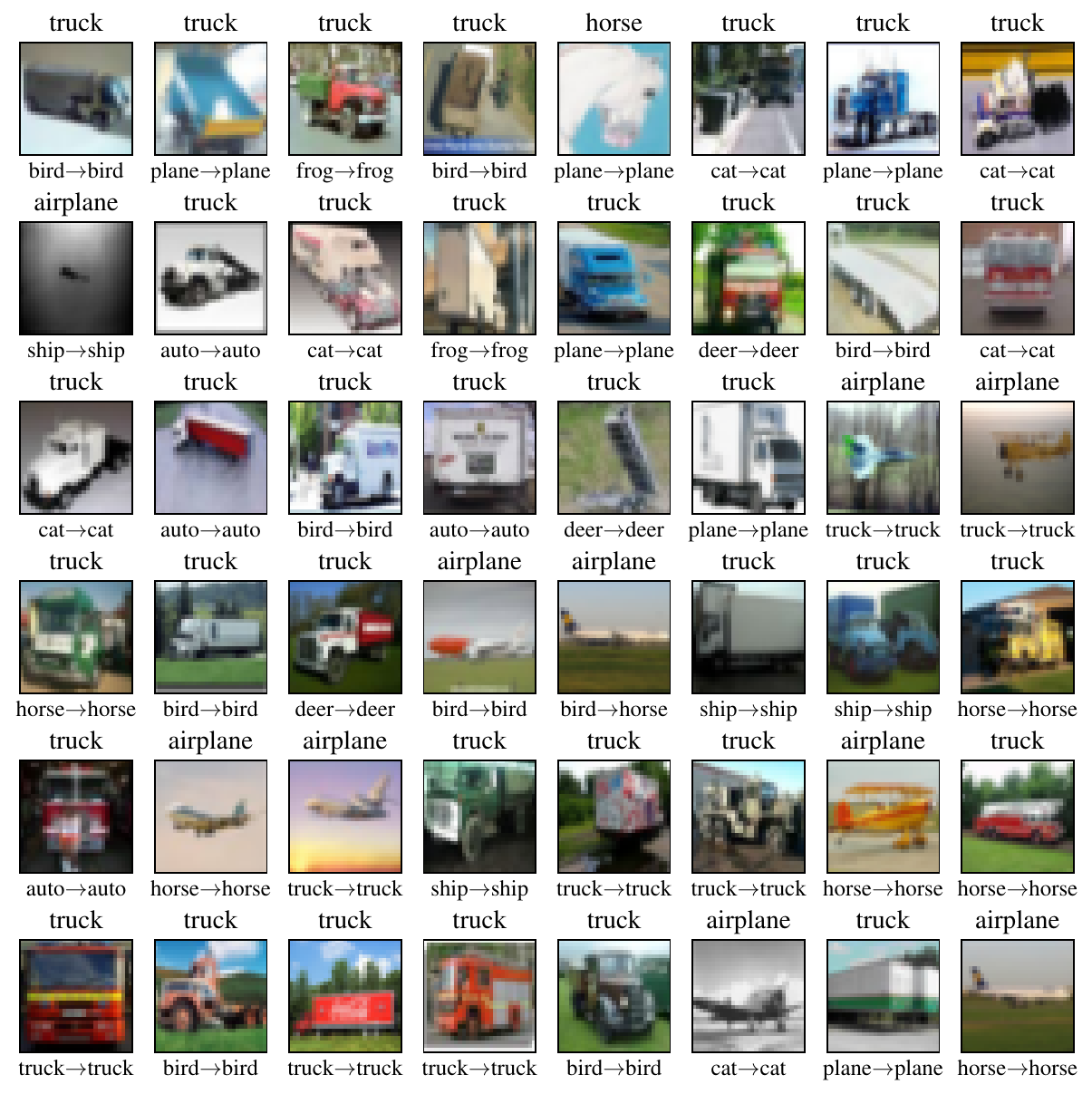}
        \caption{Step 100 to 101}
    \end{subfigure}
    \hfill
    \begin{subfigure}[b]{.475\linewidth}
        \includegraphics[width = \linewidth]{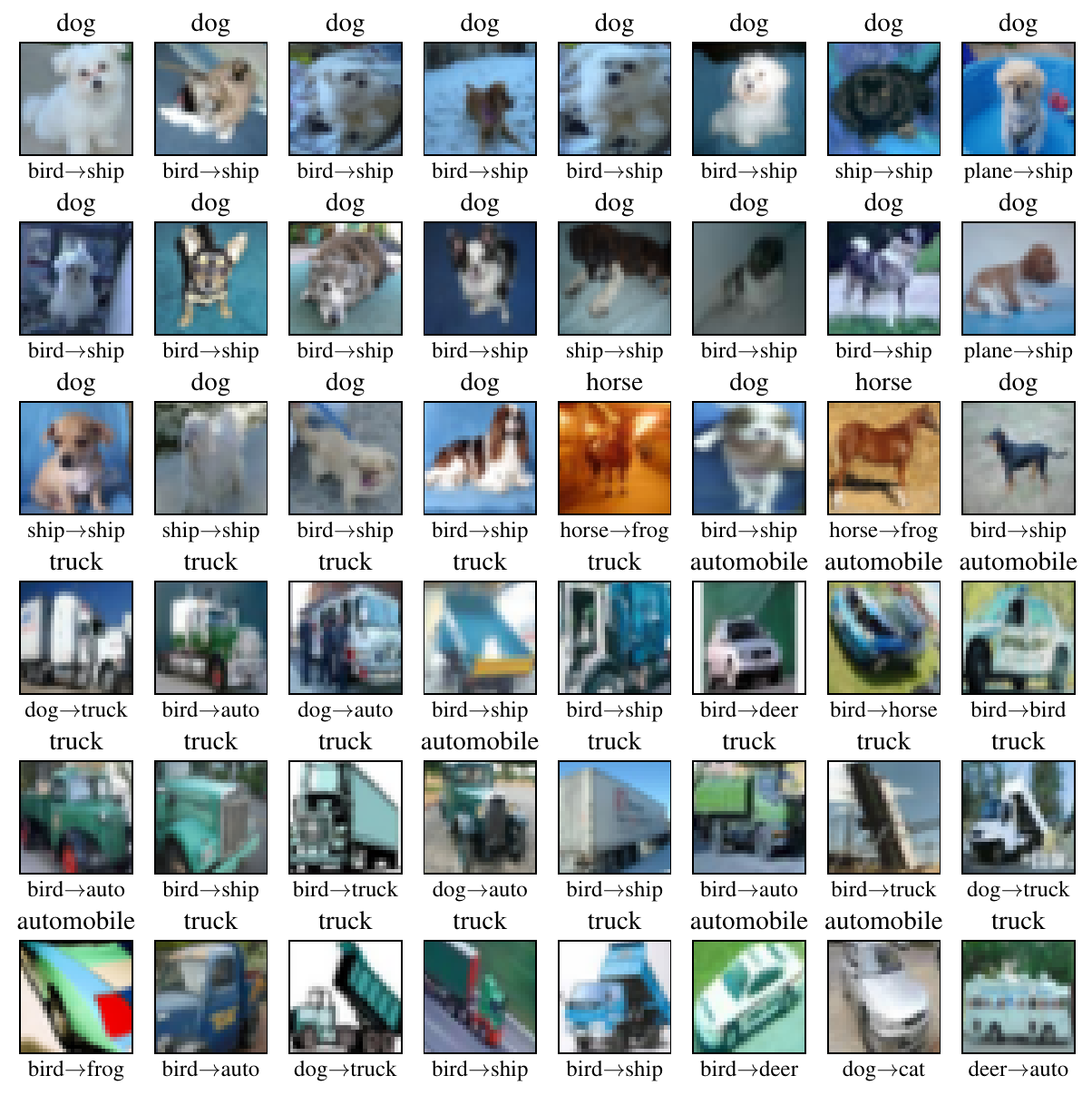}
        \caption{Step 250 to 251}
    \end{subfigure}
    \vskip\baselineskip
    \begin{subfigure}[b]{.475\linewidth}
        \includegraphics[width = \linewidth]{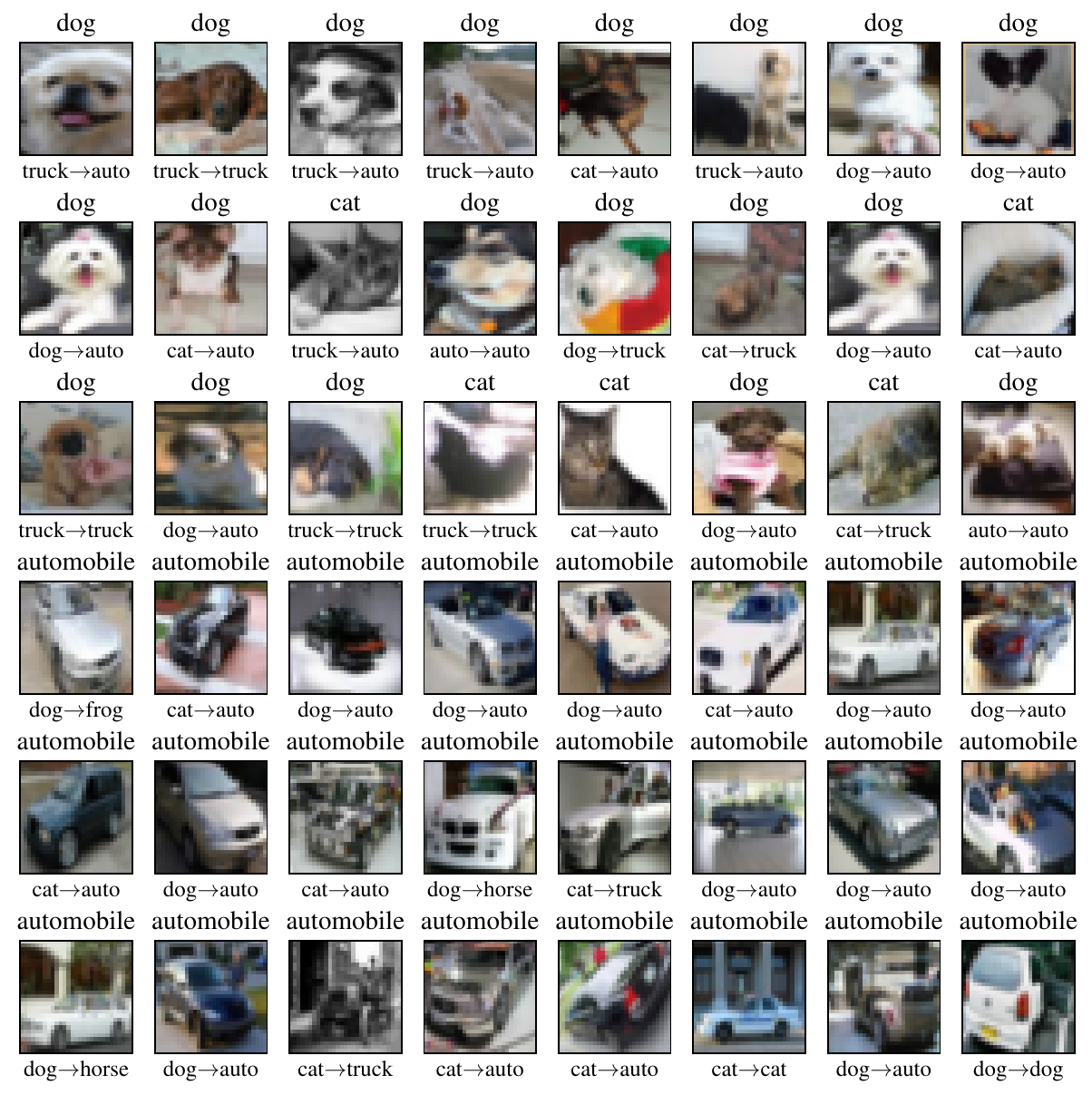}
        \caption{Step 500 to 501}
    \end{subfigure}
    \hfill
    \begin{subfigure}[b]{.475\linewidth}
        \includegraphics[width = \linewidth]{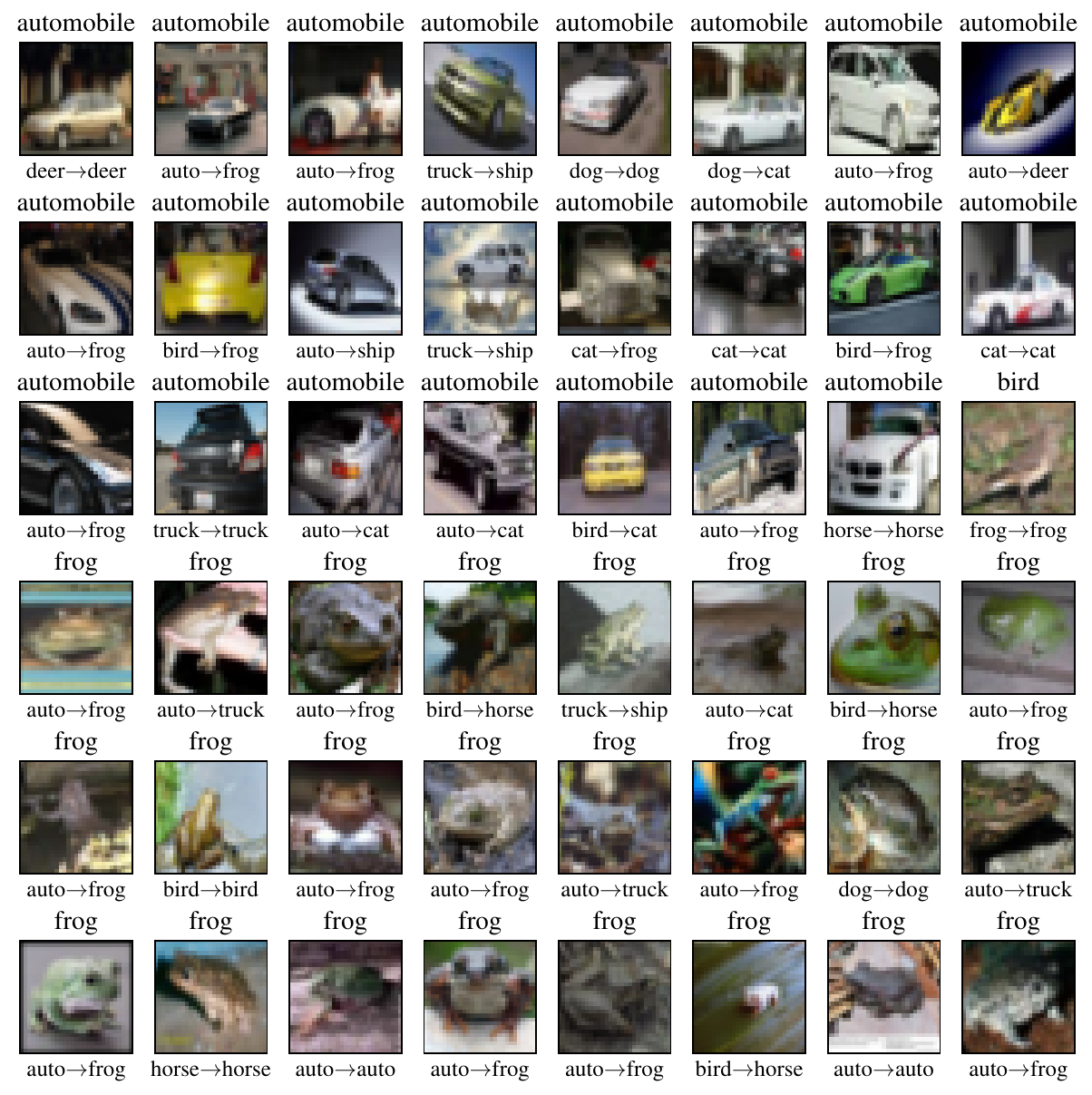}
        \caption{Step 750 to 751}
    \end{subfigure}
    \caption{\textbf{(ViT, seed 1)} Images with the most positive (top 3 rows) and most negative (bottom 3 rows) change to training loss after steps 100, 250, 500, and 750. Each image has the true label (above) and the predicted label before and after the gradient update (below).}
\end{figure}

\begin{figure}[ht!]
    \centering
    \begin{subfigure}[b]{.475\linewidth}
        \includegraphics[width = \linewidth]{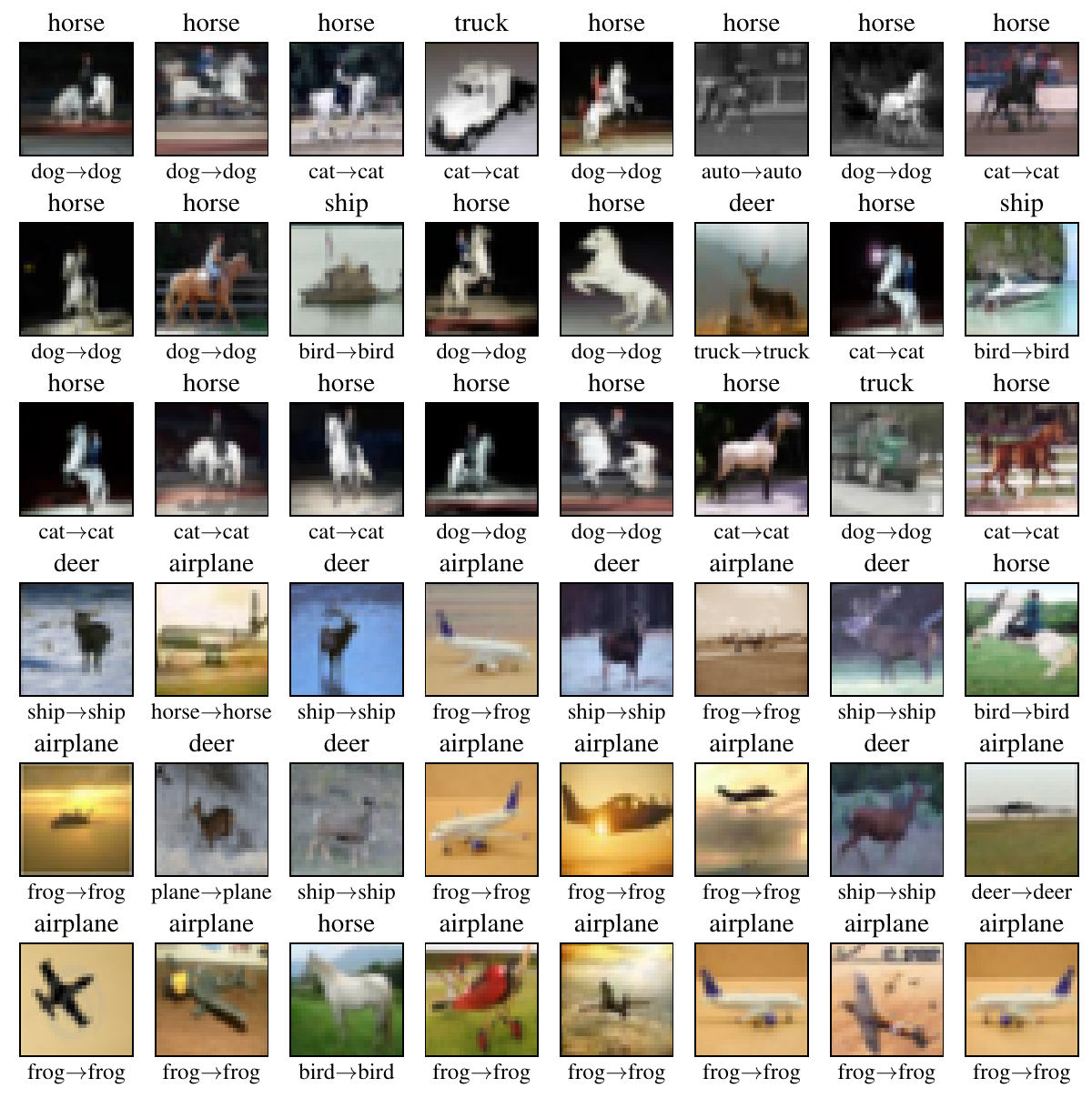}
        \caption{Step 100 to 101}
    \end{subfigure}
    \hfill
    \begin{subfigure}[b]{.475\linewidth}
        \includegraphics[width = \linewidth]{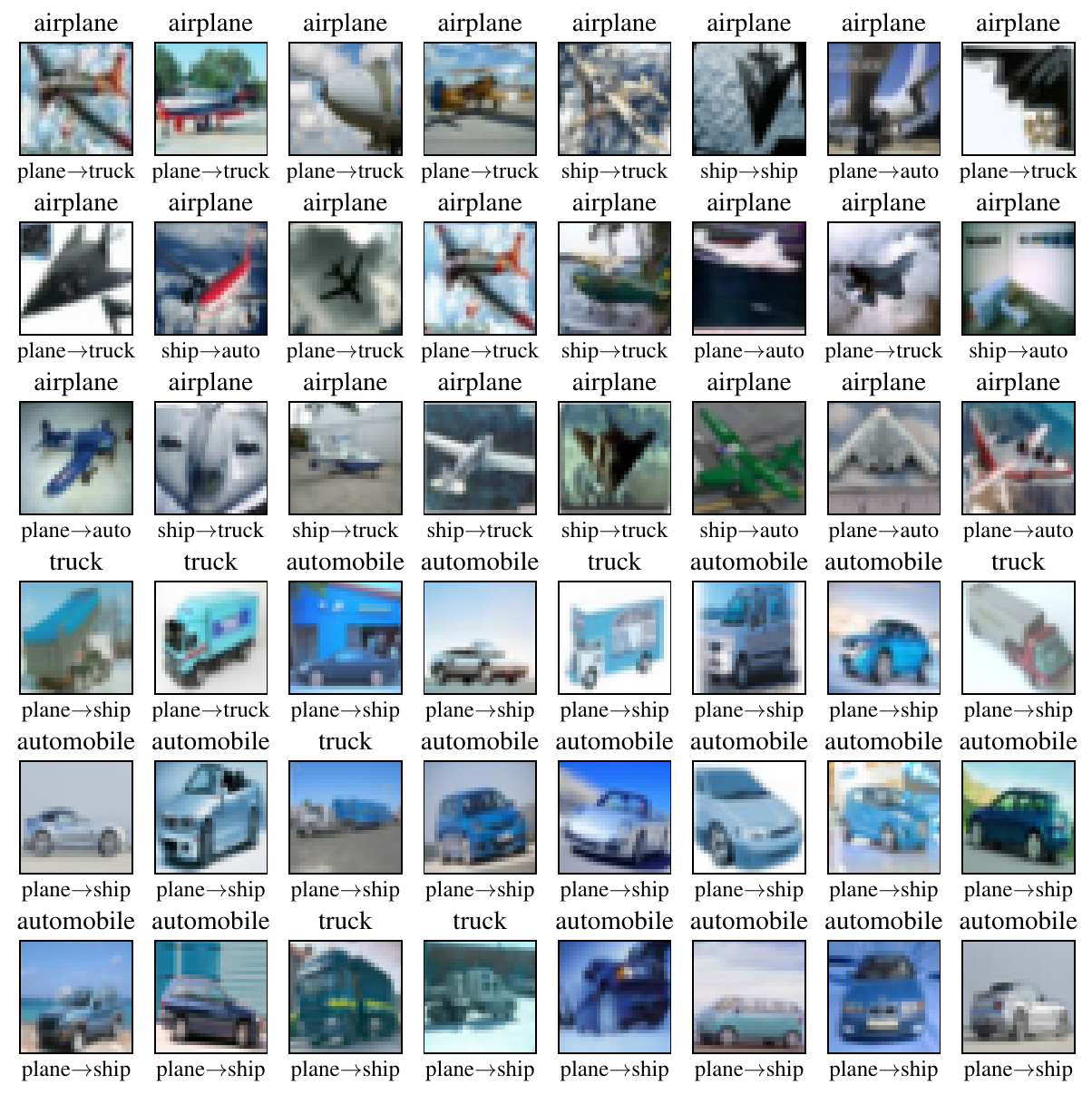}
        \caption{Step 250 to 251}
    \end{subfigure}
    \vskip\baselineskip
    \begin{subfigure}[b]{.475\linewidth}
        \includegraphics[width = \linewidth]{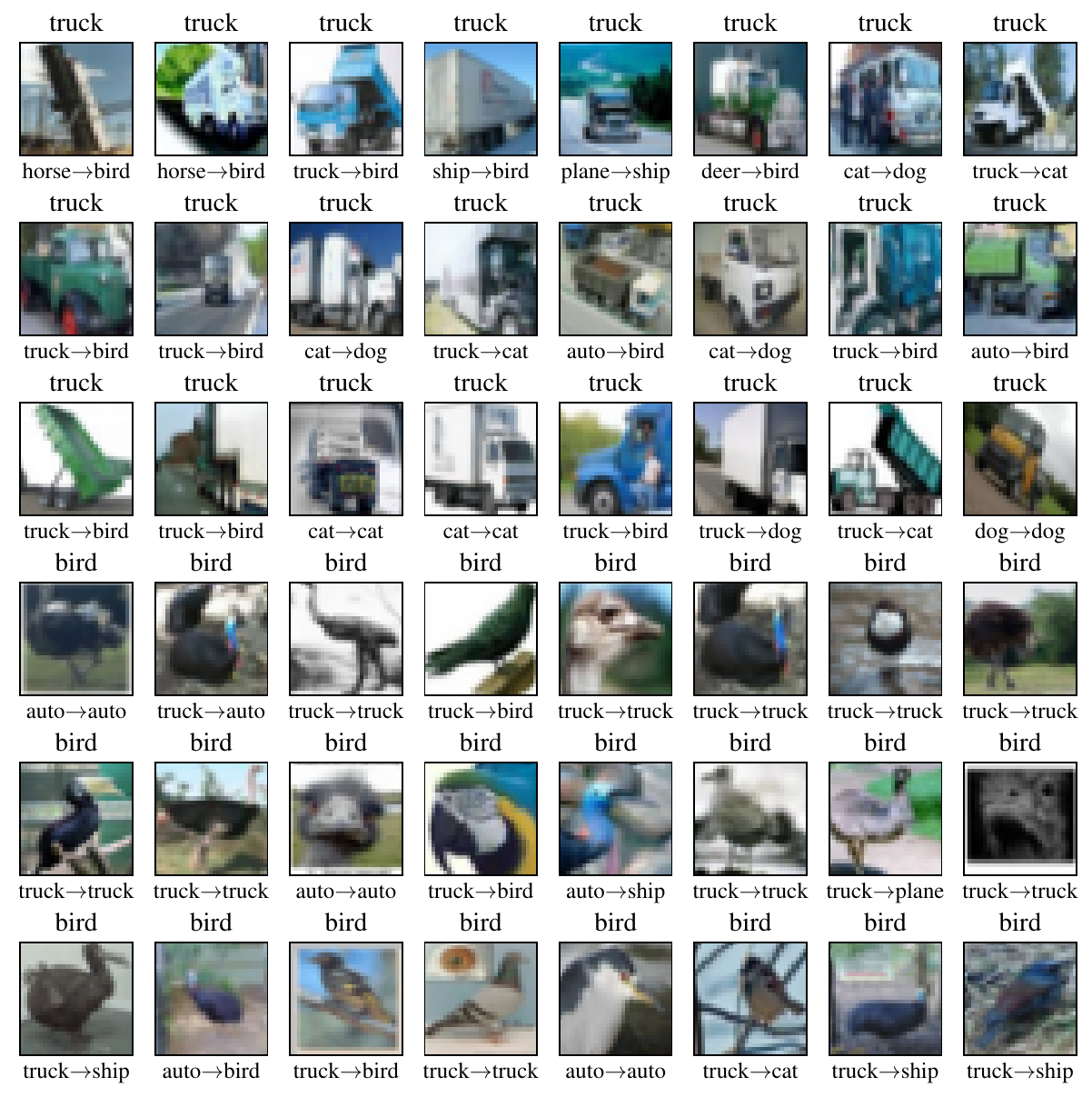}
        \caption{Step 500 to 501}
    \end{subfigure}
    \hfill
    \begin{subfigure}[b]{.475\linewidth}
        \includegraphics[width = \linewidth]{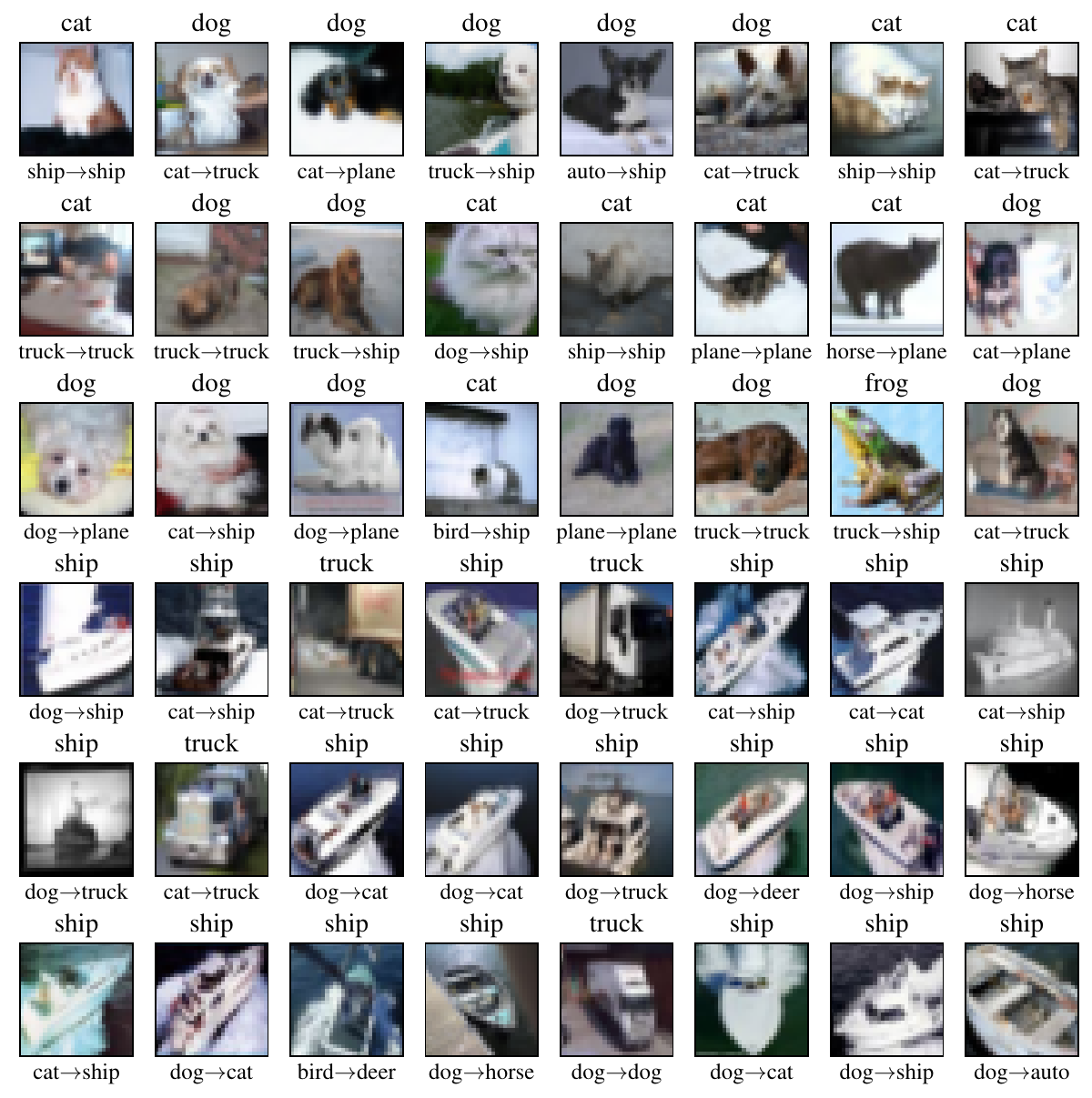}
        \caption{Step 750 to 751}
    \end{subfigure}
    \caption{\textbf{(ViT, seed 2)} Images with the most positive (top 3 rows) and most negative (bottom 3 rows) change to training loss after steps 100, 250, 500, and 750. Each image has the true label (above) and the predicted label before and after the gradient update (below).}
\end{figure}

\begin{figure}[ht!]
    \centering
    \begin{subfigure}[b]{.475\linewidth}
        \includegraphics[width = \linewidth]{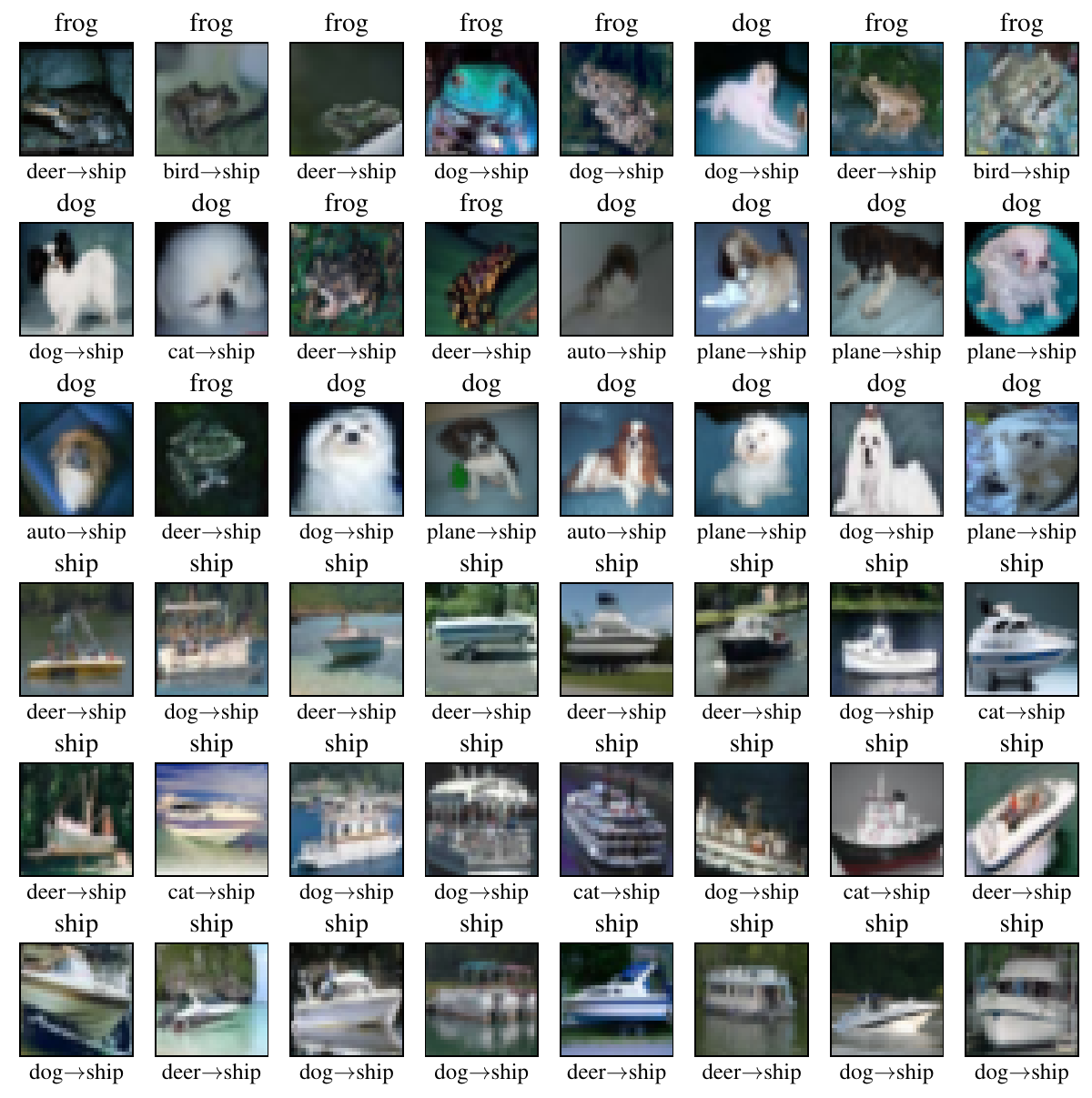}
        \caption{Step 100 to 101}
    \end{subfigure}
    \hfill
    \begin{subfigure}[b]{.475\linewidth}
        \includegraphics[width = \linewidth]{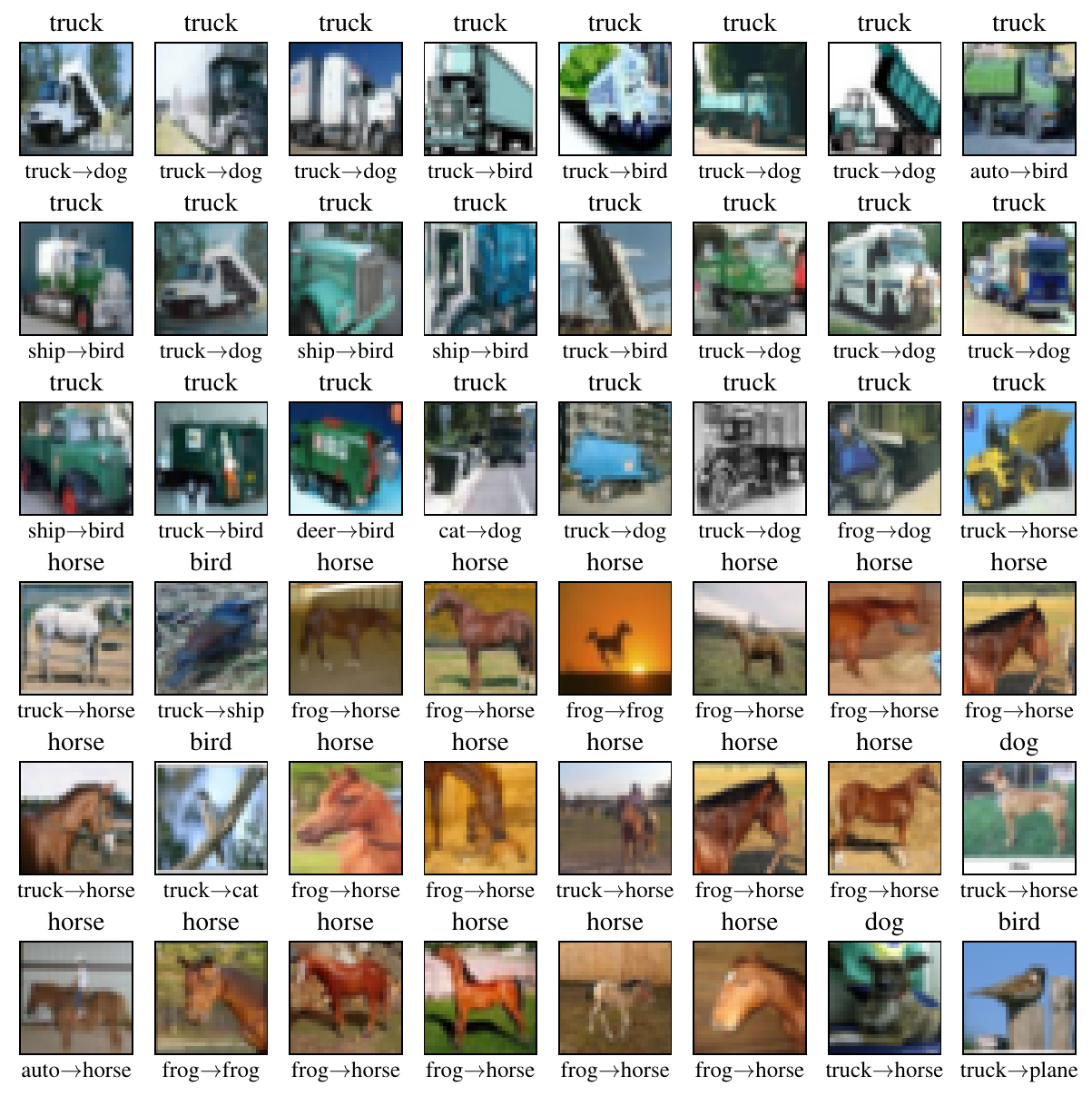}
        \caption{Step 250 to 251}
    \end{subfigure}
    \vskip\baselineskip
    \begin{subfigure}[b]{.475\linewidth}
        \includegraphics[width = \linewidth]{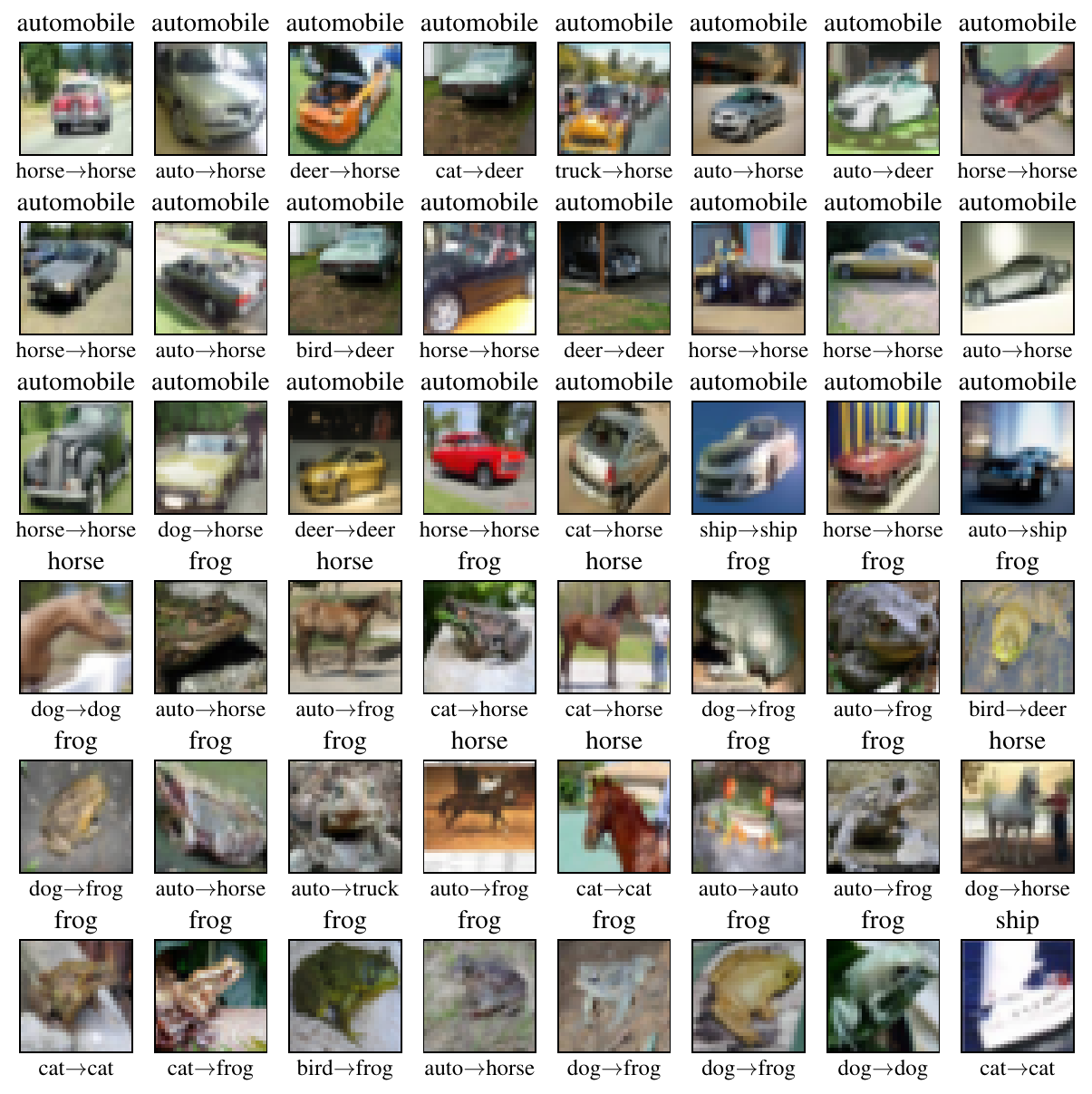}
        \caption{Step 500 to 501}
    \end{subfigure}
    \hfill
    \begin{subfigure}[b]{.475\linewidth}
        \includegraphics[width = \linewidth]{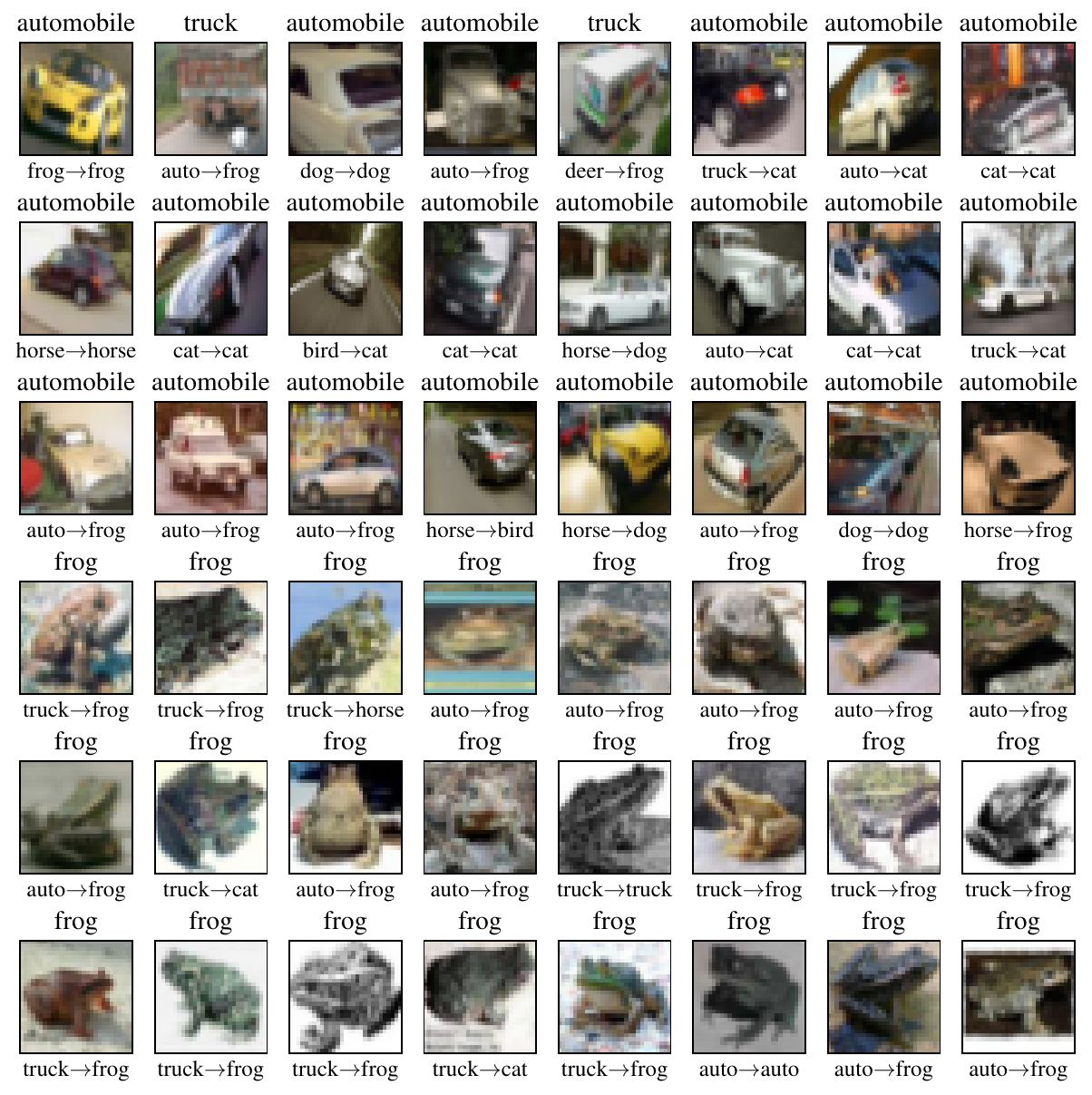}
        \caption{Step 750 to 751}
    \end{subfigure}
    \caption{\textbf{(ViT, seed 3)} Images with the most positive (top 3 rows) and most negative (bottom 3 rows) change to training loss after steps 100, 250, 500, and 750. Each image has the true label (above) and the predicted label before and after the gradient update (below).}
\end{figure}

\end{document}